%% file: main.tex
\documentclass[accepted]{uai2026} %

\usepackage[british]{babel}

\usepackage{natbib} %
\usepackage{mathtools} %
\usepackage{booktabs} %
\usepackage{tikz} %
\usepackage{amsfonts}
\usepackage[dvipsnames]{xcolor}
\everypar=\expandafter{\the\everypar\looseness=-1 }  %
\usepackage[capitalize,noabbrev]{cleveref}

\crefformat{algorithm}{Alg.~#2#1#3}
    
\crefname{thm}{Thm.}{Thms.}
\Crefname{thm}{Thm.}{Thms.}
\crefname{lem}{Lemma}{Lemmas}
\Crefname{lem}{Lemma}{Lemmas}
\crefname{proposition}{Prop.}{Props.}
\Crefname{proposition}{Prop.}{Props.}
\crefname{definition}{Def.}{Defs.}
\Crefname{definition}{Def.}{Defs.}
\crefname{property}{Property}{Properties}
\Crefname{property}{Property}{Properties}
\crefname{example}{Example}{Examples}
\Crefname{example}{Example}{Examples}
\crefname{theorem}{Thm.}{Theorems}
\Crefname{theorem}{Thm.}{Theorems}
\crefname{corollary}{Cor.}{Corols.}
\Crefname{corollary}{Cor.}{Corols.}
\crefname{equation}{Eq.}{Eqs.}
\Crefname{equation}{Eq.}{Eqs.}
\crefname{assumption}{Ass.}{Assumps.}
\Crefname{assumption}{Ass.}{Assumps.}
\crefname{section}{Sec.}{Secs.}
\Crefname{section}{Sec.}{Secs.}
\crefname{figure}{Fig.}{Figs.}
\Crefname{figure}{Fig.}{Figs.}
\crefname{table}{Tab.}{Tabs}
\Crefname{table}{Tab.}{Tabs}
\crefname{appendix}{App.}{Apps.}
\Crefname{appendix}{App.}{Apps.}

\usepackage{xspace}
\newcommand{\ours}{\textsc{TabPC}\xspace}

\usepackage{bm}
\usepackage{siunitx}
\usepackage{multirow}

\usepackage{algpseudocode}
\usepackage{algorithm}

\newcommand{\scope}{\ensuremath{{\mathsf{sc}}}}
\newcommand{\inscope}{\ensuremath{{\mathsf{in}}}}

\newcommand{\metricname}[1]{\textsc{\small #1}}
\newcommand{\ctwosttt}{\metricname{C2ST}\xspace}
\newcommand{\ctst}{\ref{eq:c2st}\xspace}
\newcommand{\ctwost}{\ctst}
\newcommand{\trend}{\metricname{Trend}\xspace}
\newcommand{\shape}{\metricname{Shape}\xspace}
\newcommand{\wnmis}{\metricname{wNMIS}\xspace}
\newcommand{\nmis}{\metricname{NMIS}\xspace}
\newcommand{\alphaprecision}{\metricname{$\alpha$-precision}\xspace}
\newcommand{\betarecall}{\metricname{$\beta$-recall}\xspace}

\newcommand{\methodname}[1]{\textsc{\small #1}}
\newcommand{\FF}{\ref{eq:ff_model}\xspace} 
\newcommand{\SM}{\ref{eq:sm_model}\xspace} 
\newcommand{\tabdiff}{\methodname{TabDiff}\xspace}
\newcommand{\tabsyn}{\methodname{TabSyn}\xspace}
\newcommand{\ctgan}{\methodname{CtGAN}\xspace}
\newcommand{\tvae}{\methodname{Tvae}\xspace}
\newcommand{\codi}{\methodname{CoDi}\xspace}
\newcommand{\stasy}{\methodname{STaSy}\xspace}
\newcommand{\cdtd}{\methodname{CDTD}\xspace}
\newcommand{\great}{\methodname{GReaT}\xspace}
\newcommand{\hivae}{\methodname{HI-VAE}\xspace}
\newcommand{\tabpfn}{\methodname{TabPFN}\xspace}

\newcommand{\sota}{SotA\xspace}

\usepackage{amsthm}
\theoremstyle{plain}
\newtheorem{theorem}{Theorem}[section]

\newtheorem{lemma}[theorem]{Lemma}

\theoremstyle{definition}

\theoremstyle{remark}

\hypersetup{
colorlinks=true,linkcolor=teal,citecolor=DarkOrchid,urlcolor=DarkOrchid,
}

\usepackage{subcaption}

\title{A Sobering Look at Tabular Data Generation via Probabilistic Circuits}

\newcommand{\random}{{\small\textcircled{r}}}

\author[ ]{\href{mailto:<davide.scassola@phd.units.it>?Subject=Your UAI 2026 paper}{Davide~Scassola}\textsuperscript{$2, 3, *$}~{\random}~\href{mailto:<D.W.Ponsford@sms.ed.ac.uk>?Subject=Your UAI 2026 paper}{Dylan~Ponsford}\textsuperscript{$1, *$}}

\author[1]
{\href{mailto:<ajavaloy@ed.ac.uk>?Subject=Your UAI 2026 paper}{Adrián~Javaloy}{}}

\author[3]
{\href{mailto:<sebastiano@aindo.com>?Subject=Your UAI 2026 paper}{Sebastiano~Saccani}{}}

\author[2]
{\href{mailto:<lbortolussi@units.it>?Subject=Your UAI 2026 paper}{Luca~Bortolussi}{}}

\author[1]
{\href{mailto:<henry.gouk@ed.ac.uk>?Subject=Your UAI 2026 paper}{Henry~Gouk}{}}

\author[1]
{\href{mailto:<avergari@ed.ac.uk>?Subject=Your UAI 2026 paper}{Antonio~Vergari}{}}
\affil[1]{%
    School of Informatics\\
    University of Edinburgh\\
    Edinburgh, UK
}
\affil[2]{%
    AI\textsc{lab}\\
    University of Trieste\\
    Trieste, Italy
}
\affil[3]{%
    Aindo SpA\\
    AREA Science Park\\
    Trieste, Italy\\
  }
  
\begin{document}
\maketitle

\def\thefootnote{*}\footnotetext{Equal contribution (\href{https://www.aeaweb.org/journals/policies/random-author-order/search?RandomAuthorsSearch\%5Bsearch\%5D=BlmxyClAO2YD}{author order was randomized}).}\def\thefootnote{\arabic{footnote}}

\begin{abstract}
Tabular data is more challenging to generate than text and images, due to its heterogeneous features and much lower sample sizes.
On this task, diffusion-based models are the current state-of-the-art (\sota) model class, %
achieving almost perfect performance on commonly used benchmarks.
In this paper, \textit{we question the perception of progress for tabular data generation}.
First, 
{we highlight the limitations of current protocols to evaluate the fidelity of generated data, and advocate for alternative ones. Next,}
we revisit a simple baseline---hierarchical mixture models in the form of deep probabilistic circuits (PCs)---which delivers competitive or superior performance to \sota models \textit{for a fraction of the cost}.
PCs are the generative counterpart of decision %
forests, and as such can natively handle heterogeneous data as well as deliver tractable probabilistic generation and inference.
{Finally}, in a rigorous empirical analysis we show that the apparent saturation of progress for \sota models is {largely} due to the {use of inadequate metrics}. %
As such, we highlight that there is still much to be done to generate realistic tabular data.
Code available at \url{https://github.com/april-tools/tabpc}.
\end{abstract}

\section{Introduction}\label{sec:intro}

Deep Generative Models (DGMs; \citealt{tomczak2024deep}) have mastered generating high-dimensional and structured data such as images \citep{croitoru2023diffusion} and text \citep{achiam2023gpt}.
However, a modality still considered challenging for DGMs is \textit{tabular data}---data found in the rows (samples) and columns (features) of a database table \citep{borisov2022deep}.
One explanation for this is that tabular data can be \textit{heterogeneous}, i.e., be both continuous and discrete, and have different statistical data types \citep{valera2017automatic}. Moreover, tables usually have \textit{fewer samples} compared to datasets typically used for training neural networks.
These issues also explain why simple \textit{discriminative} models such as decision trees and their ensembles \citep{xgboost} can still outperform neural predictors on tabular data \citep{grinsztajn2022tree, shwartz2022tabular,van2024tabular}.

\begin{figure}[!t]
    \centering
    \includegraphics[width=\linewidth]{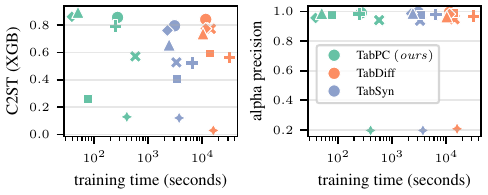}
    \caption{\textbf{PCs for tabular data (\ours) compete with diffusion-based approaches at a fraction of the cost} for all datasets (denoted by marker shape, see \cref{fig:random_FF_vs_trained_FF}) on fidelity metrics such as \alphaprecision \citep{faithful} and \ctwost \citep{lopez-paz2017revisiting} when computed with an XGBoost classifier.
    We remark this is not the commonly used implementation of \ctwost, which instead uses a logistic regressor for which even a fully-factorised model yields \sota results, \textbf{delivering a false sense of progress} (see \cref{fig:xgb_vs_lr}).}
    \label{fig:pc-global}
\end{figure}

Despite this, recent diffusion-based DGMs \citep{tabddpm, tabsyn, tabdiff, tabbyflow} have shown remarkable progress on a number of commonly used tabular data generation (TDG) benchmarks \citep{adult,magic,default,beijing}, to the point that metrics used to measure the \textit{fidelity} 
of generated data \citep{faithful,sdmetrics,stasy} cannot be further improved.
For example, the \textit{classifier two-sample test} (\ctwost; \citealt{lopez-paz2017revisiting}) metric, which uses a classifier to distinguish generated from real data, or the \trend metric 
\citep{sdmetrics}
, used to measure dependency between two features, can be close to 1.0 (their maximum value, i.e., 100\% fidelity) for \sota diffusion models \citep{tabdiff}.

In this paper, \textit{we question the apparent progress of DGMs for TDG} via a twofold approach.
First, we investigate some of the most commonly used fidelity metrics and their implementation.
We note that the saturated performance of DGMs can be explained by simplistic choices of how metrics are computed, e.g., using linear classifiers for \ctwost to distinguish generated from real data, or measuring \textit{linear} correlations and univariate marginal fidelity for \trend. 
To highlight this issue, we introduce %
fully-factorised distributions as baselines, and show how they can perform on par with DGMs under these flawed metrics (\cref{sec:sota-tab-gen}).

Second, we propose an embarrassingly %
simple baseline for TDG that %
matches diffusion-based \sota models while being one or two orders of magnitude faster (see \cref{fig:pc-global}):  \textit{Probabilistic Circuits} (PCs; \citealt{vergari2020probabilistic,choi2020pc,darwiche2003differential}),
which can be understood as the hierarchical version of %
mixture models 
\citep{mclachlan2019finite}, and as the generative equivalent of decision trees and forests \citep{correia2020joints,khosravihandling}. As such, PCs inherit the %
ability to seamlessly handle heterogeneous data and scarce sample sizes.
Moreover, PCs with tree-like structures have been extensively trained on tabular data in the past \citep{molina2018mixed,vergari2019automatic,watson2023arf}, especially on binary data and in the form of sparse computational graphs running on the CPU \citep{zhao2016unified,di2017fast,dang2020strudel}.
However, modern PCs with tensorised structures \citep{peharz2019ratspn,peharz_einsum_2020} and better learning algorithms \citep{mari2023unifying,loconte2025what} 
have never been evaluated for TDG. We argue that the TDG community has been missing a strong baseline which, as we show, can easily outperform more recent DGMs, while providing tractable inference.

Our contributions can therefore be summarised as follows.
\textbf{C1)} We question the effectiveness of standard metrics for evaluating TDG such as \ctwost and \trend \citep{sdmetrics} in \cref{sec:sota-tab-gen}, and propose alternatives that better illustrate how the current \sota stands.
\textbf{C2)} We introduce \ours as tensorised PCs for TDG in \cref{sec:tab-pc}, by showing how it is simple to learn modern PC architectures \citep{mari2023unifying,loconte2025what,liu2021tractable} on heterogeneous data.
Finally, \textbf{C3)} we show in a rigorous and extensive set of experiments in \cref{sec:exps} that \ours can \textit{outperform and compete} with current \sota DGMs for TDG under fidelity and utility metrics \textit{in a fraction of the time}.

\section{The \sota of Tabular Data Generation(?)}
\label{sec:sota-tab-gen}

Before discussing the limitations in commonly-used evaluation protocols for measuring data fidelity, we briefly introduce the current \sota of DGM for TDG.
{%
GAN-based approaches started the trend of DGMs for TDG with methods such as \ctgan \citep{ctgan},
concurrently with VAE approaches such as \tvae \citep{ctgan} and \hivae \citep{hivae,javaloy2022mitigating}.
More recently, diffusion-based approaches such as \codi \citep{codi}, \stasy \citep{stasy} and \cdtd \citep{mueller2025continuous} were introduced.

At the time of writing, the current \sota for DGMs for TDG is given by diffusion-based approaches like \tabdiff \citep{tabdiff} and \tabsyn \citep{tabsyn}, 
which provide more sophisticated ways to perform diffusion, but we remark that TDG is an actively growing field \citep{mueller2026cascaded}.
}
{Recently, foundation models such as \tabpfn \citep{hollmann2023tabpfn} and TabICL \citep{qu2025tabicl} have gained popularity for discriminative tabular tasks. Extensions of \tabpfn\footnote{\url{https://github.com/PriorLabs/tabpfn-extensions}} provide functionality for the generation of tabular data. To the best of our knowledge, as of yet this has seen only limited benchmarking in one concurrent work, focused primarily on causally-informed generation and not the general task \citep{tugnoli2026improvingtabpfnssyntheticdata}.}

{One aspect that sets tabular data apart from modalities such as image or text, is that \textit{there is no clear way to qualitatively assess sample quality}. 
We cannot simply look at a generated row and say ``yes, that looks like a realistic sample''.}
It is therefore crucial to develop meaningful metrics to evaluate synthetic data quality. Recent %
works \citep{tabsyn,tabbyflow, tabdiff} use a number of metrics that can be grouped along these axes: 
\textit{i)} statistical similarity to the original data (\textit{fidelity}); %
\textit{ii)} performance on downstream tasks, when using synthetic data to train a regressor/classifier and evaluating on real data (\textit{utility} of synthetic data); 
and \textit{iii)} protection of synthetic data against leaking sensitive information (\textit{privacy}) 
\citep{stoian2025surveytabulardatageneration}.
We now focus on fidelity metrics such as \trend and \ctwost, as they are arguably the ones most prominently used to assess the \sota for TDG, and privacy metrics have already been critically analysed in \citealt{yao_dcr_2025}.
We will discuss the evaluation of utility metrics in \cref{app:mle}.

\subsection{Perfect Fidelity?}\label{sec:ff_intro}
\sota DGMs such as \tabdiff \citep{tabdiff} and \tabsyn \citep{tabsyn} have been reported to achieve \textit{almost perfect scores} in \trend and \ctwost, i.e., values close to 1, as we also show in \cref{tab:density-Trend_updated,tab:logistic regression(C2ST)_updated},
thus suggesting almost-perfect generation.
However, as we show next, this does not mean that  data they generate is really indistinguishable from original data (\ctwost), nor that it truly recovers feature dependencies (\trend).
To this end, we consider the simplest generative model: a \textit{fully factorised} (FF) model, assuming each feature (column) to be independent, defined as  %
\begin{equation}\label{eq:ff_model}\tag{FF}
    p_{\mathrm{FF}}(\boldsymbol{x}) = \prod\nolimits_{i=1}^D p_i(x_i)
\end{equation}
where $\boldsymbol{x}\coloneqq\{x_1, \ldots, x_D\}$ denote the set of random variables representing the table features,
and each $p_i$ is modelled as a Gaussian distribution for continuous features and a categorical distribution for discrete features.
This yields a parsimonious model, with only $\mathcal{O}(D)$ trainable parameters.
Maximum likelihood estimation (MLE) training is almost instantaneous as we just need to sweep over the training data once to compute sufficient statistics (after the standard data preprocessing done for DGMs detailed in \cref{app:preprocessing}).

Clearly, the \FF model is fundamentally limited, since \textit{it cannot generate correlated features by design}, and one would expect its fidelity as measured by \trend and \ctwost to be low, except for trivial datasets.
However, this turns out not to be the case. 
On \trend, the \FF model is able to achieve competitive or \textit{occasionally superior} scores to \sota diffusion-based approaches. For example, on the Diabetes dataset (further details in \cref{sec:exps} and \cref{app:datasets}), \tabdiff scores $0.9711$ and \tabsyn $0.9593$, while \FF scores $0.9686$, nearing the perfect score of $1$. \FF beating a diffusion-based model in a metric supposedly measuring dependencies is clearly problematic. %
Similarly, \FF achieves a \textit{perfect} \ctwost utility score of $1$ on the Diabetes dataset, in comparison with $0.9591$ for \tabdiff and $0.6476$ for \tabsyn.
This problem is exacerbated by the fact that these two metrics are used in several recent works to evaluate TDG of \sota models \citep{tabsyn, tabdiff, tabbyflow}. Next, we investigate how and why exactly \FF is able to `fool' these metrics. %

\subsection{Off trend: modelling dependencies}
\label{sec:trend}

To understand why a simple \FF model is able to match \sota DGMs for \trend, we show that this metric is sensitive to the quality of the model's univariate marginals.
We later propose an alternative  that resolves this issue.
\trend computes the average similarity of correlations between feature pairs in both real and synthetic data. 
Specifically, for each pair of continuous features  $x_i, x_j$, it computes 
\begin{equation}\label{eq:trend_pearson}
   s_{\mathrm{corr}}(x_i,x_j) = 1 - 0.5 |\rho_R(x_i,x_j) - \rho_S(x_i,x_j)|,
\end{equation}
where $\rho_R(x_i,x_j)$ is the Pearson correlation between columns $x_i$ and $x_j$ in the real data, and similarly for $\rho_S$ in the synthetic data. 
For each pair of categorical features $x_i, x_j$, it instead computes the total variation distance between the real and synthetic distributions, given by
\begin{equation}\label{eq:trend_contingency}
    s_{\mathrm{contingency}}(x_i,x_j) = 1 - \frac{1}{2}\sum_{\alpha \in x_i} \sum_{\beta \in x_j} |R_{\alpha, \beta} - S_{\alpha, \beta}|,
\end{equation}
where $R_{\alpha, \beta}$ denotes the empirical probability of jointly observing $x_i = \alpha$, $x_j = \beta$ in the real data, and similarly for $S_{\alpha, \beta}$ in the synthetic data.
For pairs of mixed continuous and categorical features, it first discretises the numerical column and then computes the contingency similarity. The overall \trend score is then given by the average of the scores of all unique column pairs.

As one can clearly see, \trend \textbf{i)} only measures linear correlation among continuous features and thus fails to capture more complex dependencies.
More crucially, however, \textbf{ii)} it can be fooled by the quality of univariate marginals.
In fact, a \FF model with randomised parameters scores less on \trend than a \FF model learned by MLE (see \cref{fig:random_FF_vs_trained_FF}), while both should ideally score the same values for a measure of pairwise dependencies, which both models cannot capture.
As a result, \trend captures the marginal quality more than pairwise dependencies.\footnote{\shape is the fidelity metric used to measure univariate marginal quality \citep{sdmetrics}, and as expected a MLE-trained \FF model scores almost perfectly on it as well (see \cref{tab:density-Shape_updated} in our experiments).}
Instead, an ideal metric for bivariate dependencies should be invariant to marginal transformations, on top of capturing rich non-linear dependencies.

\begin{figure}[!t]
    \centering
    \includegraphics[width=0.9\linewidth]{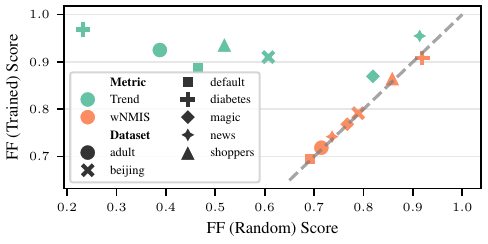}
    \caption{
    \textbf{\wnmis is invariant to the quality of the univariate marginals while \trend is not}, as shown by the fact that it assigns almost identical scores to \ref{eq:ff_model} models trained via MLE and to \ref{eq:ff_model} models with randomly initialised parameters across all datasets (denoted by marker shape), lying on the identity denoted as the grey dashed line. 
    In contrast, \trend says that a trained \FF model captures bivariate dependencies well, and an untrained model does not.
    }
    \label{fig:random_FF_vs_trained_FF}
\end{figure}

To overcome the limitations \textbf{(i-ii)} of \trend, we follow \citealt{tgan, yang2024structuredevaluationsynthetictabular} who proposed to use mutual information (MI) to measure bivariate dependencies.
Specifically, we propose a \textit{weighted normalised mutual information (NMI) similarity (\wnmis)} that weights NMI by emphasising which feature pairs have high NMI, and penalises those with low NMI.
For features $x$ and $y$, the NMI  \citep{witten_data_2011} is defined as
\begin{equation}\label{eq:NMI}
    \mathrm{NMI}(x,y) = {2\mathrm{I}(x,y)}/{(\mathrm{H}(x)+\mathrm{H}(y))} \in [0,1],
\end{equation}
where $\mathrm{I}$ and $\mathrm{H}$ refer to the mutual information and entropy, respectively.
Specifically, we compute wNMIS as follows:
\begin{align}\label{eq:wNMIS}
    &\sum\nolimits_{\{x_i,x_j \, : \, i<j\}}\widetilde{w}(x_i,x_j)\mathrm{NMIS}(x_i,x_j) \in [0,1],
\end{align}
where $x_i, x_j$ are two different columns of the dataset, and
\begin{align*}
\mathrm{NMIS}(x_i,x_j) &\coloneqq 1 - |\mathrm{NMI}_R(x_i,x_j) - \mathrm{NMI}_S(x_i,x_j)| 
\\
w(x_i,x_j) &\coloneqq |\mathrm{NMI}_R(x_i,x_j) + \mathrm{NMI}_S(x_i,x_j)|, 
\\
\widetilde{w}(x_i,x_j) &\coloneqq {w(x_i,x_j)}/{\sum\nolimits_{i<j}w(x_i,x_j)}
\end{align*}
where 
$\mathrm{NMI}_R$ denotes the NMI between features in the real dataset, and $\mathrm{NMI}_S$ in the synthetic one. 
We can estimate \wnmis as
$\widehat{\mathrm{wNMIS}}$ by replacing $\mathrm{I}$ and $\mathrm{H}$ in \cref{eq:NMI} with their empirical estimators after discretising continuous features.
{ We provide an ablation study on the effect of different discretisation schemes on \wnmis score in \cref{app:wnmis_discretisation}.}
This score is measured from 0 (worst) to 1 (best), thereby matching the semantics of \trend.
Crucially, it does not suffer from issue \textbf{(ii)} of \trend, as discussed next.

Its invariance to univariate marginal quality is reflected empirically in our previous experiment with a randomised \FF model.
\Cref{fig:random_FF_vs_trained_FF} shows the \wnmis scores are the same, up to numerical precision, for both \FF models.
Note that the \wnmis score of a \FF model is not zero but highlights how complex the bivariate interactions in a dataset are, in an analogous way in which the MI of a feature with itself is its entropy.

\subsection{Better classifiers, lower fidelity}
\label{sec:c2st}

The next fidelity metric we consider, \ctwost  \citep{lopez-paz2017revisiting}, trains a classifier to distinguish real from synthetic data, with increasing classifier performance implying lower data quality. This is based on the idea that synthetic data should be indistinguishable from real data. 
\ctwost is defined as follows. Given real and synthetic datasets $\mathcal{D}_R$ and $\mathcal{D}_S$ (respectively), with $|\mathcal{D}_R| = |\mathcal{D}_S|$, one forms balanced training ($\mathcal{D}^{\mathrm{(train)}}$) and test ($\mathcal{D}^{\mathrm{(test)}}$) sets out of concatenating $\mathcal{D}_R$ and $\mathcal{D}_S$, then trains a binary classifier $c$ to determine synthetic samples on $\mathcal{D}^{\mathrm{(train)}}$, and finally computes its area under the receiver operating  curve (AUROC) \citep{hanley1982meaning} on $\mathcal{D}^{\mathrm{(test)}}$. The overall \ctwost  score, denoted as $\mathrm{C2ST}(c, \mathcal{D}_R, \mathcal{D}_S)$, is then given by 
\begin{align}\label{eq:c2st}
\tag{\ctwosttt}
    1-(2 \cdot \max(\mathrm{AUROC}(c, \mathcal{D}^{\mathrm{(test)}}),0.5) - 1).
\end{align}
\FF models are able to achieve high \ctst scores because it is standard to use Logistic Regression (LR) as the classifier, $c$, as seen in recent \sota work \citep{tabdiff, tabsyn, tabbyflow} and the library \texttt{SDMetrics} \citep{sdmetrics}.
Intuitively, these unusually high scores may happen because a LR is not powerful enough as a classifier. This intuition is formalised in the theorem below.

\begin{theorem}\label{thm:ff_fool_lr}
    Let $\mathcal{D}_R$ with $|\mathcal{D}_R| = n$ be a real dataset, and let $p$ be an \FF model trained by MLE on this dataset.
    Then
    \begin{equation*}
       \lim_{n \to \infty} \mathbb{E}_{\mathcal{D}_S}[\ctst(LR, \mathcal{D}_R, \mathcal{D}_S)] = 1,
    \end{equation*}
    where $\mathcal{D}_S$ is an i.i.d. sample of $n$ items drawn from $p$.
\end{theorem}

We prove this result in \cref{app:lr_c2st_proof} by noting that any distribution that recovers the data univariate marginals will maximise \ctst implemented with a LR classifier.
This, in addition to our results on \trend being insufficient in capturing bivariate dependencies, begs the question: are \sota DGMs just recovering the data marginals?
The short answer is no (see \cref{sec:exps}), but a longer version reveals that under more robust metrics, these models' performance is far from optimal. 
We now look into estimating \ctst more robustly.

Some prior \citep{foolxgboost} and concurrent \citep{KINDJI2025131655, mueller2026cascaded} works have used more powerful boosted tree classifiers such as XGBoost \citep{xgboost} instead of LR. However, they do not adopt the same evaluation protocol as \sota DGMs and thus do not observe the issue of using LR in the context of evaluating DGMs.
We therefore emphasise here that \textbf{\textit{logistic regression is not a suitable choice of classifier for \ctst and we want to popularise the use of XGBoost for \ctst in TDG benchmarks}}.

\begin{figure}[!t]
    \centering
    \includegraphics[width=.9\linewidth]{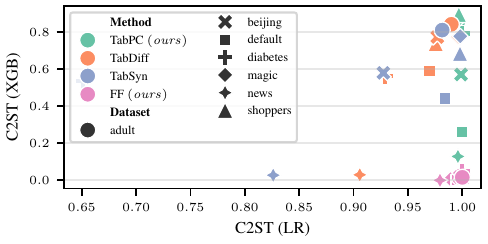}
    \caption{
    \textbf{\ctwost with XGBoost offers a much clearer stratification of model performance than with LR} across all datasets, and is able to correctly separate out the performance of the trivial \ref{eq:ff_model} model.
    Details in \cref{sec:trend}.
    }
    \label{fig:xgb_vs_lr}
\end{figure}

In fact, using an insufficiently powerful classifier for C2ST has significantly inflated our perception of how well we are creating `indistinguishable' data.
For example, \tabdiff's \ctwost on the Adult dataset goes from an almost perfect $0.9899$ using LR to $0.8409$ using XGBoost,
while a \FF model goes from $1.0$ to $0.0160$. 
More concerningly, \tabdiff also goes from $0.9057$ to a dismal $0.0296$ on the News dataset. It is clear then that we still have some way to go, as supported by \cref{fig:xgb_vs_lr} and complete results in \cref{tab:logistic regression(C2ST)_updated,tab:xgboost(C2ST)_updated}.

The \FF model has served as a cheap and simple diagnostic benchmark, uncovering flaws in existing fidelity metrics by achieving undeservedly high scores.
In the next section, we start from the \FF model to devise an expressive and efficient approach to TDG, that retains simplicity and efficiency while delivering competitive performance on stronger metrics such as \wnmis and \ctwost with XGBoost.

\section{From FF Models to DEEP \ours}
\label{sec:tab-pc}

\subsection{From FF to Shallow Mixtures\ldots}
\label{sec:sm}

A natural way to reintroduce dependencies among features, starting from a \FF model, is to combine several of them into a  \textit{shallow mixture \emph{(SM)}} \citep{mclachlan2019finite}, defined as
\begin{equation}\label{eq:sm_model}
\tag{\small\textsc{SM}}
    p_{\mathrm{SM}}(\boldsymbol{x}) = \sum\nolimits_{i=1}^K w_i  \prod\nolimits_{j=1}^D p_{i,j}(x_j)
\end{equation}
where $w_i \in \mathbb{R}_{>0}, \, \sum_i w_i = 1$ are additional learnable \emph{mixture weights} determining the contributions of the $K$ mixture components.
The base distributions, $p_{i,j}$, can be Gaussians or categoricals, depending on the feature type.
Compared to \FF models, the number of parameters will linearly increase to $\mathcal{O}(KD)$, and MLE training can be done via expectation maximisation (EM) or stochastic gradient ascent, as \SM models introduce a categorical latent variable with $K$ states that is marginalised out to compute $ p_{\mathrm{SM}}(\boldsymbol{x})$ \citep{peharz_latent_2017}.
This latent variable interpretation also enables efficient exact sampling: one first samples the latent variable state proportionally to the mixture weights, and then samples from the selected component directly \citep{peharz_latent_2017}.

Despite introducing a single latent variable, \SM models are universal density approximators in the limit of infinite mixture components \citep{mclachlan2019finite}. 
That is, by increasing $K$ we can increase model expressiveness.
For example, one can see an increase in \ctst with XGB, from $0.0160$ on Adult as scored by a \FF model  to $0.7727$ for a \SM with $K=20,000$. 
This is remarkable for such simple baselines, as TabDiff, having 
$7.87\times$ more parameters
scores $0.8554$ (see \cref{sec:c2st}).
We study the results of SM models in greater detail in \cref{app:sm_model_results}.
Next, we discuss how to use these ideas to build \textit{hierarchical and overparameterised mixture models} in the form of PCs.
As per the deep learning recipe, we will increase expressiveness by building \textit{deeper}, and \textit{not wider}, mixture models \citep{martens2014expressive,choi2020pc}.

\subsection{\ldots and to deep Probabilistic Circuits}
\label{sec:pcs}

Before detailing our `recipe' for building and learning of a PC for tabular data (\ours), we briefly review PCs.
PCs provide a framework to model  hierarchical mixture models as deep computational graphs over which one can systematically trade-off expressiveness for tractability by governing the number of parameters in them and certifying that certain structural properties of the graph are met \citep{vergari2019tractable}. 
Furthermore, many other classical tractable probabilistic models such as HMMs and trees \citep{correia2020joints,watson2023arf}, as well as tensor factorisations, are instances of PCs \citep{choi2020pc,correia2020joints,khosravihandling,loconte2025what}.

Formally, a \textit{circuit} \citep{darwiche2003differential,choi2020pc,vergari2021compositional} $c$ is a parameterised directed
acyclic computational graph over variables $\boldsymbol{x}$ encoding a function $c(\boldsymbol{x})$. A circuit comprises three kinds of computational units: \textit{input}, \textit{product}, and \textit{sum}.
Each sum or product unit $n$ receives the outputs of other units as inputs, denoted as the set $\inscope(n)$. Each $n$ encodes a function $c_n$ defined as:
    (i) a tractable function $f_n(\scope(n); \theta)$ if $n$ is an \textit{input unit} with parameters $\theta$, defined over variables $\scope(n) \subseteq \boldsymbol{x}$, called its \textit{scope};
    (ii) $\prod_{j\in\inscope(n)} c_j(\scope(j))$ if $n$ is a \textit{product unit}; and
    (iii) $\sum_{j\in\inscope(n)} w_{n,j} c_j(\scope(j))$ if $n$ is a \textit{sum unit}, where each $w_{n,j}\in\mathbb{R}$ is a parameter of $n$.
The scope of a sum or product unit is the union of the scopes of its inputs, i.e. $\scope(n) = \bigcup_{j\in\inscope(n)} \scope(j)$.
Then, a \textit{probabilistic circuit} (PC) is a circuit $c$ encoding a non-negative function, i.e., $c(\boldsymbol{x})\geq 0$ for any $\boldsymbol{x}$, thus encoding a (possibly unnormalised) probability distribution $p(\boldsymbol{x})\propto c(\boldsymbol{x})$.
In a PC, input units can model probability densities (e.g., Gaussians) or masses (e.g., categorical distributions) \citep{molina2018mixed}.
Note that both the FF model (in \cref{eq:ff_model}) and SM model (in \cref{eq:sm_model}) are special cases of PCs, as they can be represented as simple computational graphs in this language, as shown in  \cref{fig:circuit_diagrams_cirkit}.

A PC $c$ supports the tractable marginalisation 
of any subset of its variables, and hence also renormalisation, in a single forward step \citep{choi2020pc} 
if
    (i) its input functions
    $f_n$ can be integrated tractably, and 
    (ii) it is \emph{smooth} and \emph{decomposable} \citep{darwiche2002knowledge}.
A circuit is \emph{smooth} if for every sum unit $n$, all of its input units depend on the same variables, i.e. $\forall i,j\in\inscope(n)\colon \scope(i) = \scope(j)$. A circuit is \emph{decomposable} if the inputs of every product unit $n$ depend on disjoint sets of variables, i.e. $\forall i,j\in\inscope(n)\ i\neq j\colon \scope(i)\cap\scope(j) = \emptyset$. 
Throughout this work, we will assume all our PCs are smooth and decomposable by construction (see \cref{sec:building-tab-pc}), which  allows the interpretation of  PCs as hierarchical latent variable models \citep{peharz_latent_2017,gala2024pic,gala2024scaling}. 

\textbf{Beyond just sampling for tabular data.} Therefore, smoothness and decomposability also enable exact sampling in time linear in the circuit size, i.e., the number of edges in it, which can be done via ancestral sampling by recursively generalising the way in which \SM models are sampled, as detailed in \cref{app:pc_sampling}.
More crucially, smoothness and decomposability will allow \ours to perform certain inference tasks that are out of the reach of all the other \sota DGMs.
First, we can compute exact normalised likelihood scores, which we can use for model selection in our experiments 
(\cref{sec:exps}).
Second, one can handle missing values `on the fly' both at inference and training time.
Therefore, \ours can also perform exact conditional sampling for \textit{any} subset of provided evidence configuration, \textit{without} retraining.
We evaluate these additional capabilities of \ours in \cref{sec:tabpc_extensions}.

\begin{figure*}[!t]
	\centering
	\begin{minipage}{0.48\textwidth}
		\centering
		\begin{subfigure}{\textwidth}
			\centering
			\includegraphics[scale=0.55, page=1]{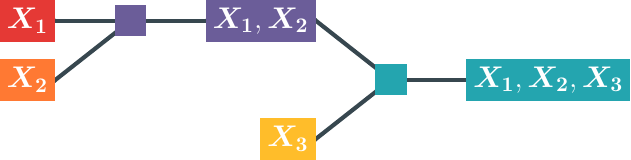}
			\caption{region graph\vspace{10pt}}
			\label{fig:sub1}
		\end{subfigure}
		
		\begin{subfigure}{\textwidth}
			\centering
			\includegraphics[scale=0.6, page=2]{uai-figures/circuit_figures/simple_circuit.pdf}
			\caption{underparameterised PC $K=1$}
			\label{fig:sub2}
		\end{subfigure}
	\end{minipage}
	\begin{minipage}{0.48\textwidth}
		\centering
		\begin{subfigure}{\textwidth}
			\centering
			\includegraphics[scale=0.6, page=3]{uai-figures/circuit_figures/simple_circuit.pdf}
			\caption{overparameterised PC with $K=3$}
			\label{fig:sub3}
		\end{subfigure}
	\end{minipage}
	\caption{
    \textbf{Constructing \ours requires three choices: the region graph (RG), level of overparameterisation, and type of sum-product layers.} \cref{fig:sub1} shows a tree-shaped RG over three variables. This acts as a template from which we construct the simple circuit shown in \cref{fig:sub2}. To increase expressivity, we \textit{overparameterise} the circuit by populating it with $K$ units organised in \textit{layers}, as seen in \cref{fig:sub3} for $K=3$. Finally, the choice of sum-product layer dictates how we connect units across layers. \ours uses CP sum-product layers, also shown in \cref{fig:sub3}, which are described in detail below.
    }
	\label{fig:simple_circuit_pipeline}
\end{figure*}

\subsection{Building \ours}
\label{sec:building-tab-pc}

Several past works  have learned PCs for tabular data, with the majority focussing on learning their structure, i.e., the edges in their  graphs \citep{vergari2015simplifying,molina2018mixed,vergari2019automatic,di2017fast,watson2023arf} and parameters \citep{zhao2016unified} but assuming their computational graphs to be trees instead of directed acyclic graphs (DAGs).
More recent works shifted to learning PCs over image and text modalities \citep{anji2022scaling-pcs-lvd,gala2024pic,
zhang2025scaling},  prescribing the structure of PCs to be a fixed but \textit{overparameterised} DAG, and learning only the parameters \citep{peharz2019ratspn,peharz_einsum_2020,liu2024scaling,loconte2025what,suresh2025tractable}.
As such, modern tricks to scale building and learning PCs have not been explored for TDG, yet.

In this section, we detail the construction of \ours, which uses modern recipes for overparameterised DAG PCs. 
Specifically, we piggyback on the  ``Lego block'' approach to construct smooth and decomposable tensorised PCs  \citep{loconte2025what}, which is summarised in the following `recipe': (i) the \textbf{choice of region graph (RG)},  (ii) the \textbf{level of overparameterisation}, and (iii) the \textbf{choice of sum-product layers}.
We remark that few adaptions must be made to unlock TDG for PCs; the flexibility of the framework enables their use almost `out of the box', highlighting how little is needed to yield a strong baseline that resets the \sota for TDG.

\textbf{(i) Region Graph (RG):}
Given the set of features of a table, a RG is a template to build PCs that are smooth and decomposable by design \citep{peharz2019ratspn}.
A RG is a  bipartite DAG comprising regions, i.e., sets of features, and partitions, describing a hierarchical partitioning of these features.
Each region and partition in a RG will serve as the basis of a \textit{layer} in a PC.
\cref{fig:sub1} illustrates a tree RG over three features and
\cref{app:chow_liu} formalises its construction.

Several approaches to RG construction have been proposed in the literature \citep{loconte2025what}.
For \ours, we use a method which builds a RG from a Chow-Liu tree \citep{chow1968clts} learned on the training data \citep{dang2021strudel,liu2021tractable}. 
This is done by computing the MI between feature pairs and iteratively adding the maximum MI feature pair to the tree. The constructed tree can then be compiled directly into a RG; see \citealt{loconte2025what} and \cref{app:chow_liu} for more details. 
We then build our PC according to the RG, associating input layers to leaf regions, and product and sum layers to partitions and inner regions in the RG. 
Each layer is a logical abstraction that can correspond to a set of units in the circuit.

\textbf{(ii) Overparameterisation:}
In principle, one can use a single unit per layer.
This would build a simple circuit obeying our desired structural constraints. \cref{fig:simple_circuit_pipeline} provides an example of this process.
Alternatively, one can \textit{overparameterise} the layered PC:
in the same way we increase the number of mixture components $K$ in a \SM model, we populate each node of the RG with not just one, but $K$ sum, input, and product units.
This allows for computations to be executed in parallel, thereby increasing efficiency. In practice, these layers are computed as tensors, i.e. are \textit{tensorised}.
More crucially, deep overparameterised circuits are more expressive efficient than shallow mixtures, i.e., using $K$ units per layer yields the expressivity of a mixture with $K^{N/2}$ components, where $N$ is the depth of the PC \citep{choi2020pc}. 
In our experiments, we use $K$ on the order of $10^3$ units per layer (specific values per dataset can be found in \cref{app:tabpc_implementation}).
However, this will not be a drawback when it comes to speed and performance, as shown in \cref{sec:exps}.

\textbf{(iii) Sum-Product Layer:}
After we decide on the number of units $K$, we need to connect them across layers.
For \ours, we use \textit{CP sum-product layers}, described in \citealt{loconte2025what}. The CP sum-product layer computation is defined as follows: if the scope of CP-layer $\bm{\ell}$ is $\bm{y} \subseteq \bm{x}$, and we assume it has two input layers $\bm{\ell}_1, \bm{\ell}_2$ with disjoint scopes $\bm{y}_1, \bm{y}_2$, with $\bm{y}_1 \cup \bm{y}_2 = \bm{y}$, then the CP-layer $\bm{\ell}$ computes
\begin{equation}\label{eq:cp}
    \bm{\ell}(\bm{y}) = \left(\bm{W}^{(1)}\bm{\ell}_1 (\bm{y}_1)\right) \odot \left( \bm{W}^{(2)}\bm{\ell}_2 (\bm{y}_2) \right)
\end{equation}
for some weight matrices $\bm{W}^{(i)} \in \mathbb{R}^{K \times K}, \, i \in \{1,2\}$ 
and where $\odot$ denotes the Hadamard (element-wise) product.

Locally, we can view this as a tensor of $K$ distributions, each of which factorises into two shallow mixture models over their respective variables $\bm{y}_1, \bm{y}_2$. Here, smoothness gives us the interpretation of a mixture model, since the scopes of all mixture components are the same, and decomposability gives us the factorisation, since the scopes of both sum layers are disjoint. By stacking these on top of each other according to the RG, as can be seen in \cref{fig:simple_circuit_pipeline}, we build a deep, hierarchical mixture model---i.e., a PC.
\cref{app:cp_layers} details how the evaluation of these tensorised \ours can be 
sped up further on the GPU.
Once built, one can train PCs by MLE either by stochastic gradient ascent, or by EM \citep{peharz_latent_2017, peharz_einsum_2020}.
In our experiments, we train \ours using the RAdam optimiser \citep{liu_variance_2021}.

In summary, we use \ours to denote our specific implementation for TDG: an overparameterised PC built from a CLT RG and using CP layers.

\section{Re-evaluating \sota for TDG}
\label{sec:exps}

We now  evaluate \ours against  \sota DGMs models on a range of popular benchmarks for TDG, starting from the experimental protocol of \citealt{tabdiff, tabsyn, tabbyflow}, but updating it after our observations in \cref{sec:sota-tab-gen}.
We aim to answer the following questions:
\textbf{Q1)} is \ours competitive in terms of fidelity,  utility and time?
\textbf{Q2)} how much does the \sota change when we use \wnmis and \ctwost with XGBoost for assessing fidelity? 
\textbf{Q3)} can we effectively harness tractable inference with \ours for conditional sampling and model selection?
{ \textbf{Q4)} how do \ours and existing \sota DGMs compare against the generative extension of foundation models such as \tabpfn?}

\textbf{Datasets:} 
We evaluate on the commonly used UCI datasets with heterogeneous features: \textit{Adult}, \textit{Beijing}, \textit{Default}, \textit{Diabetes}, \textit{Magic}, \textit{News}, and \textit{Shoppers}.
\cref{app:datasets} details their statistics and \cref{app:preprocessing} describes their pre-processing.

\textbf{Baselines:}
We follow \citealt{tabsyn,tabdiff} and compare against \sota DGMs for TDG, including 
\ctgan, \tvae, \great \citep{great} and \stasy, \codi, \tabsyn, and \tabdiff. See \cref{sec:sota-tab-gen} and
\cref{app:baseline_details} for details.
{Additionally, to answer Q4 we compare against the data generation extension of \tabpfn v3 (\cref{app:tabpfn_details}).}

\textbf{Metrics:}
We evaluate against the following metrics (organised by category). Each is described in greater technical detail in \cref{app:metrics}.  (i) \textbf{Fidelity}: \shape, \trend, and \ctwost (LR) (implemented using \texttt{SDMetrics} \citep{sdmetrics}), and \textit{$\alpha$-precision} and \textit{$\beta$-recall} \citep{faithful} (implemented using \texttt{Synthcity} \citep{synthcity}). 
We additionally evaluate against our promoted metrics \ctwost (XGB) and \wnmis. 
Due to space constraints we push to \cref{app:mle} the discussion about
(ii) \textbf{Utility}: \textit{machine learning efficacy}
\citep{ctgan} (re-using the implementation of \citealt{tabdiff}).
We provide our results, as well as a further discussion of interpreting and evaluating performance on this metric.
(iii) \textbf{Privacy}: \cref{app:dcr} discusses our results on distance-based privacy metrics, such as the commonly used \textit{distance to closest record (DCR)}. However, we remark that DCR has been strongly criticised in \citealt{yao_dcr_2025} as ``flawed by design'', and hence we do not focus in detail on these results.

\textbf{Evaluation protocol:} 
Differently from  \citealt{tabsyn,tabdiff}, who train a single model per dataset and use it to generate several data samples, we rerun each baseline with five seeds, and report averaged metrics with proper standard deviations.
{This is with the exception of \tabpfn, which does not require retraining or fine-tuning, and hence we compute its metrics averaged over five generations (where possible).}
Apart from reporting full results and average ranks in \cref{app:tables}, we also compute critical difference diagrams (CDDs), which 
provide an indication of whether the performance of two methods are statistically different from each other. \cref{app:cdds} collects all CDDs.

 \subsection{Q1 + Q2: Fidelity Metrics \& TIME }\label{sec:exp_fidelity}

\textbf{\shape} is designed to measure how well synthetic data models the univariate marginals. 
\cref{tab:density-Shape_updated} shows that 
\ours outperforms all previous \sota models in six out of the seven datasets, consistently achieving scores above $0.99$, close to the maximum value of $1$. 
The saturation of this metric suggests that all recent \sota models accurately recover univariate marginals, and scoring well on \shape is necessary but not sufficient to generate realistic data.

\textbf{\trend and \wnmis:} As we have noted, the \trend metric is problematic, and so we instead consider \wnmis scores to measure bivariate dependency quality (for reference, \trend results are reported in \cref{tab:density-Trend_updated}, and \wnmis results in \cref{tab:nmi_l1_weighted complement_updated}. 
In terms of \wnmis, \ours outperforms existing \sota models in four out of seven datasets, and achieves scores greater than $0.99$ in five datasets.
Again, metric saturation suggests that \sota models already capture low-order correlations relatively effectively, and investigating higher-order information would be required for better stratification of model performance.

\textbf{\ctwost (LR / XGB):} 
In comparison to \shape and \trend~/ \wnmis, \ctwost considers samples from the full joint, and not just univariate or pair-wise statistics.
However, as we discussed in \cref{sec:c2st}, \ctwost (LR) is flawed, and hence we instead consider the results using XGBoost as the classifier (LR results are reported in \cref{tab:logistic regression(C2ST)_updated} for reference, and XGB results in \cref{tab:xgboost(C2ST)_updated}). In terms of \ctwost (XGB), \ours offers competitive performance, outperforming the diffusion-based \sota in four out of the seven tested datasets. 
We note that all generally high-performing models struggle on the News dataset, which has a large number of numerical features (i.e. 46 out of 48 features are numerical) and skewed distributions, and remark again that under \ctwost (LR) this struggle was not visible.

\textbf{\alphaprecision and \betarecall:} These metrics extend the notion of precision and recall to generative models \citep{faithful}. \alphaprecision measures synthetic data `realism', whereas \betarecall measures synthetic data `diversity'. 
Full results are in \cref{tab:alpha precision oc_updated} for \alphaprecision and \cref{tab:beta recall oc_updated} for \betarecall. The CDDs in \cref{fig:ap_br_cdd} show that
\ours is in the top clique for both \alphaprecision and \betarecall along with \tabdiff and \tabsyn (and a few other methods for \alphaprecision), suggesting that \ours offers competitive fidelity and diversity of samples to these methods. Note that backwards compatibility of these results with specific values reported by certain prior works is problematic; see \cref{app:ap_br_discrepancy}.

\textbf{Time and discussion:} \ours is able to deliver the performance described above one or two orders of magnitude faster than other DMGs in terms of training time; see \cref{fig:pc-global} and \cref{tab:training time_updated}.
We remark that our implementation can be further sped up by leveraging recent advancements in scaling PCs over GPUs \citep{liu2024scaling,zhang2025scaling}.
{Moreover, we note that \ours provides the best ranked dataset sampling time of the current \sota TDG methods, in particular beating \tabdiff and \tabsyn in six out of the seven datasets; see \cref{tab:sampling time_updated}.}

\subsection{Q3: beyond unconditional sampling}\label{sec:tabpc_extensions}

As discussed in \cref{sec:tab-pc}, \ours allows for the efficient computation of exact likelihoods and arbitrary marginals. 
We now evaluate these capabilities for model selection and conditional sample generation.

\textbf{Model selection via likelihood:} 
We investigate if there is a correlation between the fidelity metrics discussed above with the validation likelihood for \ours.
For high dimensional data such as images, high likelihood does not always corresponding to good sample quality \citep{theis2015note,lang2022elevating,braun2025tractable}.
As we show next, this is not the case for tabular data.
Note that this type of analysis is only possible with PCs, as all previous \sota models do not allow for exact likelihood computation.

Concretely, we compute the validation bits-per-dimension (BPD, see \cref{app:ll_bpd}) and downstream sample quality, as measured by \ctwost (XGB) for the different \ours models we trained with various hyperparameters---specifically, number of sum and input units ($K$), learning rate, and batch size.
\cref{fig:bpd_vs_xgb_c2st} shows that decreasing BPD (corresponding to an increased log-likelihood) strongly correlates  with increasing sample \ctwost (XGB) (and similar plots for all datasets can be found in \cref{app:ll_expts}). 
As PCs struggle to generate high-quality image samples while delivering low bpds \citep{lang2022elevating,braun2025tractable}, we highlight that tabular data is a better modality for this kind of generative model.

\begin{figure}[!t]
    \centering
    \includegraphics[width=0.95\linewidth]{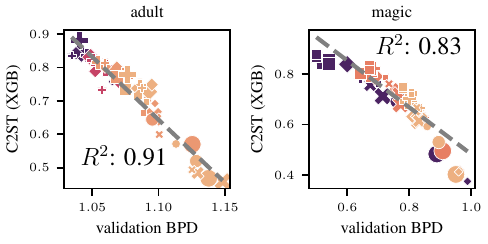}
    \caption{\textbf{Validation bits-per-dimension (BPD) provides a strong signal of downstream sample quality} as measured by \ctwost (XGB), as displayed here for the Adult \textit{(left)} and Magic \textit{(right)} datasets (full results in \cref{app:ll_plots}). 
    Each point represents a different hyperparameter configuration for \ours, across number of sum and input units (colour), batch size (marker size), and learning rate (marker style). The dashed grey line is a Huber regression fit. 
    }
    \label{fig:bpd_vs_xgb_c2st}
\end{figure}

\textbf{Exact and efficient conditional sampling:} 
As discussed in \cref{sec:pcs}, circuits can exactly condition on any arbitrary feature subset.
This enables us to evaluate the fidelity of conditionally generated data in a systematic way.
Note that this is in stark contrast to other DGMs such as diffusion which cannot marginalise exactly and can only retrieve conditional sampling with either heuristics involving training over conditioning masks \citep{song2020score} or more sophisticated MCMC schemes \citep{wedenig2025effective}.

\cref{fig:c2st_cond_sampling_main} reports \ctwost(XGB) scores when conditioning on different percentages of randomly selected features on the test set.
Specifically, $0\%$ of conditioning (left) corresponds to unconditional generation, while observing $100\%$ (right) of the dataset simply means copying it, hence maximum performance. 
We can see how \ours performance quickly increases and never drops (thanks to exact conditioning) and it is able to generate completions which almost perfectly fool the XGBoost classifier with already 50\% observations, hence displaying a high degree of realism.
Note how a simple baseline such as mean imputation (or mode imputation for categorical features) grows much slower.
\cref{fig:cond_sampling_full_size} shows the same trend for \wnmis.
We remark that we condition on a different set of features for each data row, this would make it infeasible to perform heuristic imputation through optimisation \citep{ho2022classifierfreediffusionguidance} with diffusion models.

\begin{figure}[!t]
    \centering
    \includegraphics[width=0.95\linewidth]{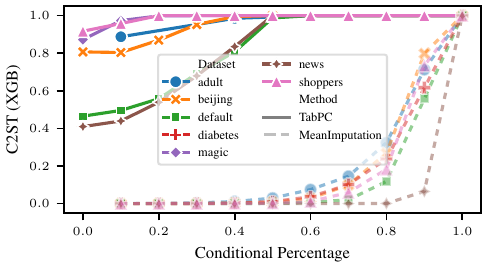}
    \caption{\textbf{\ours produces high fidelity conditional samples}, as shown by its ability to produce datasets with a high \ctwost (XGB) score when \textbf{seamlessly performing exact conditioning} on various percentages of the unseen test set. Solid lines correspond to results from datasets generated by \ours, dashed translucent lines correspond to mean imputation (or mode imputation for categorical features).}
    \label{fig:c2st_cond_sampling_main}
\end{figure}

{
\subsection{Q4: Deep Generative Models vs \tabpfn}

Given the recent promising results of tabular foundation models such as \tabpfn on discriminative tasks (e.g. as shown by \citealt{erickson2026tabarena}), this begs the question: how well can they be extended to generative tasks, and how do they compare to DGMs designed for this task? We focus on \tabpfn as one of the current most popular models, and because of its community extensions for unsupervised tasks.

 Notably, \tabpfn has been extended to provide functionality for synthetic data generation. It does this by sampling features autoregressively following some feature ordering. For robustness, the extension generates according to multiple permutations of the feature ordering, and then averages over these generations. See \cref{app:tabpfn_details} for specific implementation details.

Full results for \tabpfn can again be found in \cref{app:tables}. In particular, we focus here on its \shape, \ctwost (XGB), and \wnmis scores. \tabpfn's scores on \shape place it somewhat below the level of the diffusion and PC-based \sota in capturing univariate marginals, falling somewhere between \tabsyn and \stasy. For \wnmis, \tabpfn again offers performance somewhere between the current \sota and \stasy, suggesting that it is not quite as good at capturing pairwise dependencies as \sota models. In terms of \ctwost (XGB), \tabpfn offers varying performance. For example, on Magic, \tabpfn scores similarly to \ours (the top-ranking method for this dataset), whereas on News and Shoppers, the score is almost zero. We conclude that the quality of the joint distribution captured by these samples varies greatly per dataset. Also, \tabpfn has an average rank above \tabsyn and \ours for \betarecall, and generally scores well for this metric. This suggests that \tabpfn is able to generate a comparably diverse range of samples. And, while the utility metric MLE is not a main focus of this work, it also performs well under this; see \cref{app:mle}.

We also note that the dataset sampling time is significantly higher for \tabpfn than all other methods. 
See \cref{fig:c2st_train_n_sample} for comparison against \sota DGMs.
This all suggests that, while it is interesting that \tabpfn can be adapted to handle generative tasks, its performance in generating high-fidelity data is currently somewhat behind the level of bespoke models. 
All in all, \ours stands as the fastest alternative that delivers the best fidelity performance overall.

\begin{figure}[!t]
    \centering
    \includegraphics[width=0.95\linewidth]{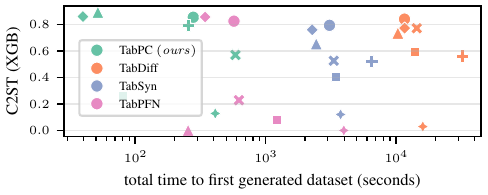}
    \caption{{\textbf{\ours models can train \textit{and} generate one dataset in less time than it takes to generate one dataset from \tabpfn,} while providing superior fidelity under our metrics. Though producing  similar fidelity data to \ours, \tabdiff and \tabsyn come at the cost of a longer total time. For \tabpfn we just plot the time to generate one dataset.}}
    \label{fig:c2st_train_n_sample}
\end{figure}

}

\section{Conclusion}
In this work, we \textit{\textbf{reset the current \sota for TDG}}, highlighting how  there is clearly still a long way to go in generating realistic synthetic data.
By highlighting some of the flaws in how two fidelity metrics (\trend and \ctwost) are currently used, we have determined that \textit{\textbf{our progress in fidelity is not as advanced as we may have believed}}. 
We then provided empirical and theoretical evidence to \textit{\textbf{support the widespread use of more robust metrics such as \wnmis and \ctwost(XGB)}}.
Moreover, we introduced the simplest baseline that sets \textit{\textbf{a new standard in trading-off computation for downstream data quality}}, \ours, which opens up future venues in TDG such as efficient conditional generation without retraining.
We also note that our \ours implementation can benefit from further advancements from the circuit literature such as introducing negative parameters \citep{loconte2024subtractive,loconte2025sum,loconte2026square}, and continuous latent variables \citep{correia2023continuous,gala2024pic,gala2024scaling}, which are expected to boost expressiveness further with a small computational overhead.
Lastly, we highlight how we should revisit in future work classical PC papers learning the structure of circuits \citep{molina2018mixed,vergari2019automatic,correia2020joints} or proposing different objectives for parameter learning \citep{watson2023arf}, noting that they could be extended to our overparameterized setting.

\begin{contributions} %
DS and AV conceived the original idea and they later discussed it with DP, AJ, HG, SS and LB. 
DS provided the first implementation of \ours, ran preliminary experiments for hyperparameter selection and proposed the \methodname{wNMI}-based evaluation metric.
DP proved \cref{thm:ff_fool_lr} with help from HG and AV,
provided the final implementation and hyperparameter selection used in the experiments, and is responsible for all tables and plots, with the exception of \cref{fig:simple_circuit_pipeline}, later improved by AJ.
AJ implemented conditional sampling, with help from DS.
HG proposed to compare  bpds and downstream performance.
DP led the writing of the paper, with help from DS, AJ, HG and AV.
AV supervised all phases of the project and provided feedback throughout.
\end{contributions}

\begin{acknowledgements} 
We would like to acknowledge the \href{https://april-tools.github.io/}{april lab} for its support and feedback, in particular Charles Bricout. Furthermore, we would like to thank Eleonora Giunchiglia for her advice and early feedback, and Chris Williams for an insightful discussion about earlier VAE-based approaches to TDG and association measures for the $2 \times 2$ contingency table. {Also, we would like to thank Gennaro Gala for providing a more efficient sampling implementation for PCs which helped to greatly improve the dataset sampling times.}
DS and SS were funded by the European Union through grant agreement 101218531 - SydAi.
HG was supported by the Royal Academy of Engineering under the Research Fellowship programme.
DP, AJ and AV are supported by the "UNREAL: Unified Reasoning Layer for Trustworthy ML" project (EP/Y023838/1) selected by the ERC and funded by UKRI EPSRC. 

\end{acknowledgements}

\bibliography{uai2026-template}

\newpage

\onecolumn

\title{A Sobering Look at Tabular Data Generation via Probabilistic Circuits\\(Supplementary Material)}
\maketitle

\appendix

\section{Proofs}\label{app:lr_c2st_proof}

In this section, we prove \cref{thm:ff_fool_lr}, firstly by proving the following lemma.

\begin{lemma}\label{lem:c2st_lr}
    Let $\mathcal{D}_p,\,\mathcal{D}_q$ with $|\mathcal{D}_p| = |\mathcal{D}_q| = N$ be two sets of datapoints drawn i.i.d. from distributions $p$ and $q$ respectively. If
    \begin{equation}\label{eq:lr_assumption}
        \frac{1}{N} \sum_{\bm{x}\in \mathcal{D}_p} \bm{x} = \frac{1}{N} \sum_{\bm{x}\in \mathcal{D}_q} \bm{x},
    \end{equation}
    i.e. the means of $\mathcal{D}_p$ and $\mathcal{D}_q$ are equal, then the LR classifier learned by maximum likelihood estimation to distinguish $\mathcal{D}_p$ and $\mathcal{D}_q$ is the random classifier.
\end{lemma}

\begin{proof}
    Let $y=1$ denote the case that a datapoint $\bm{x}$ comes from $p$, and $y=0$ denote the case that it comes from $q$. We form a dataset $\mathcal{D} = \{(\bm{x},0) \, | \, \bm{x}\in \mathcal{D}_q\} \cup \{(\bm{x},1) \, | \, \bm{x}\in \mathcal{D}_p\}$, on which to train the logistic regressor classifier $c_{\bm{w}}(y=1|\bm{x}) = \sigma(\bm{w}^{T} \bm{x})$, where $\sigma$ denotes the sigmoid function $\sigma(x) = 1 / (1+\exp(-x))$.

    When training with maximum likelihood estimation, we want the classifier to maximise the data log-likelihood. In this case, the data log-likelihood is given by
    \begin{align}
        \mathcal{L}(\mathcal{D}, \bm{w}) &= \sum_{(\bm{x},y) \in \mathcal{D}} \log c_{\bm{w}}(y\,|\,\bm{x}), \\
        &= \sum_{\bm{x}\in\mathcal{D}_p} \log c_{\bm{w}}(y=1 \,|\, \bm{x}) + \sum_{\bm{x}\in\mathcal{D}_q} \log c_{\bm{w}}(y=0 \,|\, \bm{x}), \\
        &= \sum_{\bm{x}\in\mathcal{D}_p} \log \sigma(\bm{w}^T\bm{x}) + \sum_{\bm{x}\in\mathcal{D}_q} \log (1 -  \sigma(\bm{w}^T\bm{x})).
    \end{align}
    Taking the gradient with respect to $\bm{w}$, we therefore have by linearity that
    \begin{align}
        \nabla_{\bm{w}} \mathcal{L}(\mathcal{D}, \bm{w}) &= \sum_{\bm{x}\in\mathcal{D}_p} \nabla_{\bm{w}} \left(\log \sigma(\bm{w}^T\bm{x}) \right) + \sum_{\bm{x}\in\mathcal{D}_q} \nabla_{\bm{w}} \left( \log (1 - \sigma(\bm{w}^T\bm{x})) \right) \label{eq:lr_step_linearity}, \\
        &= \sum_{\bm{x}\in\mathcal{D}_p} \left(1- \sigma(\bm{w}^T\bm{x}) \right)\bm{x} - \sum_{\bm{x}\in\mathcal{D}_q} \sigma(\bm{w}^T\bm{x})\bm{x} \label{eq:lr_step_chain_rule},
    \end{align}
    where \cref{eq:lr_step_chain_rule} follows from \cref{eq:lr_step_linearity} by the chain rule and due to the fact that $\sigma'(x) = \sigma(x)(1 - \sigma(x))$.
    
    Setting the weights to be zero, $\bm{w} = \bm{0}$, we see that
    \begin{align}
        \nabla_{\bm{w}} \mathcal{L}(\mathcal{D}, \bm{w}=\bm{0}) &= \sum_{\bm{x}\in\mathcal{D}_p} \left(1- \sigma(\bm{0}^T\bm{x}) \right)\bm{x} - \sum_{\bm{x}\in\mathcal{D}_q} \sigma(\bm{0}^T\bm{x})\bm{x}, \\
        &= \sum_{\bm{x}\in\mathcal{D}_p} \left(1- \frac{1}{2} \right)\bm{x} - \sum_{\bm{x}\in\mathcal{D}_q} \frac{1}{2}\bm{x}, \\
        &= \frac{1}{2} \left( \sum_{\bm{x}\in\mathcal{D}_p} \bm{x} - \sum_{\bm{x}\in\mathcal{D}_q} \bm{x}\right), \\
        &= \bm{0},
    \end{align}
    where the last step holds by \cref{eq:lr_assumption}. Therefore, $\bm{w}=\bm{0}$ is a stationary point of the log-likelihood. Now, since the log-likelihood function of logistic regression is concave, $\bm{w}=\bm{0}$ is hence the unique global maximum. The logistic regressor classifier with weights $\bm{0}$ assigns probability $1/2$ to all points, i.e. is the random classifier. The classifier learned by MLE is hence the random classifier.
\end{proof}

We use this lemma to prove \cref{thm:ff_fool_lr}, which we restate here for convenience.

\begin{theorem}

    Let $\mathcal{D}_R$ with $|\mathcal{D}_R| = n$ be a real dataset, and let $p$ denote a \FF model trained by MLE on this dataset.
    Then
    \begin{equation*}
       \lim_{n \to \infty} \mathbb{E}_{\mathcal{D}_S}[\ctst(LR, \mathcal{D}_R, \mathcal{D}_S)] = 1,
    \end{equation*}
    where $\mathcal{D}_S$ is an i.i.d. sample of $n$ items drawn from $p$.
\end{theorem}

\begin{proof}
    We learn FF models with MLE by matching the empirical sufficient statistics.
    Therefore, in the limit of $n \to \infty$, we have that $\frac{1}{n}\sum_{\mathcal{D}_S} \bm{x}$ (i.e. the sample mean) converges to the original mean $\frac{1}{n}\sum_{\mathcal{D}_R} \bm{x}$ for any dataset $\mathcal{D}_S$ of i.i.d. samples from $p$.
    By applying Lemma \ref{lem:c2st_lr}, in the limit as $n \to \infty$ we learn the random classifier with weights $\bm{w}=\bm{0}$. 
    Now, the expected AUROC of the random classifier is 0.5, and therefore the expected \ctwost score in the limit is $1$.
\end{proof}

\section{Probabilistic Circuits}\label{app:pcs_top}

\begin{figure}
    \centering
    \includegraphics[width=0.23\linewidth, page=8]{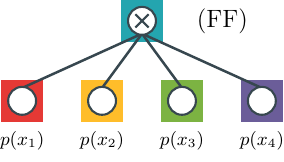}\hspace{50pt}
    \includegraphics[width=0.37\linewidth, page=9]{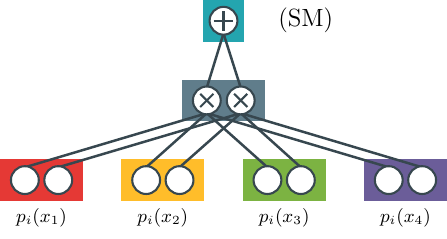}
    \caption{\textbf{The simple \ref{eq:ff_model} and \ref{eq:sm_model} models are special cases of probabilistic circuits.} The \ref{eq:ff_model} simply puts a product unit over independent input distributions. The \ref{eq:sm_model} mixes together with a sum unit $K$ separate \ref{eq:ff_model} models (here, $K=2$).} 
    \label{fig:circuit_diagrams_cirkit}
\end{figure}

\subsection{Chow-Liu Algorithm and Region Graph}\label{app:chow_liu}

The region graph (RG) tells us how to build our PC; in the interpretation as a deep, hierarchical mixture model, the RG instructs us at what depth to mix together different components. If the RG is built carefully, the PC we construct will be smooth and decomposable by design.

Each node in the RG is either a region $\mathcal{R}$, denoting a subset $\mathcal{R}\subseteq\boldsymbol{x}$, or a partition, which describes how a region is partitioned into other regions, $\mathcal{R} = \mathcal{R}_1 \cup \mathcal{R}_2$, with $\mathcal{R}_1 \cap \mathcal{R}_2 = \emptyset$. The root of the RG must necessarily be the `full' region $\mathcal{R}=\bm{x}$. So to construct an RG, we need a way of choosing how to partition regions at each level.

When using the Chow-Liu algorithm \citep{chow1968clts} to build a region graph, we first build the Chow-Liu tree based on the training data. The Chow-Liu tree is built by computing the pairwise mutual information between all pairs of features. This yields a fully connected weighted graph between all features. After that, it builds the maximum spanning tree by using Kruskal's algorithm \citep{kruskal1956shortest}. This yields the Chow-Liu tree (CLT).

Once we have the CLT, this can be compiled directly into a region graph. This process works as follows. We obtain the Jordan centre \citep{gadat2018barycenter, floyd1962shortest, wasserman1994social} of the CLT, and select this as the root. Then, we build the region graph by progressively merging scopes, starting from the leaf nodes and moving to the root. An example of this is shown above in \cref{fig:clt}. For more details, see \citealt{dang2021strudel, liu2021tractable, loconte2025what}.

\begin{figure}
    \centering
    \includegraphics[width=0.32\linewidth]{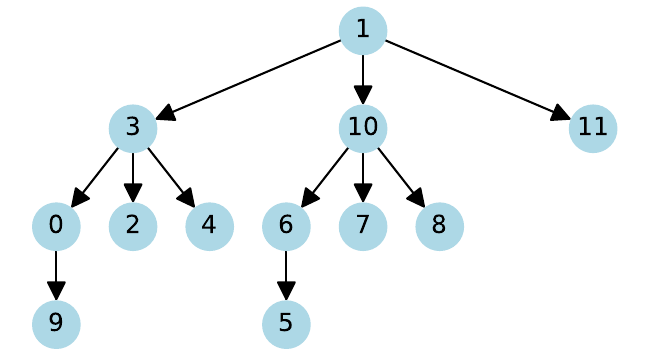}
    \includegraphics[width=0.32\linewidth]{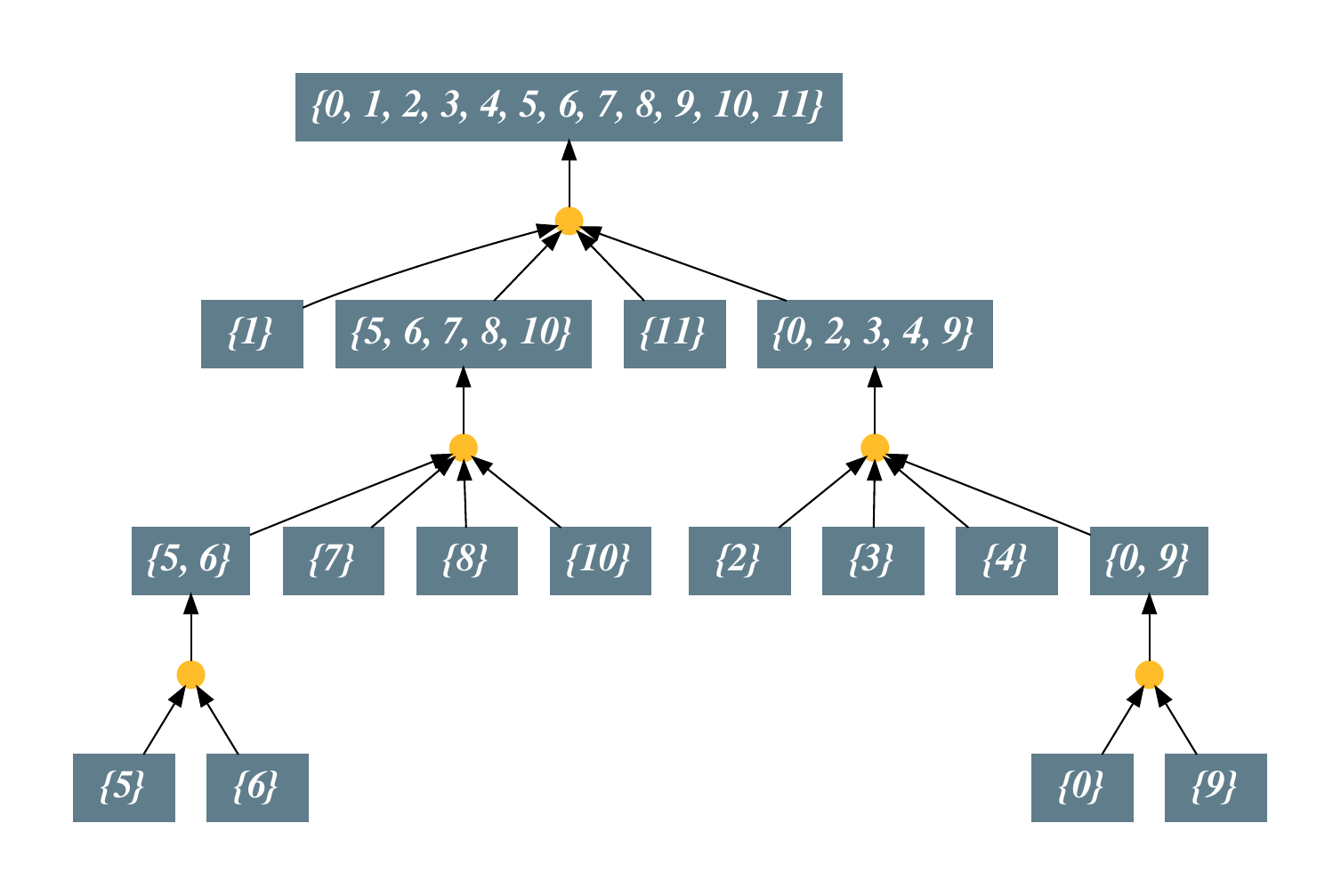}
    \includegraphics[width=0.32\linewidth]{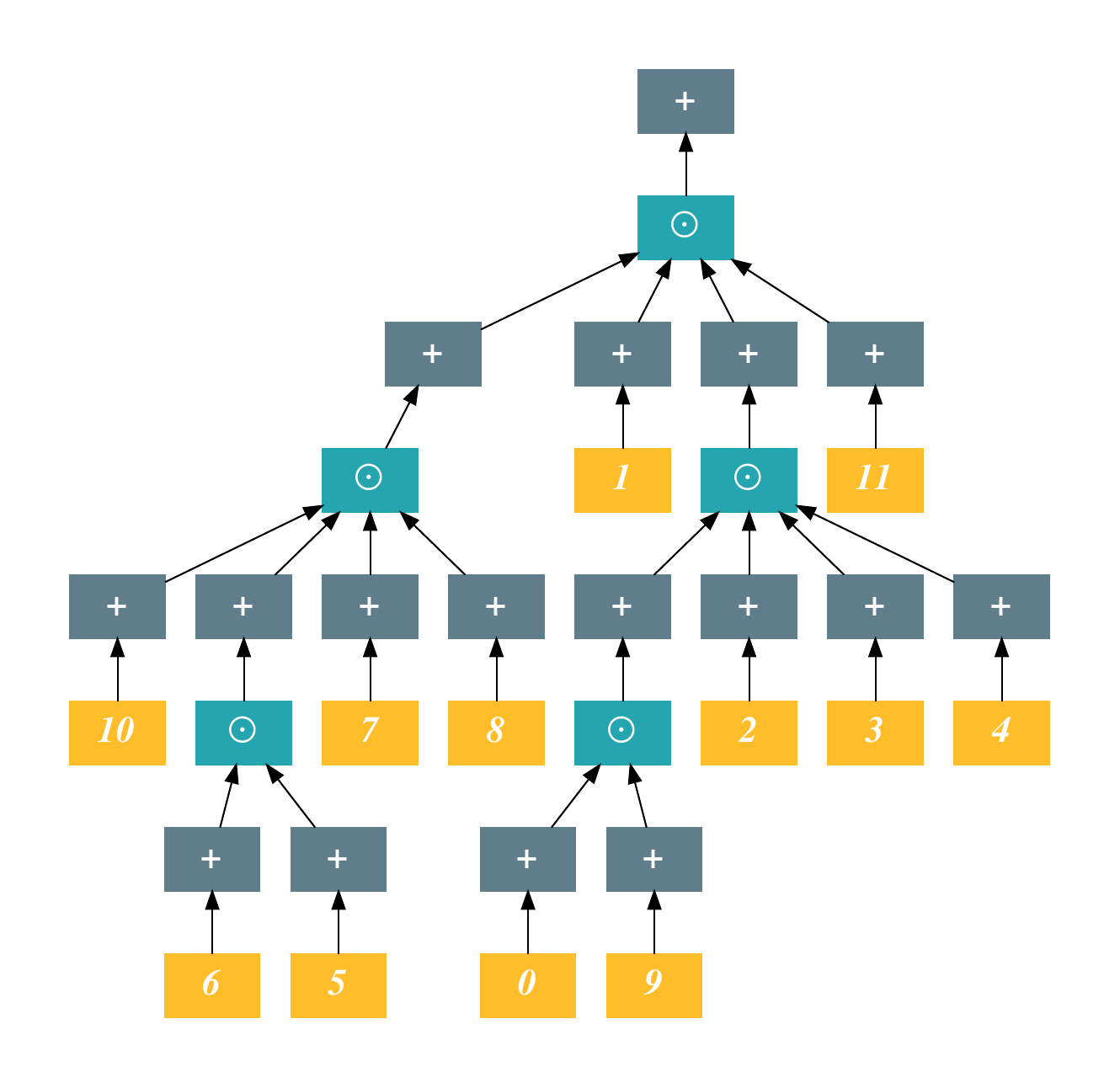}
    \caption{
    \textbf{Circuits can be built directly from a learned Chow-Liu tree}, as we see in the pipeline shown here for the Magic dataset. We first learn a Chow-Liu tree from the training data and root it at its barycentre \textit{(left)}, then compile this into a region graph \textit{(centre)}, then from this we build a circuit \textit{(right)}, where here we choose CP sum-product layers.
    }
    \label{fig:clt}
\end{figure}

\subsection{Sampling (conditionally) from smooth and decomposable PCs}\label{app:pc_sampling}

Sampling from smooth and decomposable PCs can be done via ancestral sampling, by using the latent variable interpretation of each sum unit. This corresponds to a backward traversal of its computational graph. Starting at the output unit, at each sum unit we sample one input branch proportionally to its weight, and then continue traversing the graph towards the inputs. At each product unit, we traverse the graph along all input branches. Once we reach an input unit, we sample from this as normal. Assuming smoothness and decomposability, we are guaranteed to end up in a set of input units whose scope is $\bm{x}$ and where only one input unit (or rather, input layer) is selected per variable.

We can additionally perform conditional sampling if the input units support tractable conditional sampling. Overall, conditional sampling can be performed by first performing a forward pass with the observed evidence, using this information to modify the sum weights, and then performing the same ancestral sampling as before using these new weights. During this forward pass, the input units either propagate their encoded function evaluated at the observed value (if their scope variable is observed) or $1$ (since we are marginalising out unobserved variables). Then, we use these modified weights and perform sampling as before. When we reach an input unit whose scope is observed, instead of sampling we take the observed value of its scope variable (assuming univariate scope).

This procedure is detailed in \cref{alg:pc-sampling}, which is adapted from \citealt{grivas2026fast}. Here, we additionally assume for simplicity that input units have univariate scope, but the algorithm can be extended to consider multivariate inputs.

\begin{algorithm}[!t]
    \caption{$\textsc{ConditionalSample}(c)$}\label{alg:pc-sampling}
    \textbf{Input:} A smooth, decomposable and normalised PC  $c$  encoding a joint distribution over $\bm{X}=\{X_1,\ldots,X_n\}$
    and observed values $\bm{X}_o = \bm{x}_o$ for $\bm{X}_o \subseteq \bm{X}$. \\
    \textbf{Output:} a sample $\bm{x}_m \sim c(\bm{X}_m | \bm{X}_o)$ for $\bm{X}_m = \bm{X} \setminus \bm{X}_o$.
    \begin{algorithmic}[1]
        \For{$c_n \in \mathsf{FeedforwardOrder(c)}$}
        \Comment{first compute $c(\bm{x}_o)$ and store intermediate results}
        \If{$c_n$ is an input unit over variable $X_i \notin \bm{X}_o$}
        \State $r_n \leftarrow 1$
        \ElsIf{$c_n$ is an input unit over variable $X_i \in \bm{X}_o$}
        \State $r_n \leftarrow c_n(x_{\phi(n)})$
        \Comment{compute probability of observation; $\phi(n)$ outputs $c_n$'s scope variable}
        \ElsIf{$c_n$ is a sum unit}
        \State $r_n \leftarrow \sum_{j\in\inscope(n)} \omega_j r_j $
        \Comment{$\inscope(n)$ denotes the input units to $c_n$}
        \ElsIf{$c_n$ is a product unit}
        \State $r_n \leftarrow \prod_{j\in\inscope(n)} r_j$
        \EndIf
        \EndFor
        \State $\bm{x} \leftarrow \mathsf{zeroes}(n)$
        \Comment{init empty sample}
        \State $c_n \leftarrow \mathsf{output}(c)$
        \State $\mathcal{N} \leftarrow \mathsf{queue}(\{c_n\})$
        \Comment{traverse the computational graph from outputs to inputs}
        \While{$\mathcal{N}$ not empty} 
        \State $c_n \leftarrow \mathsf{pop}(\mathcal{N})$
        \If{$c_n=\sum_{j=1}^{K}\omega_jc_j$ }
        \Comment{$c_n$ is a sum unit with $K$ inputs}
        \For{$j=1 \ldots K$}
        \State $\widetilde{\omega}_j \leftarrow \omega_j r_j / \sum_{i=1}^K \omega_i r_i$
        \Comment{modify sum weights by conditioning information}
        \EndFor
        \State $k\leftarrow \mathsf{sampleCategorical}(\widetilde{\omega}_1,\ldots, \widetilde{\omega}_K)$
        \Comment{sample from a categorical with $K$ states}
        \State $\mathcal{N} \leftarrow \mathsf{push}(\mathcal{N}, c_k)$
        \ElsIf {$c_n=\prod_{j=1}^{d} c_j$}
        \Comment{$c_n$ is a product unit with $d$ inputs}
        \For{$k=1\ldots d$}
        \State $\mathcal{N} \leftarrow \mathsf{push}(\mathcal{N}, c_k)$
        \Comment{visit all inputs of $c_n$}
        \EndFor
        \ElsIf{$c_n$ is an input unit over variable $X_i \notin \bm{X}_o$ with parameters $\phi_i$}
        \State $x_i\leftarrow \mathsf{sampleUnit}(\phi_i)$ \Comment{sample from the input distribution if unobserved}
        \ElsIf{$c_n$ is an input unit over variable $X_i \in \bm{X}_o$}
        \State $x_i \leftarrow x_{o,i}$
        \Comment{otherwise take observed value}
        \EndIf
        \EndWhile
    \State \textbf{return} $\bm{x}$ 
    \end{algorithmic}
\end{algorithm}

This algorithm relies on the following observations: for a smooth and decomposable PC, the conditional of a sum unit is a sum unit over conditional inputs, and the conditional of a product unit is a product unit over conditional inputs. Further details can be found in \url{https://github.com/smatmo/ESSAI24-PCs/blob/master/lecture01/lecture01.pdf} (slides 37--42).

\subsection{Additional details on tensorising and training PCs}\label{app:cp_layers}

Following the choice of region graph, degree of overparameterisation, and choice of sum-product layer, the final step in circuit construction pipeline is to \textit{fold} the PC. This means stacking layers with the same functional form so that they can be evaluated in parallel, yielding significant computational benefits \citep{loconte2025what}. Note that folding does not change expressivity --- it just enables more efficient computation. For \ours, we use the default folding algorithm provided in the \texttt{cirkit} package \citep{The_APRIL_Lab_cirkit_2024}.\footnote{\url{https://github.com/april-tools/cirkit}} All later experiments are also implemented using this package.

The training objective for PCs is to directly maximise the log-likelihood (LL) of the training data under the model, which can be tractably evaluated in a single forward pass (if the PC is smooth and decomposable). The LL can be optimised by stochastic gradient ascent (SGD), or by the more specific expectation maximisation (EM), which has been derived for PCs \citep{peharz_latent_2017, peharz_einsum_2020}. In our experiments, we train using the RAdam optimiser \citep{liu_variance_2021}, since EM is not currently implemented in \texttt{cirkit}.

\section{More Related Works}\label{app:related}

Before the recent spate of works on TDG mentioned in the main text, some earlier works focused in greater detail on using VAEs. These works were not necessarily framed under the modern banner of TDG, but typically tended to focus on missing data imputation and various tasks associated with this. Such works include \citealt{ivanov2018variational, ma2020vaem, collier2020vaes, ma2018eddi, mattei2019miwae}.

Another related line of research tackles imposing semantic constraints over the generated data \citep{stoian2024realistic,stoian2025beyond} and evaluating their violation by DGMs.
While \ours is not designed to satisfy constraints, works on using PCs and other tractable models for neuro-symbolic constraint satisfaction could be investigated in the future \citep{ahmed2022semantic,kurscheidt2025probabilistic}.

Some prior works have also focused in some detail on investigating suitable measures of association between random variables, in particular for the $2 \times 2$ contingency table \citep{williams2022suspicious, edwards1963measure, hasenclever_comparing_2016}. Since we consider a more general case, we focus in the main text on comparing \trend and \wnmis, but note that adapting such discussions to evaluating synthetic tabular data could be an interesting avenue for future work.

\section{Experimental Setup}\label{app:exps_top}

We release the code, and will release all final model checkpoints and generated datasets for reproducibility. The code is available at \url{https://github.com/april-tools/tabpc} and will be updated with the link to checkpoints and datasets.

\subsection{Dataset Details}\label{app:datasets}

\begin{table}[ht]
    \centering
    \caption{Dataset statistics. \# Num = number of numerical columns, \# Cat = number of categorical columns, \# Max Cat = max categories in any categorical column. \# Train and \# Test refer to the number of samples in the training and test splits respectively.}
    \label{tbl:exp-dataset}
    \begin{tabular}{lrrrrrrl}
        \toprule
        \textbf{Dataset} & \textbf{\# Rows} & \textbf{\# Num} & \textbf{\# Cat} & \textbf{\# Max Cat} & \textbf{\# Train} & \textbf{\# Test} & \textbf{Task} \\
        \midrule
        Adult \citep{adult}     & $48,842$  & $6$  & $9$  & $42$  & $32,561$ & $16,281$ & Classification \\
        Beijing \citep{beijing}   & $43,824$  & $7$  & $5$  & $31$  & $39,441$ & $4,383$  & Regression     \\
        Default \citep{default}   & $30,000$  & $14$ & $11$ & $11$  & $27,000$ & $3,000$  & Classification \\
        Diabetes \citep{diabetes}  & $101,766$ & $9$  & $27$ & $716$ & $81,412$ & $20,354$ & Classification \\
        Magic \citep{magic}     & $19,019$  & $10$ & $1$  & $2$   & $17,117$ & $1,902$  & Classification \\
        News \citep{news}      & $39,644$  & $46$ & $2$  & $7$   & $35,679$ & $3,965$  & Regression     \\
        Shoppers \citep{shoppers}  & $12,330$  & $10$ & $8$  & $20$  & $11,097$ & $1,233$  & Classification \\
        \bottomrule
    \end{tabular}
\end{table}

\subsubsection{Dataset Splits}

We follow the same dataset splitting as performed by \citealt{tabdiff}, who in turn follow the protocol of \citealt{tabsyn}. Namely, each dataset above is split into train and test sets, with 90\% of the data forming the training set, and 10\% forming the test set. From the training set, we then form a validation set comprising 10\% of this data. Random splits are performed based on the random seed $1234$.

\subsection{Dataset Preprocessing}\label{app:preprocessing}

For comparability, we follow the same pre-processing protocol as in \citealt{tabdiff} for removing missing data. Specifically, we replace missing numerical values with the column mean, and treat missing categorical values as a new category. However, we note that as tractable models, PCs would be able to seamlessly handle missing data during training by marginalising out the missing variable(s).

The following pre-processing steps are taken following this missing-value protocol.

For \ref{eq:ff_model}, we include two data pre-processing steps.
\begin{enumerate}
    \item \textbf{Inflated value handling:} Inflated values are specific values which are significantly oversampled in the original data; for example, these could be the maximum or minimum possible values. A common example of an inflated value is zero (e.g. the minimum value on some measuring device \citep{rozanec_dealing_2025}). In the case of \ref{eq:ff_model}, we simply discard inflated values during fitting.
    \item \textbf{Quantile normalisation:} this transforms samples from any distribution into samples from a standard Gaussian, and can later be inverted to obtain samples from the data distribution. More details below in \cref{app:quantnorm}.
\end{enumerate}

For \ours, we apply the same quantile normalization pre-processing. However, we modify the inflated value handling step.
\begin{enumerate}
    \item \textbf{Inflated value handling:} For \ours, we handle features with inflated values by creating a new category indicating that an inflated value is present, and then treating the value itself as though it were missing (i.e. by marginalising out that feature). Recall that, with a smooth and decomposable PC, we can efficiently and exactly marginalise out this feature.
    \item \textbf{Quantile normalisation:} as before.
\end{enumerate}

\subsubsection{Quantile Normalisation}\label{app:quantnorm}

We use the \texttt{QuantileTransformer} from \texttt{scikit-learn} \citep{scikit-learn}.\footnote{For further details, see \url{https://scikit-learn.org/stable/modules/generated/sklearn.preprocessing.QuantileTransformer.html}.} This works by estimating the cumulative distribution function (CDF) of the data, and then uses this to map the data to samples from a $\mathrm{Uniform}([0,1])$. Then, the quantile function of the normal distribution is used to transform the samples to those from a standard Gaussian. This transformation is applied independently to each feature. We go from samples in the normal space to samples in the data space by inverting the transformation.

\subsubsection{Preprocessing Ablation Studies}\label{app:preprocessing_ablation}

Here we study the effect of these different pre-processing steps on model performance. The results of these studies can be found in \cref{tab:ablation_c2st_xgboost}. We note here the importance of quantile normalisation in model performance, supporting its use by other methods \citep{tabsyn, tabdiff} as an important pre-processing step. We also provide training curves with and without the full pre-processing, which show that the stability of training is greatly improved with the pre-processing, in \cref{fig:ablation_losses}.

\input{tables_new_protocol/ablation}

\begin{figure}[!t]
    \centering
    \includegraphics[width=0.45\linewidth]{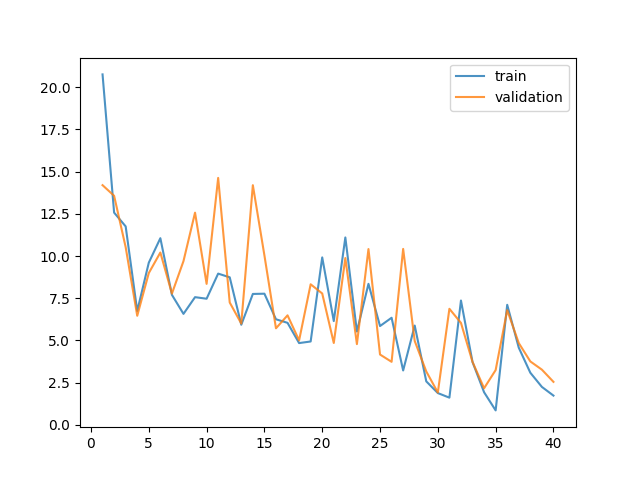}
    \includegraphics[width=0.45\linewidth]{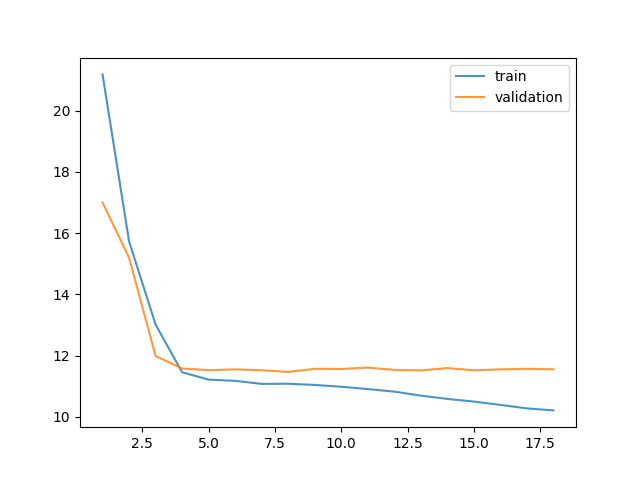}
    \caption{\textbf{Dataset pre-processing improves the stability of training.} The plots display training (blue) and validation (orange) negative log-likelihoods on the Adult dataset. Curves are shown for the base model with no pre-processing \textit{(left)}, and for the model with the full pre-processing steps of inflated value handling and quantile normalisation \textit{(right)}. The numerical values for the log-likelihood change when applying the transformation since this changes the density scale.}
    \label{fig:ablation_losses}
\end{figure}

\subsection{Metric Details}\label{app:metrics}

\subsubsection{\shape}\label{app:metrics_shape}

The \shape metric, implemented in \texttt{SDMetrics} \citep{sdmetrics}, captures how well synthetic data models the univariate marginal distributions of the real data. The way it is computed for each feature depends on that feature's type.

For numerical features, the \textit{Kolmogorov-Smirnov statistic} is used. This compares the two empirical cumulative distribution functions (CDFs) of the univariate marginals in both the real and synthetic data. Denoting these by $\widehat{F}_x$ and $\widehat{G}_x$ for feature $x$ respectively, the Kolmogorov-Smirnov statistic is defined for feature $x$ by \citep{massey_kolmogorov_1951}
\begin{equation}\label{eq:ks_statistic}
    D_{KS}\left(\widehat{F}_x, \widehat{G}_x\right) = \max_y \left|\widehat{F}_x(y) - \widehat{G}_x(y) \right| \in [0,1],
\end{equation}
i.e. the maximum of the set of distances between the CDFs. This difference is subtracted from $1$ to give a similarity score to the feature, so for numerical feature $x$, its corresponding score is given by
\begin{equation}\label{eq:numerical_shape_score}
    s_{\mathrm{numerical}}(x) = 1 - D_{KS}\left(\widehat{F}_x, \widehat{G}_x\right).
\end{equation}
This score is also in $[0,1]$, where $1$ is the maximum.

For categorical features, the \textit{total variation distance} (TVD) is used. This computes the difference between the empirical probabilities as follows:
\begin{equation}\label{eq:tvd_1d}
    \mathrm{TVD}(R_x, S_x) = \frac{1}{2}\sum_{\alpha\in x} |R_\alpha - S_\alpha| \in [0,1],
\end{equation}
where the sum over $\alpha$ ranges over the possible categories of $x$, and $R_\alpha$, $S_\alpha$ denote the empirical probabilities of $x=\alpha$ in the real and synthetic datasets respectively. This distance is again subtracted from $1$ to yield a similarity score:
\begin{equation}\label{eq:categorical_shape_score}
    s_{\mathrm{categorical}}(x) = 1 - \mathrm{TVD}(R_x, S_x)
\end{equation}
Again, this score is also in $[0,1]$, where $1$ is the maximum.

The overall \shape score is then given by the average scores of all columns (features) in the dataset.

\subsubsection{\alphaprecision and \betarecall}\label{app:metrics_ap_br}

\alphaprecision and \betarecall are proposed by \citealt{faithful}, extending the work by \citealt{sajjadi_assessing_2018} on assessing the precision and recall of generative models. These are sample-level metrics quantifying the sample's `fidelity' and `diversity'. In this case, they use fidelity to refer to sample quality (`realism' of samples), similarly to before, and diversity to refer to how well the samples cover the variability of the real data. Since these are sample-level metrics, they can be averaged over the entire synthetic dataset to provide an overall score.

They work by embedding the real and synthetic data into hyperspheres via a learned one-class neural network, where most samples are concentrated in the centre (so that the supports of the real and generated data are hyperspherical). Then, they use minimum volume sets covering a proportion ($\alpha$ or $\beta$) of the (real or synthetic) data, which in the hypersphere embedding space are given by hyperspheres of certain radii. For a given proportion $\alpha$, its $\alpha$-support describes the minimum volume subset of its support containing probability mass $\alpha$.

For \alphaprecision, we want to know whether the synthetic sample falls within the $\alpha$-support of the real data. This tells us whether a synthetic sample is typical (i.e. whether it could come from the real dataset). For \betarecall, we want to know whether the real samples fall within the $\beta$-support of the synthetic data. This tells us whether the synthetic data covers the full variability of real samples. To compute these, we train classifiers to predict whether a given sample is contained within the $\alpha$- or $\beta$-support.

For a full technical description of these metrics, see \citealt{faithful}.

\textbf{One important consideration is that the actual implementation for these metrics can sometimes differ in recent related works to the approach described above; this is described in further detail in \cref{app:ap_br_discrepancy}.}

\subsection{Baseline Details}\label{app:baseline_details}

For comparison, we use the codebase provided by \citealt{tabdiff} for \tabdiff,\footnote{\url{https://github.com/MinkaiXu/TabDiff}} and the codebase provided by \citealt{tabsyn} for all other mentioned methods.\footnote{\url{https://github.com/amazon-science/tabsyn}} We also use their provided default hyperparameters.

\textbf{Other Baselines:} We were unable to reproduce previous results for the methods GOGGLE \citep{goggle} and TabDDPM \citep{tabddpm} due to environment and training issues. In particular, TabDDPM has previously been criticised for such issues on some datasets like Diabetes \citep{tabdiff}, but we experienced issues across all datasets. Moreover, we were unable to compare against the newer flow-based method TabbyFlow \citep{tabbyflow}, since at the original time of writing there appeared to be errors in the sampling code released online.\footnote{E.g. \url{https://github.com/andresguzco/ef-vfm/issues/1}} {Since then, the repo has been updated, but we found further issues in trying to get the codebase to run.\footnote{See \url{https://github.com/andresguzco/ef-vfm/issues/4}.}}

\subsubsection{\tabpfn Further Details}\label{app:tabpfn_details}

\tabpfn is a transformer-based model pre-trained on a large amount of synthetic data generated by Structural Causal Models \citep{pearl_causality_2009}. It shifts the discriminative paradigm to an in-context learning (ICL) approach, and is a `Prior-data Fitted Network (PFN)' performing approximate Bayesian inference \citep{muller2022transformers}. When using \tabpfn for discriminative tasks, the dataset is passed to the model, embedded, and then the model outputs predictions directly in a single forward pass (see e.g. \citealt{hollmann2023tabpfn}). In particular, no training or fine-tuning is required.

For \tabpfn, we installed the version of \texttt{tabpfn-extensions} corresponding to \href{https://github.com/PriorLabs/tabpfn-extensions/tree/7f84343b59ffb03fdecd4fb6a8576d0eea09db57}{commit \texttt{7f84343} on GitHub}. We used the default version of \tabpfn, which is currently v3 \citep{grinsztajn2026tabpfn3technicalreport}. Additionally, we follow the package's \href{https://github.com/PriorLabs/tabpfn-extensions/blob/7f84343b59ffb03fdecd4fb6a8576d0eea09db57/src/tabpfn_extensions/unsupervised/unsupervised.py#L882}{default value of three feature ordering permutations} for our data generation experiments, following concurrent work \citep{tugnoli2026improvingtabpfnssyntheticdata}. While other approaches to using \tabpfn for TDG exist \citep{ma2023tabpfgen}, we choose to use \texttt{tabpfn-extensions} as the `most official' implementation.

When no features have yet been sampled, \tabpfn conditions on random Gaussian noise so as to approximate the first feature's marginal distribution. After that, features are autoregressively generated based on the current feature ordering permutation \citep{tugnoli2026improvingtabpfnssyntheticdata}.

The documentation of \tabpfn specifies that minimal pre-processing should be applied. Hence we apply only an ordinal encoding to categorical features. Moreover, we specify the categorical features to \tabpfn as those with fewer than 160 unique values (the maximum category count of \tabpfn v3). This is a reasonable heuristic given the number of samples our datasets have ($>1000$, see \cref{tbl:exp-dataset}).

Diabetes has a feature which exceeds this maximum number of categories, and so we were not able to generate data for this dataset. Moreover, Magic contained one failed run, and News contained three failed runs, so the standard deviations for these datasets are reported over four runs and two runs respectively (by failed, we mean a run which failed more than 3 times, which we set as our `patience'). The error in each case related to generating a NaN for one of the numerical features.

It is possible that this instability during generation is caused by some underlying bug in the implementation, as it appears that other users have previously run into a similar issue (\url{https://github.com/PriorLabs/tabpfn-extensions/issues/120}). Since this is package of community extensions, the documentation specifies that the functionality is less rigorously tested than the main \tabpfn library, so this remains a plausible explanation.

\subsection{\ours Further Details}\label{app:tabpc_implementation}

Hyperparameter configurations for \ours on each dataset can be found in \cref{tbl:pc_hyperparameters}.

\begin{table}[!t]
    \centering
    \caption{\ours hyperparameters (number of sum and input units, batch size, and learning rate) and number of shallow mixture components for each dataset. Values were found by evaluating the BPD of trained instances of \ours on the respective validation set.}
    \label{tbl:pc_hyperparameters}

    \begin{tabular}{
    l %
    S[table-format=4, table-alignment = center] %
    S[table-format=3, table-alignment = center] %
    S[table-format=1.2, table-alignment = center] %
    S[table-format=6, table-alignment = center] %
    }
        \toprule
        \multirow{2}{*}{\textbf{Dataset}} & \multicolumn{3}{c}{\textbf{\ours}} & \textbf{Shallow Mixture} \\
         & \textbf{\# Units} & \textbf{Batch Size} & \textbf{Learning Rate} & \textbf{\# Components ($K$)} \\ 
        \midrule
        Adult  &  4096   &  512   & 0.25 & 20000  \\
        Beijing &  4096  &   512  & 0.25 & 50000  \\
        Default &  512  &  512   & 0.10 & 20000  \\
        Diabetes &  2048  &  512  & 0.25 & 20000  \\
        Magic   &    2048 &   256 & 0.10 & 10000  \\
        News    &  1024   &  512  & 0.10 & 10000  \\
        Shoppers   & 2048   & 512 & 0.10 & 20000  \\ 
        \bottomrule
    \end{tabular}
\end{table}

When training \ours (alongside the hyperparameters mentioned above), we use a learning rate scheduler which decreases the learning rate by a factor of 0.85 if the validation log-likelihood has not decreased in the last epoch. Recall that we use the RAdam optimiser \citep{liu_variance_2021} for training. We also use the validation set (recall that this is a $10\%$ split of the training set) to perform early stopping, with a patience of 10 epochs. Columns with fewer than 50 unique values are treated as categorical, and otherwise treated as continuous. Quantised numerical columns are dequantised by adding uniform noise in the range $[-q/2, q/2]$, where $q$ is the quantisation step. After sampling, these are re-quantised to their original grid.

\subsection{Computational Resources}

The following hardware was used to train the models and evaluate the results:
\begin{itemize}
    \item \textbf{GPU:} NVIDIA RTX A$6000$ ($49$ GiB)
    \item \textbf{Processor:} AMD EPYC $7452$ $32$-Core Processor
    \item \textbf{Memory:} $512$ GiB
\end{itemize}

\section{Experimental Results}\label{app:full_results_top}

\subsection{Tables}\label{app:tables}

The uncertainties in each table represent the standard deviation of measurements over 5 model seeds. {For \tabpfn, these are over 5 random generations with the exception of Magic and News, where \tabpfn experienced crashed multiple times during these runs (see \cref{app:tabpfn_details}). Hence, these standard deviations are over four and two datasets respectively.}

In \cref{tab:mle_updated}, Beijing and News are associated with regression tasks, with trained models evaluated by root mean squared error (RMSE). The other datasets are associated with classification tasks, evaluated by classifier area under curve (AUC).

Where ranks are assigned, missing entries are given the worst possible rank. The missing entries for \great correspond to datasets for which it either failed to train (Diabetes) or generate (News). On some occasions, \great was able to generate datasets, but would generate categories which did not appear in the training data, thereby throwing an error when evaluating metrics. \stasy technically finished training on Diabetes, but obtains quite poor results. 

Entries containing $<0.0001$ mean that the five model seeds each achieved a score less than this precision threshold.

\input{tables_new_protocol/shape}

\input{tables_new_protocol/trend}

\input{tables_new_protocol/wnmis}

\input{tables_new_protocol/lr_detection}

\input{tables_new_protocol/xgb_detection}

\input{tables_new_protocol/alpha_precision_oc}

\input{tables_new_protocol/beta_recall_oc}

\input{tables_new_protocol/mle}

\input{tables_new_protocol/training_time}

\input{tables_new_protocol/sampling_time}

\input{tables_new_protocol/num_parameters}

\subsection{CDDs}\label{app:cdds}

In the generated CDDs, the value in the scale represents the average rank of the method (based on the means of the observations for each dataset). Methods are then connected by a solid line if they cannot be statistically distinguished from one another (i.e. are in the same \textit{clique}).

These CDDs are produced using the \texttt{aeon} package \citep{aeon24jmlr}. To compute the cliques, \texttt{aeon} uses a one-sided Wilcoxon sign rank test with the Holm correction.\footnote{Further details can be found at \url{https://www.aeon-toolkit.org/en/stable/api_reference/auto_generated/aeon.visualisation.plot_critical_difference.html}.}

It is important to note that this implementation requires observations for all datasets. Therefore, methods which cannot be trained or generate successfully, or whose metric computations fail across some datasets (such as \great and \stasy) will not appear in the corresponding CDD. {For this reason, since \tabpfn did not work for Diabetes, it does not appear in any CDDs.}

\begin{figure}[!t]
    \centering
    \includegraphics[width=0.45\linewidth]{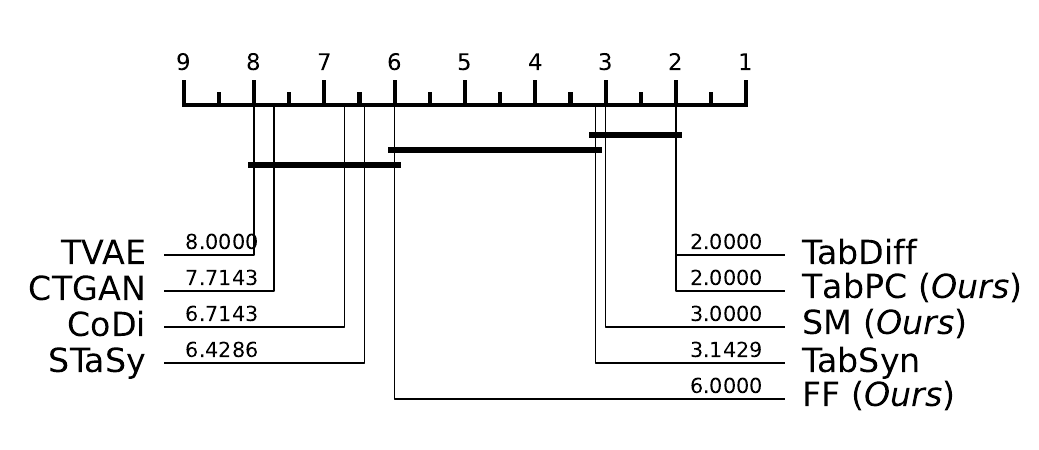}
    \includegraphics[width=0.45\linewidth]{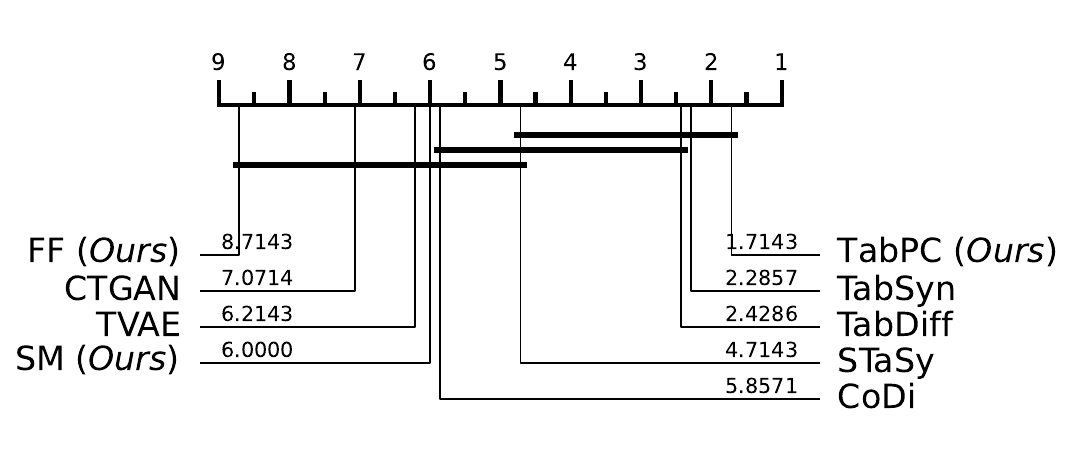}
    \caption{\textbf{\ours falls in the top clique of both CDDs}, and it is clear that \textbf{\wnmis (right) is able to successfully identify the trivial \FF model} where \trend (left) is not able to.}
    \label{fig:trend_wNMIS_cdd}
\end{figure}

\begin{figure}[!t]
    \centering
    \includegraphics[width=0.45\linewidth]{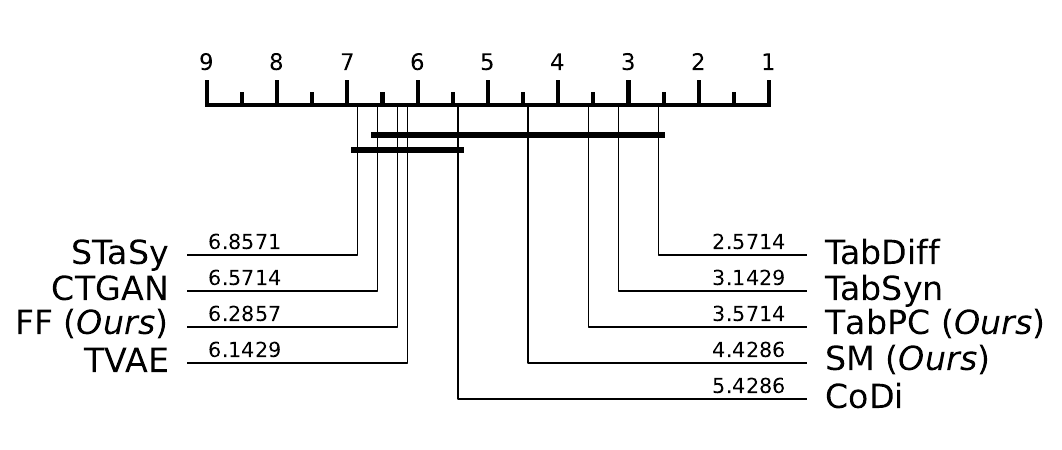}
    \includegraphics[width=0.45\linewidth]{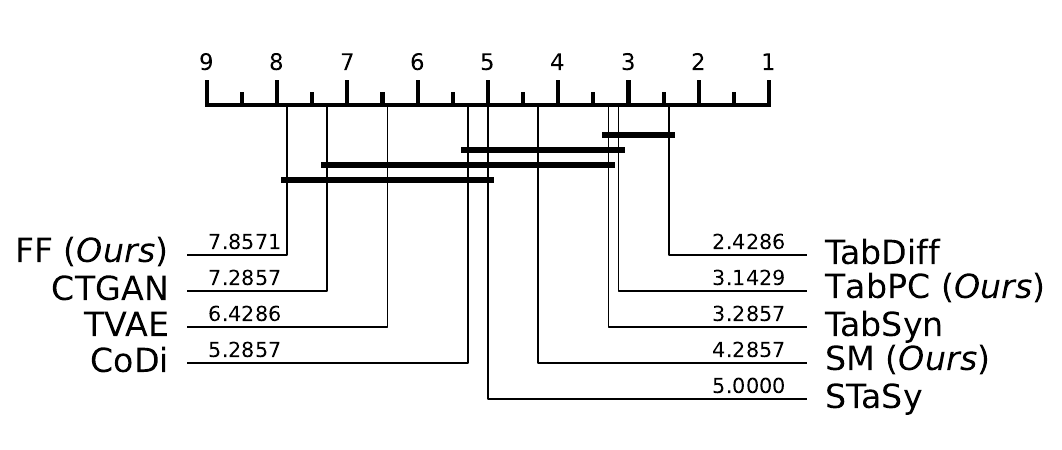}
    \caption{\textbf{\ours is in the top clique for both \alphaprecision (left) and \betarecall (right).} This suggests that its synthetic samples are competitive with the fidelity and diversity of samples generated by \tabdiff and \tabsyn.}
    \label{fig:ap_br_cdd}
\end{figure}

\begin{figure}[!t]
    \centering
    \includegraphics[width=0.45\linewidth]{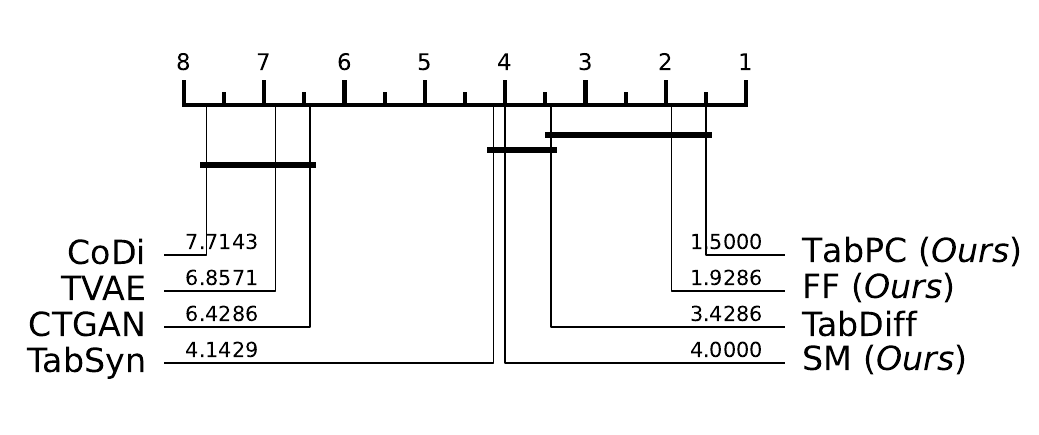}
    \includegraphics[width=0.45\linewidth]{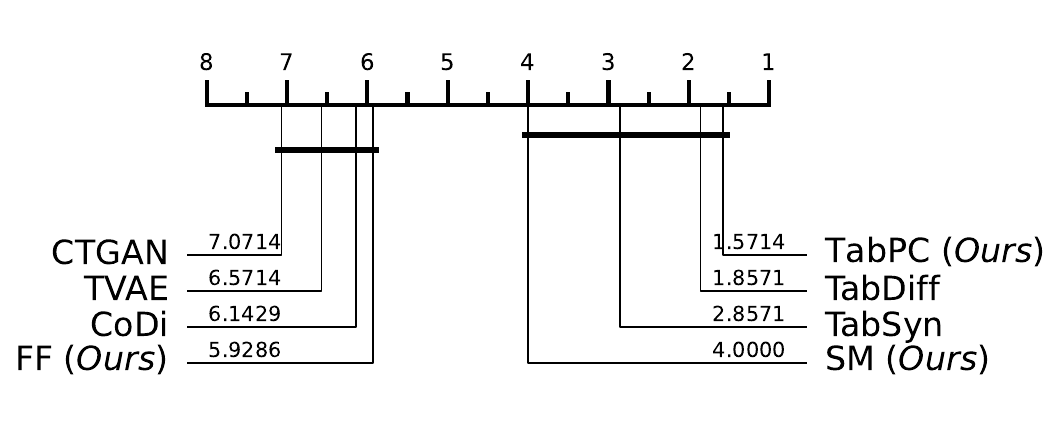}
    \caption{\textbf{\ctwost (XGB) (right) is able to clearly identify the trivial \FF model} where the metric using LR (left) is not. All lower performing models do poorly on \ctwost (XGB) and belong in the same clique, whereas the diffusion-based \sota, \ours and \SM fall within the top clique and hence have comparable performance.}
    \label{fig:lr_xgb_c2st_cdd}
\end{figure}

\begin{figure}[!t]
    \centering
    \includegraphics[width=0.45\linewidth]{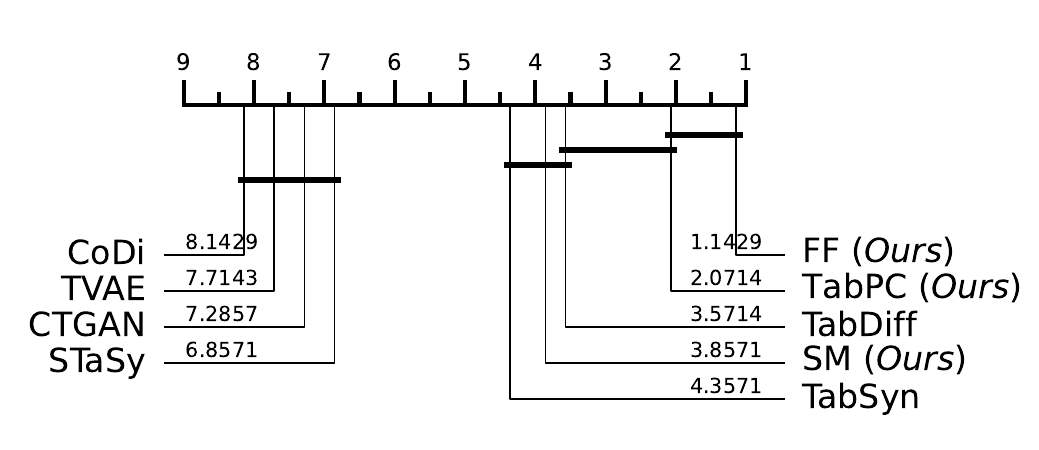}
    \caption{\textbf{\ours and \FF form the top clique for \shape,} showing that they accurately recover the univariate marginals. This is somewhat expected for \FF due to the pre-processing and the fact that we fit all marginals. The lower performing models again form a bottom clique (\codi, \tvae, \ctgan, \stasy).}
    \label{fig:shape_cdd}
\end{figure}

\subsection{Shallow Mixture Results}\label{app:sm_model_results}

\textbf{RQ)} How much does increasing PC complexity help? That is, how close to TDG \sota can we get with just a shallow mixture?

The \SM models have one to two orders of magnitude fewer parameters than \ours, and typically one order fewer than the diffusion-based \sota (except on Diabetes). For this reason, they represent a good intermediate level of complexity between \FF and \ours --- expressive enough to increase their performance significantly, with a more comparable parameter count to other models. Recall that, as tractable models, PCs require more parameters in order to compensate for having non-linearities only in their inputs.

It is clear from the results that the \SM models are a tier below \ours and the diffusion-based \sota. In the CDDs of \cref{app:cdds}, they often appear in fourth place after these models. In numerical terms, such as on \ctwost (XGB), while not generally competitive with \sota models they offer a significant step up from \FF, even beating \ours on Beijing. Where \FF posts results which are very close to zero, \SM make a step in the direction of \sota performance.

Their training times are also typically faster than for \ours, as to be expected, but for some datasets it can be lower (in particular, Default, where the \sota version of \ours is smaller than for other datasets and hence ends up being faster).

Given that simply mixing together many \FF distributions is a naive way of constructing a generative model, it is interesting that we can achieve reasonable results with the use of \SM.

\subsection{Utility Metric Results}
\label{app:mle}

MLE is designed to measure how useful synthetic data is for the purpose of training new machine learning models. This focuses on synthetic data's use as a proxy for real data. Here, we follow the protocol and implementation of \citealt{tabdiff}. 

To compute MLE for a given dataset, we train a discriminative machine learning model (in this case, an XGBoost regressor / classifier depending on the dataset task) on that dataset, and evaluate it on hold-out test data from the real distribution.

As mentioned, we follow the same protocol as \citealt{tabdiff}, and reuse their implementation. Specifically, we split the given dataset into training and validation sets with a ratio of $8:1$, learn the discriminative model on this training set, and use this validation set to select the optimal hyperparameters. Then the performance on the test data is evaluated, which is the value reported for MLE. It is important to compare this value against the value obtained by using the `real' training data,
which can be seen in the row `Real' in \cref{tab:mle_updated}. As seen in the table, training with real data yields the best results across all datasets, but some results for \sota models come close to matching this performance.

We observe that the data generated by \ours is reasonably competitive with \sota models. Downstream discriminative models trained on data from \ours sometimes outperform that of \sota diffusion models, and are sometimes outperformed. We note, as do previous works \citep{tabdiff}, that methods which typically generate data of low fidelity are sometimes able to achieve good scores, so this is not the most convincing metric. However, it provides some indication of how useful the generated data could be for downstream ML tasks.

{\tabpfn generates synthetic data which performs well under this metric, generally comparable or better to \sota models. This is perhaps not entirely surprising given its connection to predictive tasks, but is interesting in the context that its fidelity scores are generally lower than these models.}

\textbf{\textit{We strongly remark that one practice to avoid when reporting MLE is averaging over the numerical results, as is done in \citealt{tabbyflow}}}. Since the objective for each dataset is different (i.e. we want to minimise RMSE for regression tasks, and maximise AUROC for classification tasks), it does not make sense to average in this way as the two metrics are ``incomparable'' \citep{javaloy2025copa}.

\subsection{Privacy Metric Results}
\label{app:dcr}

It is complex to evaluate the privacy protection of synthetic data. Several recent works \citep{tabsyn, tabdiff, tabbyflow} rely on a simple distance-based metric called DCR (Distance to Closest Record).

As a privacy metric, distance to closest record (DCR) is motivated by the idea that if synthetic data and training data are too close, according to some distance metric, then there may be some information leakage.

The original DCR score, as used for example in \citealt{tabdiff, tabsyn}, represents the probability that a generated data sample is closer to the training data than to some hold-out test set. In this setup, values closer to $50\%$ are desirable, as they indicate more even distance between the synthetic data and both the training and hold-out sets.

One issue with implementing this is that it requires retraining models with a modified data split (50/50 train and test, so that the training and hold-out sets are balanced). As another issue,\textbf{\textit{ using DCR as a proxy metric for privacy has recently been strongly criticised and described as ``flawed by design''  in \citealt{yao_dcr_2025}, which recommends moving away from DCR}}. For this reason, we report these results here in \cref{tab:dcr_updated} for comparison only.

Here, we also compute an alternative formulation to previous works \citep{tabdiff, tabsyn}. This is because their original formulation requires splitting the training data in half and retraining a new model. (Results for the original DCR formulation are also reported for reference in \cref{tab:dcr_updated}{, excluding \tabpfn for which this metric was not computed due to the different data split.})

We instead use a formulation from \citealt{palacios_contrastive_2025} which compares the distributions of distances from the synthetic data and a test set to the training data. If there are significantly more synthetic entries closer to the training data than expected (i.e. compared to the test data), then this indicates a potential privacy risk.

We call this metric `quantile DCR', and denote it by DCR-002 and DCR-005 in the following tables, with the number corresponding to the quantile (either 2\% or 5\% respectively). We observe the corresponding results for quantile DCR in \cref{tab:dcr fraction<0.02 test quantile_updated} and \cref{tab:dcr fraction<0.05 test quantile_updated}. No privacy leaks for \ours are suggested according to the DCR-002 and DCR-005 metrics, which compare the sizes of the 2\% and 5\% quantiles of distances respectively. In this respect, \ours offers a similar risk of privacy leaks as TabDiff and TabSyn, which also show no indication of privacy leaks under these metrics.

\input{tables_new_protocol/dcr}
\input{tables_new_protocol/dcr002}
\input{tables_new_protocol/dcr005}

However, we stress again that, according to \citealt{yao_dcr_2025}, distance-based privacy metrics do not even give a strong indication of dataset privacy --- stronger conclusions would require a much more involved and rigorous investigation. Since our primary aim is to highlight the performance of \ours as a fast and high-fidelity TDG method, we leave such an investigation to future work.

\section{Likelihood vs C2ST (XGB) Experiment}\label{app:ll_expts}

\subsection{Bits-Per-Dimension (BPD)}\label{app:ll_bpd}

Bits-per-dimension offers a log-likelihood normalised by the number of features, making values more comparable across datasets. It represents the average number of bits required to encode a single dimension.

For a dataset $\mathcal{D} = \{\bm{x}_n\}_{n=1}^N$ consisting of datapoints $\bm{x} \in X_1 \times \ldots \times X_D$, we define the BPD by
\begin{equation}\label{eq:bpd}
    \mathrm{BPD}(\mathcal{D}) = \frac{\mathrm{NLL}(\mathcal{D}) / N}{\log 2 \cdot D},
\end{equation}
where $\mathrm{NLL}(\mathcal{D})$ denotes the negative log-likelihood of dataset $\mathcal{D}$ under the given model.

\subsection{Likelihood vs C2ST (XGB) Plots}\label{app:ll_plots}

Detailed plots for all datasets can be found in \cref{fig:ll_plots}. As mentioned previously, each marker in the plot corresponds to a different hyperparameter configuration. These hyperparameters are number of (sum and input) units, batch size, and learning rate. Number of units are selected to be powers of 2. The minimal grid for each dataset tests: (i) number of units $128$, $512$, $2048$; (ii) batch size $64$, $256$, $512$ (iii) learning rates $0.1$, $0.25$, and $0.5$. Where possible, we also test higher numbers of units. Early iterations tested lower learning rates ($0.01$, $0.001$) which proved to be much less effective, and so these do not appear for all datasets.

Regression lines are fit with Huber loss using the \texttt{HuberRegressor} implementation of \texttt{scikit-learn}. Further details can be found at \url{https://scikit-learn.org/stable/modules/generated/sklearn.linear_model.HuberRegressor.html}.

\begin{figure}
    \centering
    \includegraphics[width=0.31\linewidth]{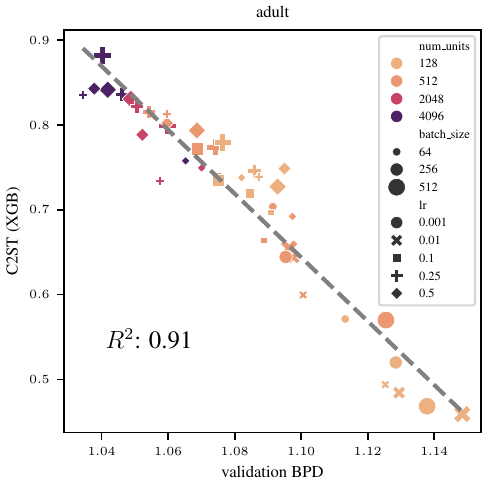}
    \includegraphics[width=0.31\linewidth]{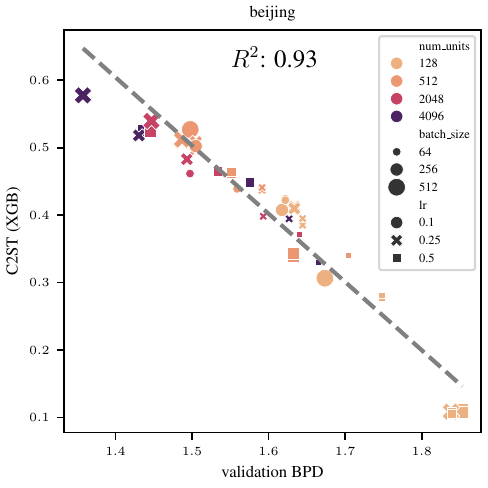}
    \includegraphics[width=0.31\linewidth]{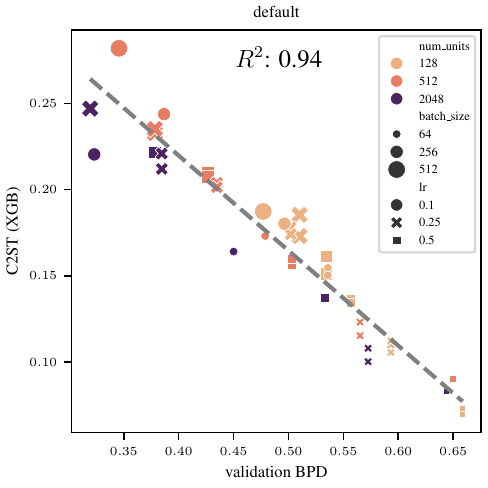}
    \includegraphics[width=0.31\linewidth]{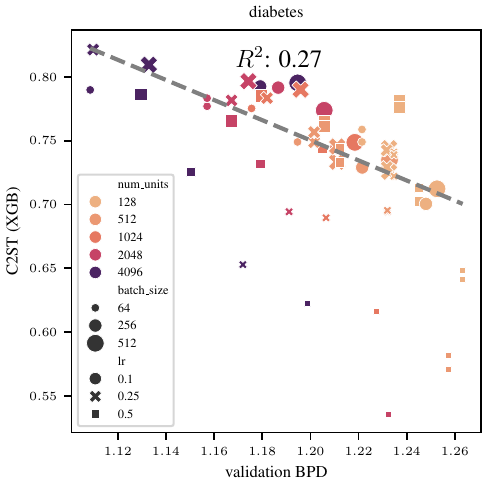}
    \includegraphics[width=0.31\linewidth]{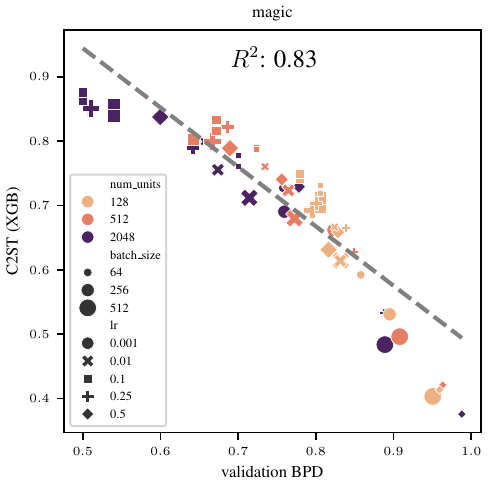}
    \includegraphics[width=0.31\linewidth]{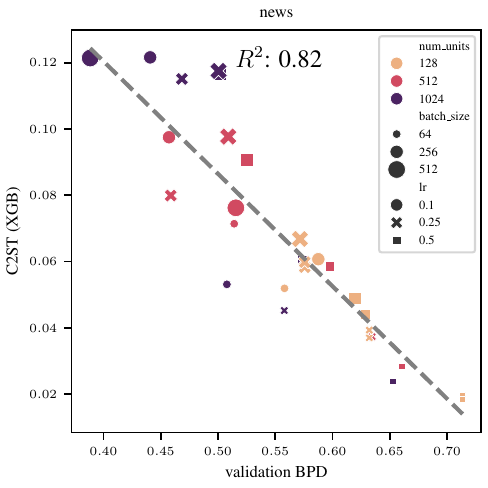}
    \includegraphics[width=0.31\linewidth]{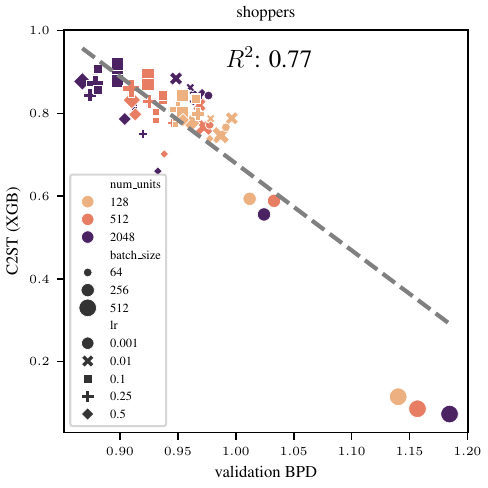}
    \caption{\textbf{Validation set bits-per-dimension (BPD) and downstream sample quality are correlated} across all datasets. In the image domain, PCs typically struggle to generate high-quality image samples while achieving low BPDs, and so we highlight this correlation here in the tabular domain.}
    \label{fig:ll_plots}
\end{figure}

\subsection{BPD Tables}

\cref{tab:bpd_values} displays BPD values across all datasets and splits.

\input{tables_new_protocol/bpd_values}

\subsection{Number of Units $K$ vs \ctwost (XGB) Plots}

To better highlight the relationship between the number of units used in \ours $K$ and the downstream \ctwost (XGB), here we provide additional plots studying this, which can be found in \cref{fig:k_vs_c2st_plots}. Note that these simply replot the information in \cref{fig:ll_plots} with $K$ as the $x$-axis, but it helps to visualise the effect.

\begin{figure}
    \centering
    \includegraphics[width=0.31\linewidth]{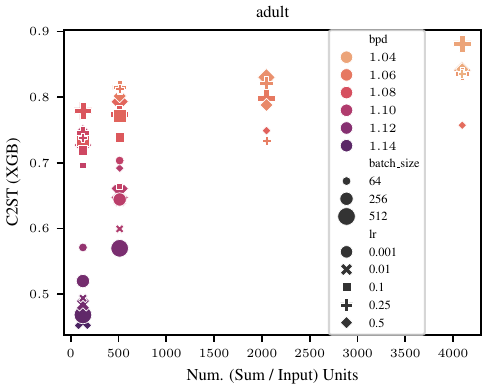}
    \includegraphics[width=0.31\linewidth]{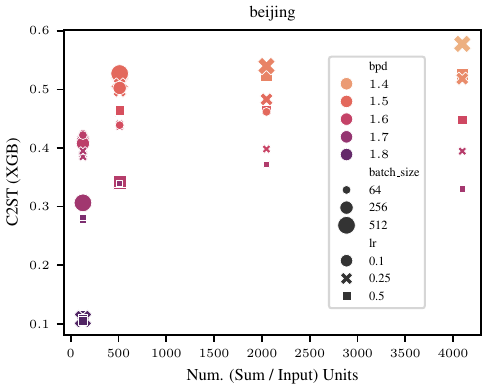}
    \includegraphics[width=0.31\linewidth]{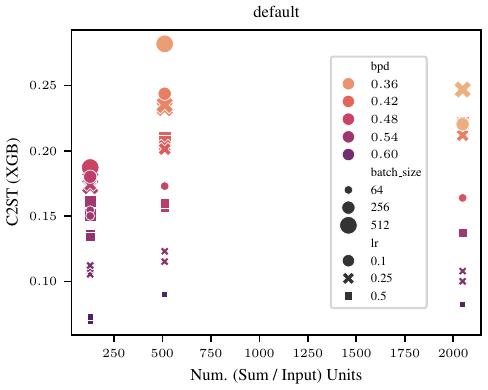}
    \includegraphics[width=0.31\linewidth]{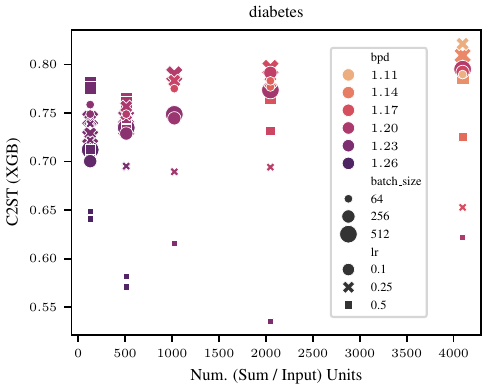}
    \includegraphics[width=0.31\linewidth]{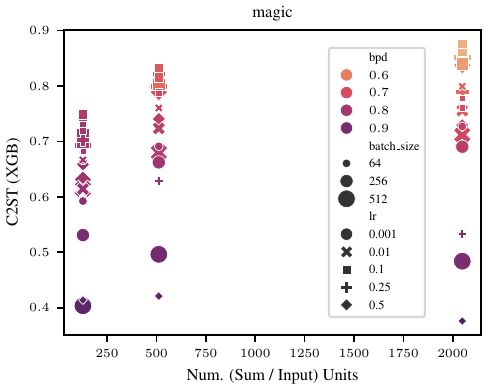}
    \includegraphics[width=0.31\linewidth]{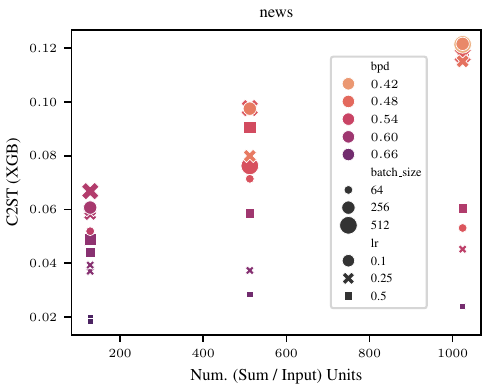}
    \includegraphics[width=0.31\linewidth]{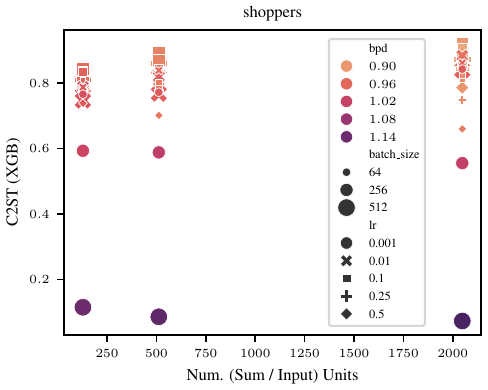}
    \caption{{\textbf{In general, increasing the number of sum and input units $K$ corresponds to increased \ctwost (XGB), as expected.} Quoted BPD values are for the validation set. This replots the data from \cref{fig:ll_plots} to focus on the effect of increasing $K$.}}
    \label{fig:k_vs_c2st_plots}
\end{figure}

\section{Conditional Sampling Experiment}

Full-sized plots for this experiment when conditioning on the test set are found in \cref{fig:cond_sampling_full_size}.

\begin{figure}
    \centering
    \includegraphics[width=0.45\linewidth]{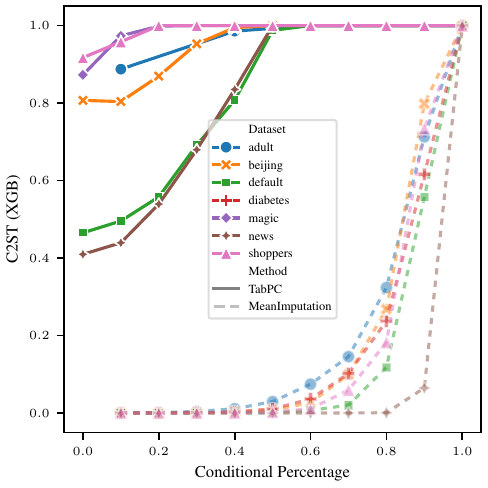}
    \includegraphics[width=0.45\linewidth]{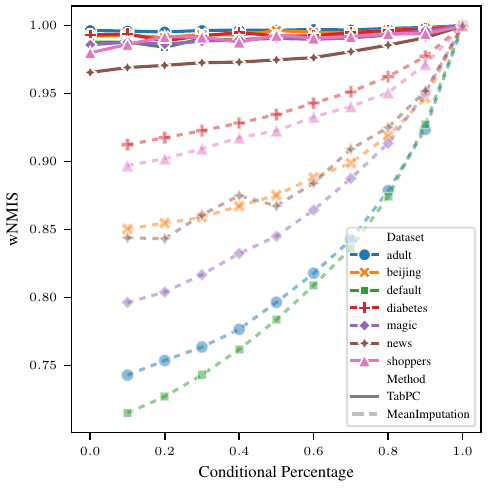}
    \caption{
    \textbf{\ours can realistically impute unobserved features when conditioned on any features}, and it can do this exactly and efficiently without retraining as a tractable probabilistic model. Realism is measured by \ctwost (XGB). The x-axis denotes the percentage of values on which we condition to generate the samples; 0\% conditioning corresponds to unconditional generation, up to 100\% conditioning which corresponds to copying the observed data. Dashed translucent lines denote the mean (or mode for categorical features) imputation baseline, which imputes missing values based on the observed column mean (or mode). By injecting increasing amounts of information about the training set via conditioning, we can generate highly realistic samples which can almost perfectly fool the XGBoost classifier.
    }
    \label{fig:cond_sampling_full_size}
\end{figure}

\section{Evaluation of Promoted Metrics}\label{app:metric_eval}

\subsection{\wnmis vs \trend}\label{app:wNMIS_vs_trend}

Critical difference diagrams (CDDs) for the two metrics can be seen in \cref{fig:trend_wNMIS_cdd}. We note first that, despite the \FF model being of average rank $6.0$ for \trend, it is in the same clique as the recent diffusion-based model \tabsyn, which tends to be one of the top performing models. The CDD for \wnmis instead shows that \FF is relegated to last place, as should be expected since it cannot model any correlations.

Aside from this, both metrics assign methods to reasonable cliques -- the most recent and performant methods, \tabdiff and \tabsyn form the top clique, and the more simplistic or earlier models, \ctgan, \tvae, \stasy and \codi, are in a lower clique. Moreover, the diffusion-based approaches (\stasy and \codi) can form a bridge between the lower clique and upper clique (meaning that they can be connected to methods from both), whereas the GAN or VAE-based methods do not, and remain firmly in the lower clique. This is reasonable since these GAN / VAE-methods are now the oldest.

Full observations for each metric can be seen in \cref{tab:density-Trend_updated} and \cref{tab:nmi_l1_weighted complement_updated}.

We also note that, at the top end of performance, we achieve scores which are very close to $1.0$. \wnmis therefore does not alleviate the issue of metric saturation.

However, \wnmis provides a more principled approach overall to quantifying how well synthetic data models correlations. This is shown by its ability to separate out the simplistic \FF baseline, whereas \trend assigns it almost perfect scores.

\subsection{\ctwost (XGB) vs \ctwost (LR)}\label{app:xgb_vs_lr}

Using XGBoost as our classifier for \ctwost, we can now clearly separate out the \FF model from the top-performing models. This can be seen in both \cref{fig:pc-global} and \cref{fig:lr_xgb_c2st_cdd}. In fact, in terms of \ctwost (LR), the simplistic \FF model even beats the sophisticated diffusion-based \tabdiff and \tabsyn, likely because it is designed to direct match the empirical means. This provides some empirical support to our theoretical result in \cref{app:lr_c2st_proof}, and underscores the insufficiency of \ctwost (LR) for the evaluation of synthetic data fidelity.

Relative to older baselines, \tabsyn and \tabdiff have made some progress in being able to fool XGBoost across most datasets. However, the News dataset, which is the most difficult to model, is still lacking across all methods. Moreover, no values are as close to $1.0$ as with LR, suggesting that there is still more to be done in improving synthetic data fidelity.

\subsection{\wnmis vs \nmis}

In this section, we investigate the advantages of \textit{weighting} the normalised mutual information similarities.

The unweighted \nmis is defined as

\begin{equation}\label{eq:NMIS}
    \mathrm{NMIS} \coloneqq \frac{1}{\#\{x_i,x_j \, : \, i<j\}} \cdot \sum_{\{x_i,x_j \, : \, i<j\}}\mathrm{NMIS}(x_i,x_j) \in [0,1],
\end{equation}
where $\mathrm{NMIS}(x_i,x_j)$ is identically defined as in \cref{eq:wNMIS}. Each pair therefore contributes equally to the score, as opposed to \cref{eq:wNMIS} which emphasises the contribution of pairs with high NMI in the real data and / or high NMI in the synthetic data. The motivation of this is that we want the score to better represent how well the model's generated data captures inter-pair dependencies \textit{where they exist}; in principle, by having many independent feature pairs, the score could otherwise be pulled up.

Results for unweighted \nmis can be found in \cref{tab:nmi l1 complement_updated}. One observation is that all values for \nmis are significantly higher than for \wnmis, making it more difficult to distinguish between model performances.
Moreover, the weighting provides a more principled approach to penalising errors \textit{only where dependencies exist}, as opposed to being affected by the presence of independent columns. It is for these reasons that we use \wnmis and not \nmis in the main text.

We note that the \texttt{SDMetrics} library has concurrently been updated to include a threshold in the \trend score computation (\url{https://github.com/sdv-dev/SDMetrics/releases/tag/v0.27.0}). That is, column pair scores are only included in the average if the association metric (either correlation or Cramer's V) is above this threshold. This helps to mitigate some of the observed issues we note in the main paper. However, we note several issues with this. 
\begin{itemize}
    \item These thresholds are rather arbitrarily set.\footnote{See \url{https://github.com/sdv-dev/SDMetrics/blob/7b5c44d4576ca4ff939fb221e32c8ae43d2946db/sdmetrics/reports/single_table/quality_report.py\#L16-L17}.} The weighting in \wnmis provides a more principled approach. Moreover, the weighting also results in the score decreasing when we have synthetic correlations but not real correlations between a pair of features. This is also an error in the synthetic data, but would be obfuscated by applying this thresholding technique (since then the error is not included in the trend average).
    \item Correlation and contingency similarities are not necessarily commensurable. Simply averaging over all column pair scores means that overall scores may vary based solely on the relative proportions of continuous, discrete, and mixed pairs. For \wnmis in contrast, by estimating the mutual information for feature pairs, we make scores commensurable, and so scores should be more comparable across different proportions of feature pair types.
\end{itemize}

\input{tables_new_protocol/nmis}

\subsection{Robustness of \wnmis to Discretisation}\label{app:wnmis_discretisation}

In this section we consider the sensitivity of \wnmis to the number of bins used to discretise the continuous feature(s). To do this, we compute the \wnmis score of \ours for each dataset using a different number of bins to discretise continuous features. Results for this ablation can be found in \cref{tab:wnmis_ablation_discretisation}.

\input{tables_new_protocol/wnmis_discretisation}

\section{$\alpha$-Precision and $\beta$-Recall Implementation Discrepancy}\label{app:ap_br_discrepancy}

After publication, we realised that our \alphaprecision and \betarecall results were being computed using a ``naive'' implementation internal to the evaluation package \href{https://github.com/vanderschaarlab/synthcity}{Synthcity}. This is the result of the evaluation code for these two metrics being passed down through several different codebases: first from the \href{https://github.com/vanderschaarlab/GOGGLE/blob/6fd497a2ff38a93a0f65e6906d51bf0362057957/src/goggle/GoggleModel.py/#L223}{repository of GOGGLE}, to that of \href{https://github.com/amazon-science/tabsyn/blob/cb5ac0f74ec36ee88e7a974a393dfbef50d42da7/eval/eval_quality.py/#L143}{\tabsyn}, and finally into \href{https://github.com/MinkaiXu/TabDiff/blob/5ecdb3356261aea72716cc9a779f31d7ad083bf4/eval/eval_quality.py/#L100}{\tabdiff}, from which we inherited our repository's structure and \href{https://github.com/april-tools/tabpc/blob/3f926684885a2a7c511cfbfe1aae201aa0c880c8/scripts/alpha_precision_beta_recall.py/#L146}{this particular piece of evaluation code}.

Based on the Synthcity code, it appears that the main difference between this \href{https://github.com/vanderschaarlab/synthcity/blob/23f322fe381326ed01c41b13d469a06e38cce545/src/synthcity/metrics/eval_statistical.py#L738-L750}{\textit{``naive''} implementation} and the \href{https://github.com/vanderschaarlab/synthcity/blob/23f322fe381326ed01c41b13d469a06e38cce545/src/synthcity/metrics/eval_statistical.py#L718-L736}{\textit{``original''} implementation} described by \citealt{faithful} (which is the approach outlined in \cref{app:metrics_ap_br}) is that data-points are not embedded in a hypersphere in the way described there; instead, they are normalised via a min-max scaling to the interval $[0,1]$. The comparisons between sets used to compute \alphaprecision and \betarecall are then made in this normalised space, instead of the hypersphere.

We are unsure why this implementation was used originally in GOGGLE / \tabsyn, and it seems that \href{https://github.com/amazon-science/tabsyn/issues/27}{other people have also noticed this discrepancy in the past}, without the intention behind this decision being made clear. It appears as though this issue of using the ``naive'' implementation has been replicated across multiple different works without being spotted (in particular, at least for the repositories linked above which are the most comparable recent works included as baselines).

We have therefore updated all figures and tables in the main paper to instead use the numbers computed under this ``original'' implementation, as opposed to the ``naive'' implementation. Importantly, this means that our numbers for these metrics are no longer backwards compatible with works which report numbers computed with the ``naive'' implementation, such as those mentioned above. For compatibility, we retain our previously computed metrics under new tables, \cref{tab:alpha precision_updated} and \cref{tab:beta recall_updated}. We also present their corresponding CDDs in \cref{fig:ap_br_naive_cdd}.

We note in particular that the previous version of \cref{fig:pc-global} displayed a very similar trend to the current version, with the exception that the \alphaprecision scores for the News dataset are now significantly lower (the same also holds for the \betarecall scores for News).

\input{tables_new_protocol/alpha_precision}

\input{tables_new_protocol/beta_recall}

\begin{figure}[!t]
    \centering
    \includegraphics[width=0.45\linewidth]{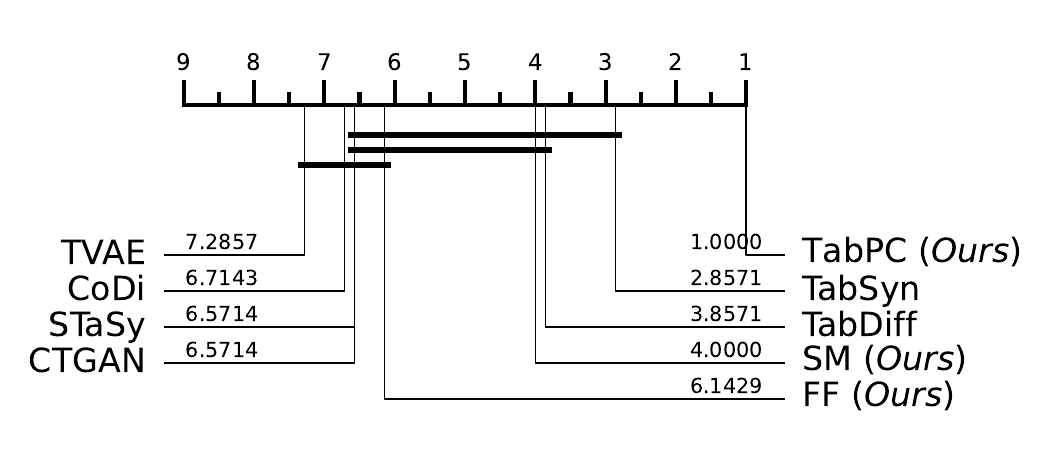}
    \includegraphics[width=0.45\linewidth]{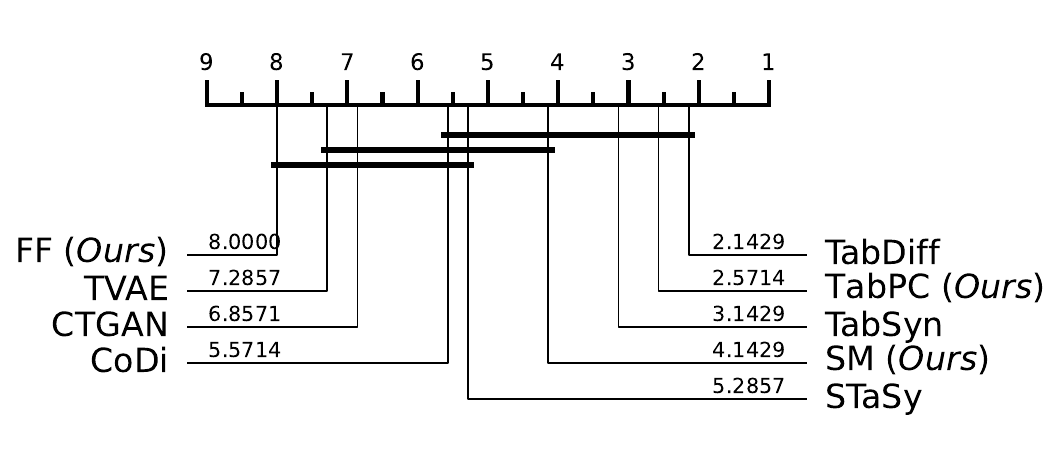}
    \caption{We present the corresponding CDDs for the ``naive'' implementation of \alphaprecision and \betarecall for backwards compatibility with prior works. \textbf{\ours is statistically higher performing than all other methods for \alphaprecision (naive; left)}, and falls in the top clique for \betarecall (naive; right).}
    \label{fig:ap_br_naive_cdd}
\end{figure}

\end{document}

%% file: tables_new_protocol/ablation.tex
\begin{table}
\centering
\caption{Ablation results for pre-processing components. Base denotes no pre-processing, IV denotes inflated value handling, QN denotes quantile normalisation, and IV + QN denotes having both of these components. We observe in particular that quantile normalisation is an important component of model performance over the base model---all datasets except Diabetes show a benefit when enabling quantile normalisation.}
\label{tab:ablation_c2st_xgboost}
\begin{tabular}{llS[table-format=1.4, table-alignment = right]}
\toprule
\textbf{Dataset} & \textbf{Pre-processing} & \textbf{C2ST (XGB)} \\
\midrule
Adult & Base & 0.0080 \\
Adult & IV & 0.6239 \\
Adult & QN & 0.7544 \\
Adult & IV + QN & \textbf{0.8910} \\
\midrule
Beijing & Base & 0.0068 \\
Beijing & IV & 0.0072 \\
Beijing & QN & \textbf{0.6088} \\
Beijing & IV + QN & 0.5744 \\
\midrule
Default & Base & 0.0103 \\
Default & IV & 0.0200 \\
Default & QN & 0.2719 \\
Default & IV + QN & \textbf{0.2799} \\
\midrule
Diabetes & Base & \textbf{0.7968} \\
Diabetes & IV & 0.7915 \\
Diabetes & QN & 0.7884 \\
Diabetes & IV + QN & 0.7889 \\
\midrule
Magic & Base & 0.7721 \\
Magic & IV & 0.7743 \\
Magic & QN & 0.8585 \\
Magic & IV + QN & \textbf{0.8627} \\
\midrule
News & Base & 0.0000 \\
News & IV & 0.0000 \\
News & QN & \textbf{0.1588} \\
News & IV + QN & 0.1132 \\
\midrule
Shoppers & Base & 0.0022 \\
Shoppers & IV & 0.5822 \\
Shoppers & QN & 0.8194 \\
Shoppers & IV + QN & \textbf{0.8866} \\
\bottomrule
\end{tabular}
\end{table}

%% file: tables_new_protocol/shape.tex
\begin{table}[htbp]
\centering
\caption{\textbf{\ours reports strong results on the \shape metric, exceeding the performance of \sota diffusion-based models in six out of the seven datasets.} We note that the \FF model in fact obtains higher scores even than \ours in all except one dataset, likely because matching the marginals is essentially how it is trained.}
\label{tab:density-Shape_updated}
\resizebox{\textwidth}{!}{\begin{tabular}{
l
r@{}l@{}r
r@{}l@{}r
r@{}l@{}r
r@{}l@{}r
r@{}l@{}r
r@{}l@{}r
r@{}l@{}r
r
}
    \toprule[0.8pt]
     \textbf{Method} & \multicolumn{3}{c}{\textbf{Adult}}
 & \multicolumn{3}{c}{\textbf{Beijing}}
 & \multicolumn{3}{c}{\textbf{Default}}
 & \multicolumn{3}{c}{\textbf{Diabetes}}
 & \multicolumn{3}{c}{\textbf{Magic}}
 & \multicolumn{3}{c}{\textbf{News}}
 & \multicolumn{3}{c}{\textbf{Shoppers}}
 & \textbf{Avg. Rank} \\
    \midrule 
    CTGAN & 0.8087 & \footnotesize{$\pm 0.0191$} & ~~(8)
 & 0.7999 & \footnotesize{$\pm 0.0221$} & ~~(10)
 & 0.8308 & \footnotesize{$\pm 0.0111$} & ~~(9)
 & 0.8957 & \footnotesize{$\pm 0.0042$} & ~~(6)
 & 0.9443 & \footnotesize{$\pm 0.0065$} & ~~(8)
 & 0.8622 & \footnotesize{$\pm 0.0036$} & ~~(8)
 & 0.7406 & \footnotesize{$\pm 0.0174$} & ~~(10)
 & 8.43 \\
    TVAE & 0.7619 & \footnotesize{$\pm 0.0135$} & ~~(10)
 & 0.7216 & \footnotesize{$\pm 0.0196$} & ~~(11)
 & 0.9127 & \footnotesize{$\pm 0.0023$} & ~~(8)
 & 0.7618 & \footnotesize{$\pm 0.0471$} & ~~(8)
 & 0.9540 & \footnotesize{$\pm 0.0067$} & ~~(7)
 & 0.8169 & \footnotesize{$\pm 0.0083$} & ~~(9)
 & 0.7508 & \footnotesize{$\pm 0.0187$} & ~~(9)
 & 8.86 \\
    GReaT & 0.4250 & \footnotesize{$\pm 0.0002$} & ~~(11)
 & 0.9339 & \footnotesize{$\pm 0.0010$} & ~~(7)
 & 0.8047 & \footnotesize{$\pm 0.0015$} & ~~(10)
 & \multicolumn{2}{c}{*} & ~~(10) & 0.8498 & \footnotesize{$\pm 0.0009$} & ~~(11)
 & \multicolumn{2}{c}{*} & ~~(11) & 0.8578 & \footnotesize{$\pm 0.0006$} & ~~(8)
 & 9.71 \\
    STaSy & 0.8915 & \footnotesize{$\pm 0.0172$} & ~~(7)
 & 0.8916 & \footnotesize{$\pm 0.0204$} & ~~(8)
 & 0.9285 & \footnotesize{$\pm 0.0172$} & ~~(7)
 & 0.3712 & \footnotesize{$\pm{0.0000}$} & ~~(9) & 0.8817 & \footnotesize{$\pm 0.0200$} & ~~(10)
 & 0.9016 & \footnotesize{$\pm 0.0233$} & ~~(7)
 & 0.8651 & \footnotesize{$\pm 0.0291$} & ~~(7)
 & 7.86 \\
    CoDi & 0.7662 & \footnotesize{$\pm 0.0245$} & ~~(9)
 & 0.8115 & \footnotesize{$\pm 0.0404$} & ~~(9)
 & 0.7790 & \footnotesize{$\pm 0.0300$} & ~~(11)
 & 0.7868 & \footnotesize{$\pm{0.0000}$} & ~~(7) & 0.9050 & \footnotesize{$\pm 0.0088$} & ~~(9)
 & 0.7173 & \footnotesize{$\pm 0.0146$} & ~~(10)
 & 0.6671 & \footnotesize{$\pm 0.0138$} & ~~(11)
 & 9.43 \\
    TabSyn & 0.9921 & \footnotesize{$\pm 0.0010$} & ~~(5)
 & 0.9739 & \footnotesize{$\pm 0.0126$} & ~~(6)
 & 0.9868 & \footnotesize{$\pm 0.0043$} & ~~(5)
 & 0.9822 & \footnotesize{$\pm 0.0008$} & ~~(5)
 & 0.9902 & \footnotesize{$\pm 0.0019$} & ~~(3)
 & 0.9766 & \footnotesize{$\pm 0.0122$} & ~~(4)
 & 0.9852 & \footnotesize{$\pm 0.0018$} & ~~(3)
 & 4.43 \\
    TabDiff & 0.9932 & \footnotesize{$\pm 0.0008$} & ~~(3)
 & 0.9895 & \footnotesize{$\pm 0.0004$} & ~~(3)
 & 0.9884 & \footnotesize{$\pm 0.0016$} & ~~(4)
 & 0.9855 & \footnotesize{$\pm 0.0070$} & ~~(4)
 & 0.9920 & \footnotesize{$\pm 0.0008$} & ~~(2)
 & 0.9725 & \footnotesize{$\pm 0.0047$} & ~~(5)
 & 0.9851 & \footnotesize{$\pm 0.0028$} & ~~(4)
 & 3.57 \\
    TabPFN & 0.9734 & \footnotesize{$\pm 0.0006$} & ~~(6)
 & 0.9869 & \footnotesize{$\pm 0.0004$} & ~~(4)
 & 0.9369 & \footnotesize{$\pm 0.0005$} & ~~(6)
 & \multicolumn{2}{c}{*} & ~~(10) & 0.9819 & \footnotesize{$\pm 0.0014$} & ~~(6)
 & 0.9224 & \footnotesize{$\pm 0.0001$} & ~~(6)
 & 0.8846 & \footnotesize{$\pm 0.0009$} & ~~(6)
 & 6.29 \\
    \midrule
    FF & 0.9961 & \footnotesize{$\pm 0.0006$} & ~~(1)
 & 0.9949 & \footnotesize{$\pm 0.0003$} & ~~(1)
 & 0.9954 & \footnotesize{$\pm 0.0003$} & ~~(1)
 & 0.9966 & \footnotesize{$\pm 0.0001$} & ~~(1)
 & 0.9938 & \footnotesize{$\pm 0.0007$} & ~~(1)
 & 0.9907 & \footnotesize{$\pm 0.0001$} & ~~(2)
 & 0.9934 & \footnotesize{$\pm 0.0007$} & ~~(1)
 & \textbf{1.14} \\
    SM & 0.9924 & \footnotesize{$\pm 0.0008$} & ~~(4)
 & 0.9811 & \footnotesize{$\pm 0.0046$} & ~~(5)
 & 0.9907 & \footnotesize{$\pm 0.0006$} & ~~(3)
 & 0.9940 & \footnotesize{$\pm 0.0013$} & ~~(3)
 & 0.9840 & \footnotesize{$\pm 0.0012$} & ~~(5)
 & 0.9880 & \footnotesize{$\pm 0.0012$} & ~~(3)
 & 0.9712 & \footnotesize{$\pm 0.0105$} & ~~(5)
 & 4.00 \\
    TabPC & 0.9942 & \footnotesize{$\pm 0.0007$} & ~~(2)
 & 0.9943 & \footnotesize{$\pm 0.0005$} & ~~(2)
 & 0.9946 & \footnotesize{$\pm 0.0004$} & ~~(2)
 & 0.9942 & \footnotesize{$\pm 0.0004$} & ~~(2)
 & 0.9902 & \footnotesize{$\pm 0.0013$} & ~~(3)
 & 0.9936 & \footnotesize{$\pm 0.0009$} & ~~(1)
 & 0.9910 & \footnotesize{$\pm 0.0014$} & ~~(2)
 & 2.00 \\
    \bottomrule[1.0pt]
\end{tabular}}
\end{table}

%% file: tables_new_protocol/trend.tex
\begin{table}[htbp]
\centering
\caption{As we have described in \cref{sec:trend}, \textbf{\trend is problematic} in that it is easily `fooled' by the trivial \FF model below. For example, on Diabetes, the \FF model is able to beat diffusion-based \tabsyn.}
\label{tab:density-Trend_updated}
\resizebox{\textwidth}{!}{\begin{tabular}{
l
r@{}l@{}r
r@{}l@{}r
r@{}l@{}r
r@{}l@{}r
r@{}l@{}r
r@{}l@{}r
r@{}l@{}r
r
}
    \toprule[0.8pt]
     \textbf{Method} & \multicolumn{3}{c}{\textbf{Adult}}
 & \multicolumn{3}{c}{\textbf{Beijing}}
 & \multicolumn{3}{c}{\textbf{Default}}
 & \multicolumn{3}{c}{\textbf{Diabetes}}
 & \multicolumn{3}{c}{\textbf{Magic}}
 & \multicolumn{3}{c}{\textbf{News}}
 & \multicolumn{3}{c}{\textbf{Shoppers}}
 & \textbf{Avg. Rank} \\
    \midrule 
    CTGAN & 0.7581 & \footnotesize{$\pm 0.0205$} & ~~(9)
 & 0.7430 & \footnotesize{$\pm 0.0311$} & ~~(10)
 & 0.7024 & \footnotesize{$\pm 0.0285$} & ~~(11)
 & 0.8095 & \footnotesize{$\pm 0.0170$} & ~~(6)
 & 0.9445 & \footnotesize{$\pm 0.0071$} & ~~(6)
 & 0.9485 & \footnotesize{$\pm 0.0011$} & ~~(9)
 & 0.7554 & \footnotesize{$\pm 0.0189$} & ~~(11)
 & 8.86 \\
    TVAE & 0.6664 & \footnotesize{$\pm 0.0330$} & ~~(10)
 & 0.7022 & \footnotesize{$\pm 0.0356$} & ~~(11)
 & 0.8135 & \footnotesize{$\pm 0.0287$} & ~~(8)
 & 0.5936 & \footnotesize{$\pm 0.0925$} & ~~(8)
 & 0.9531 & \footnotesize{$\pm 0.0120$} & ~~(5)
 & 0.9386 & \footnotesize{$\pm 0.0064$} & ~~(10)
 & 0.7928 & \footnotesize{$\pm 0.0256$} & ~~(10)
 & 8.86 \\
    GReaT & 0.1907 & \footnotesize{$\pm 0.0023$} & ~~(11)
 & 0.9185 & \footnotesize{$\pm 0.0272$} & ~~(7)
 & 0.7765 & \footnotesize{$\pm 0.0187$} & ~~(9)
 & \multicolumn{2}{c}{*} & ~~(10) & 0.9076 & \footnotesize{$\pm 0.0094$} & ~~(9)
 & \multicolumn{2}{c}{*} & ~~(11) & 0.8892 & \footnotesize{$\pm 0.0047$} & ~~(6)
 & 9.00 \\
    STaSy & 0.8526 & \footnotesize{$\pm 0.0204$} & ~~(7)
 & 0.8804 & \footnotesize{$\pm 0.0203$} & ~~(9)
 & 0.9297 & \footnotesize{$\pm 0.0221$} & ~~(5)
 & 0.1480 & \footnotesize{$\pm{0.0000}$} & ~~(9) & 0.9399 & \footnotesize{$\pm 0.0158$} & ~~(7)
 & 0.9699 & \footnotesize{$\pm 0.0043$} & ~~(5)
 & 0.8745 & \footnotesize{$\pm 0.0259$} & ~~(7)
 & 7.00 \\
    CoDi & 0.7714 & \footnotesize{$\pm 0.0163$} & ~~(8)
 & 0.9258 & \footnotesize{$\pm 0.0101$} & ~~(5)
 & 0.8376 & \footnotesize{$\pm 0.0119$} & ~~(7)
 & 0.6859 & \footnotesize{$\pm{0.0000}$} & ~~(7) & 0.9397 & \footnotesize{$\pm 0.0050$} & ~~(8)
 & 0.9571 & \footnotesize{$\pm 0.0008$} & ~~(6)
 & 0.8059 & \footnotesize{$\pm 0.0095$} & ~~(9)
 & 7.14 \\
    TabSyn & 0.9795 & \footnotesize{$\pm 0.0030$} & ~~(4)
 & 0.9550 & \footnotesize{$\pm 0.0157$} & ~~(4)
 & 0.9712 & \footnotesize{$\pm 0.0094$} & ~~(3)
 & 0.9603 & \footnotesize{$\pm 0.0011$} & ~~(5)
 & 0.9918 & \footnotesize{$\pm 0.0014$} & ~~(1)
 & 0.9833 & \footnotesize{$\pm 0.0044$} & ~~(2)
 & 0.9771 & \footnotesize{$\pm 0.0013$} & ~~(3)
 & 3.14 \\
    TabDiff & 0.9846 & \footnotesize{$\pm 0.0007$} & ~~(2)
 & 0.9741 & \footnotesize{$\pm 0.0018$} & ~~(2)
 & 0.9726 & \footnotesize{$\pm 0.0068$} & ~~(2)
 & 0.9686 & \footnotesize{$\pm 0.0083$} & ~~(3)
 & 0.9916 & \footnotesize{$\pm 0.0021$} & ~~(2)
 & 0.9836 & \footnotesize{$\pm 0.0026$} & ~~(1)
 & 0.9809 & \footnotesize{$\pm 0.0017$} & ~~(2)
 & \textbf{2.00} \\
    TabPFN & 0.9152 & \footnotesize{$\pm 0.0201$} & ~~(6)
 & 0.9236 & \footnotesize{$\pm 0.0363$} & ~~(6)
 & 0.7699 & \footnotesize{$\pm 0.0542$} & ~~(10)
 & \multicolumn{2}{c}{*} & ~~(10) & 0.8936 & \footnotesize{$\pm 0.0139$} & ~~(10)
 & 0.9494 & \footnotesize{$\pm 0.0030$} & ~~(8)
 & 0.8099 & \footnotesize{$\pm 0.0233$} & ~~(8)
 & 8.29 \\
    \midrule
    FF & 0.9251 & \footnotesize{$\pm 0.0004$} & ~~(5)
 & 0.9101 & \footnotesize{$\pm 0.0010$} & ~~(8)
 & 0.8875 & \footnotesize{$\pm 0.0013$} & ~~(6)
 & 0.9685 & \footnotesize{$\pm 0.0017$} & ~~(4)
 & 0.8696 & \footnotesize{$\pm 0.0005$} & ~~(11)
 & 0.9548 & \footnotesize{$\pm 0.0006$} & ~~(7)
 & 0.9367 & \footnotesize{$\pm 0.0014$} & ~~(5)
 & 6.57 \\
    SM & 0.9814 & \footnotesize{$\pm 0.0012$} & ~~(3)
 & 0.9646 & \footnotesize{$\pm 0.0062$} & ~~(3)
 & 0.9771 & \footnotesize{$\pm 0.0013$} & ~~(1)
 & 0.9741 & \footnotesize{$\pm 0.0016$} & ~~(2)
 & 0.9789 & \footnotesize{$\pm 0.0010$} & ~~(4)
 & 0.9758 & \footnotesize{$\pm 0.0007$} & ~~(4)
 & 0.9503 & \footnotesize{$\pm 0.0091$} & ~~(4)
 & 3.00 \\
    TabPC & 0.9856 & \footnotesize{$\pm 0.0012$} & ~~(1)
 & 0.9781 & \footnotesize{$\pm 0.0025$} & ~~(1)
 & 0.9496 & \footnotesize{$\pm 0.0189$} & ~~(4)
 & 0.9810 & \footnotesize{$\pm 0.0019$} & ~~(1)
 & 0.9830 & \footnotesize{$\pm 0.0051$} & ~~(3)
 & 0.9815 & \footnotesize{$\pm 0.0020$} & ~~(3)
 & 0.9822 & \footnotesize{$\pm 0.0023$} & ~~(1)
 & \textbf{2.00} \\
    \bottomrule[1.0pt]
\end{tabular}}
\end{table}

%% file: tables_new_protocol/wnmis.tex
\begin{table}[htbp]
\centering
\caption{\textbf{\ours has the highest average rank of all methods on \wnmis}, exceeding the mean performance of the diffusion-based \sota in four out of the seven datasets.}
\label{tab:nmi_l1_weighted complement_updated}
\resizebox{\textwidth}{!}{\begin{tabular}{
l
r@{}l@{}r
r@{}l@{}r
r@{}l@{}r
r@{}l@{}r
r@{}l@{}r
r@{}l@{}r
r@{}l@{}r
r
}
    \toprule[0.8pt]
     \textbf{Method} & \multicolumn{3}{c}{\textbf{Adult}}
 & \multicolumn{3}{c}{\textbf{Beijing}}
 & \multicolumn{3}{c}{\textbf{Default}}
 & \multicolumn{3}{c}{\textbf{Diabetes}}
 & \multicolumn{3}{c}{\textbf{Magic}}
 & \multicolumn{3}{c}{\textbf{News}}
 & \multicolumn{3}{c}{\textbf{Shoppers}}
 & \textbf{Avg. Rank} \\
    \midrule 
    CTGAN & 0.8423 & \footnotesize{$\pm 0.0284$} & ~~(9)
 & 0.8734 & \footnotesize{$\pm 0.0399$} & ~~(9)
 & 0.8274 & \footnotesize{$\pm 0.0110$} & ~~(10)
 & 0.9455 & \footnotesize{$\pm 0.0028$} & ~~(4)
 & 0.9198 & \footnotesize{$\pm 0.0051$} & ~~(10)
 & 0.8679 & \footnotesize{$\pm 0.0020$} & ~~(8)
 & 0.8921 & \footnotesize{$\pm 0.0306$} & ~~(10)
 & 8.57 \\
    TVAE & 0.9127 & \footnotesize{$\pm 0.0070$} & ~~(8)
 & 0.8734 & \footnotesize{$\pm 0.0314$} & ~~(9)
 & 0.9599 & \footnotesize{$\pm 0.0099$} & ~~(5)
 & 0.8704 & \footnotesize{$\pm 0.0448$} & ~~(9)
 & 0.9719 & \footnotesize{$\pm 0.0024$} & ~~(8)
 & 0.8793 & \footnotesize{$\pm 0.0047$} & ~~(7)
 & 0.9392 & \footnotesize{$\pm 0.0110$} & ~~(7)
 & 7.57 \\
    GReaT & 0.9561 & \footnotesize{$\pm 0.0005$} & ~~(6)
 & 0.9837 & \footnotesize{$\pm 0.0016$} & ~~(4)
 & 0.8709 & \footnotesize{$\pm 0.0016$} & ~~(9)
 & \multicolumn{2}{c}{*} & ~~(10) & 0.9723 & \footnotesize{$\pm 0.0008$} & ~~(7)
 & \multicolumn{2}{c}{*} & ~~(11) & 0.9577 & \footnotesize{$\pm 0.0027$} & ~~(6)
 & 7.57 \\
    STaSy & 0.9735 & \footnotesize{$\pm 0.0049$} & ~~(5)
 & 0.9781 & \footnotesize{$\pm 0.0067$} & ~~(7)
 & 0.9504 & \footnotesize{$\pm 0.0172$} & ~~(7)
 & 0.8917 & \footnotesize{$\pm{0.0000}$} & ~~(8) & 0.9796 & \footnotesize{$\pm 0.0120$} & ~~(5)
 & 0.9392 & \footnotesize{$\pm 0.0063$} & ~~(2)
 & 0.9602 & \footnotesize{$\pm 0.0031$} & ~~(5)
 & 5.57 \\
    CoDi & 0.7898 & \footnotesize{$\pm 0.0073$} & ~~(10)
 & 0.9812 & \footnotesize{$\pm 0.0077$} & ~~(6)
 & 0.8927 & \footnotesize{$\pm 0.0049$} & ~~(8)
 & 0.9336 & \footnotesize{$\pm{0.0000}$} & ~~(6) & 0.9741 & \footnotesize{$\pm 0.0035$} & ~~(6)
 & 0.9232 & \footnotesize{$\pm 0.0036$} & ~~(5)
 & 0.9171 & \footnotesize{$\pm 0.0086$} & ~~(9)
 & 7.14 \\
    TabSyn & 0.9860 & \footnotesize{$\pm 0.0026$} & ~~(4)
 & 0.9833 & \footnotesize{$\pm 0.0076$} & ~~(5)
 & 0.9846 & \footnotesize{$\pm 0.0029$} & ~~(3)
 & 0.9797 & \footnotesize{$\pm 0.0006$} & ~~(2)
 & 0.9971 & \footnotesize{$\pm 0.0003$} & ~~(1)
 & 0.9441 & \footnotesize{$\pm 0.0122$} & ~~(1)
 & 0.9832 & \footnotesize{$\pm 0.0042$} & ~~(3)
 & 2.71 \\
    TabDiff & 0.9889 & \footnotesize{$\pm 0.0009$} & ~~(2)
 & 0.9927 & \footnotesize{$\pm 0.0004$} & ~~(2)
 & 0.9881 & \footnotesize{$\pm 0.0016$} & ~~(2)
 & 0.9687 & \footnotesize{$\pm 0.0309$} & ~~(3)
 & 0.9953 & \footnotesize{$\pm 0.0016$} & ~~(2)
 & 0.9127 & \footnotesize{$\pm 0.0262$} & ~~(6)
 & 0.9844 & \footnotesize{$\pm 0.0011$} & ~~(2)
 & 2.71 \\
    TabPFN & 0.9885 & \footnotesize{$\pm 0.0007$} & ~~(3)
 & 0.9971 & \footnotesize{$\pm 0.0005$} & ~~(1)
 & 0.9648 & \footnotesize{$\pm 0.0012$} & ~~(4)
 & \multicolumn{2}{c}{*} & ~~(10) & 0.9941 & \footnotesize{$\pm 0.0010$} & ~~(3)
 & 0.9249 & \footnotesize{$\pm 0.0002$} & ~~(4)
 & 0.9614 & \footnotesize{$\pm 0.0009$} & ~~(4)
 & 4.14 \\
    \midrule
    FF & 0.7189 & \footnotesize{$\pm 0.0002$} & ~~(11)
 & 0.7918 & \footnotesize{$\pm 0.0001$} & ~~(11)
 & 0.6942 & \footnotesize{$\pm 0.0001$} & ~~(11)
 & 0.9082 & \footnotesize{$\pm 0.0002$} & ~~(7)
 & 0.7684 & \footnotesize{$\pm{0.0000}$} & ~~(11) & 0.7421 & \footnotesize{$\pm 0.0001$} & ~~(10)
 & 0.8661 & \footnotesize{$\pm 0.0002$} & ~~(11)
 & 10.29 \\
    SM & 0.9540 & \footnotesize{$\pm 0.0071$} & ~~(7)
 & 0.9696 & \footnotesize{$\pm 0.0032$} & ~~(8)
 & 0.9573 & \footnotesize{$\pm 0.0016$} & ~~(6)
 & 0.9342 & \footnotesize{$\pm 0.0088$} & ~~(5)
 & 0.9432 & \footnotesize{$\pm 0.0037$} & ~~(9)
 & 0.8175 & \footnotesize{$\pm 0.0081$} & ~~(9)
 & 0.9266 & \footnotesize{$\pm 0.0048$} & ~~(8)
 & 7.43 \\
    TabPC & 0.9933 & \footnotesize{$\pm 0.0049$} & ~~(1)
 & 0.9925 & \footnotesize{$\pm 0.0010$} & ~~(3)
 & 0.9894 & \footnotesize{$\pm 0.0012$} & ~~(1)
 & 0.9919 & \footnotesize{$\pm 0.0024$} & ~~(1)
 & 0.9925 & \footnotesize{$\pm 0.0012$} & ~~(4)
 & 0.9388 & \footnotesize{$\pm 0.0048$} & ~~(3)
 & 0.9927 & \footnotesize{$\pm 0.0009$} & ~~(1)
 & \textbf{2.00} \\
    \bottomrule[1.0pt]
\end{tabular}}
\end{table}

%% file: tables_new_protocol/lr_detection.tex
\begin{table}[htbp]
\centering
\caption{\textbf{\ctwost (LR) is flawed and should not be considered a strong indicator of model performance.} This is evidenced by the performance of the \FF model, which is able to achieve perfect or almost-perfect scores across a range of datasets.}
\label{tab:logistic regression(C2ST)_updated}
\resizebox{\textwidth}{!}{\begin{tabular}{
l
r@{}l@{}r
r@{}l@{}r
r@{}l@{}r
r@{}l@{}r
r@{}l@{}r
r@{}l@{}r
r@{}l@{}r
r
}
    \toprule[0.8pt]
     \textbf{Method} & \multicolumn{3}{c}{\textbf{Adult}}
 & \multicolumn{3}{c}{\textbf{Beijing}}
 & \multicolumn{3}{c}{\textbf{Default}}
 & \multicolumn{3}{c}{\textbf{Diabetes}}
 & \multicolumn{3}{c}{\textbf{Magic}}
 & \multicolumn{3}{c}{\textbf{News}}
 & \multicolumn{3}{c}{\textbf{Shoppers}}
 & \textbf{Avg. Rank} \\
    \midrule 
    CTGAN & 0.5881 & \footnotesize{$\pm 0.0928$} & ~~(7)
 & 0.5534 & \footnotesize{$\pm 0.1899$} & ~~(10)
 & 0.4129 & \footnotesize{$\pm 0.0875$} & ~~(10)
 & 0.5293 & \footnotesize{$\pm 0.0693$} & ~~(6)
 & 0.7966 & \footnotesize{$\pm 0.0422$} & ~~(8)
 & 0.7486 & \footnotesize{$\pm 0.0428$} & ~~(7)
 & 0.5255 & \footnotesize{$\pm 0.0319$} & ~~(7)
 & 7.86 \\
    TVAE & 0.2286 & \footnotesize{$\pm 0.0416$} & ~~(9)
 & 0.3707 & \footnotesize{$\pm 0.0498$} & ~~(11)
 & 0.6662 & \footnotesize{$\pm 0.0182$} & ~~(7)
 & 0.0177 & \footnotesize{$\pm 0.0259$} & ~~(7)
 & 0.8765 & \footnotesize{$\pm 0.0261$} & ~~(7)
 & 0.4355 & \footnotesize{$\pm 0.0321$} & ~~(9)
 & 0.2728 & \footnotesize{$\pm 0.0675$} & ~~(10)
 & 8.57 \\
    GReaT & \multicolumn{2}{c}{*} & ~~(11) & 0.7645 & \footnotesize{$\pm 0.0040$} & ~~(7)
 & 0.4815 & \footnotesize{$\pm 0.0019$} & ~~(9)
 & \multicolumn{2}{c}{*} & ~~(9) & 0.4628 & \footnotesize{$\pm 0.0027$} & ~~(11)
 & \multicolumn{2}{c}{*} & ~~(11) & 0.4402 & \footnotesize{$\pm 0.0047$} & ~~(8)
 & 9.43 \\
    STaSy & 0.4537 & \footnotesize{$\pm 0.0268$} & ~~(8)
 & 0.6531 & \footnotesize{$\pm 0.0414$} & ~~(8)
 & 0.6303 & \footnotesize{$\pm 0.0628$} & ~~(8)
 & \multicolumn{2}{c}{*} & ~~(9) & 0.5496 & \footnotesize{$\pm 0.0707$} & ~~(10)
 & 0.4448 & \footnotesize{$\pm 0.1707$} & ~~(8)
 & 0.3860 & \footnotesize{$\pm 0.0979$} & ~~(9)
 & 8.57 \\
    CoDi & 0.1986 & \footnotesize{$\pm 0.0286$} & ~~(10)
 & 0.6345 & \footnotesize{$\pm 0.2901$} & ~~(9)
 & 0.3287 & \footnotesize{$\pm 0.0602$} & ~~(11)
 & 0.0023 & \footnotesize{$\pm{0.0000}$} & ~~(8) & 0.7425 & \footnotesize{$\pm 0.0238$} & ~~(9)
 & 0.1133 & \footnotesize{$\pm 0.0860$} & ~~(10)
 & 0.2169 & \footnotesize{$\pm 0.0164$} & ~~(11)
 & 9.71 \\
    TabSyn & 0.9866 & \footnotesize{$\pm 0.0081$} & ~~(5)
 & 0.9226 & \footnotesize{$\pm 0.0303$} & ~~(5)
 & 0.9463 & \footnotesize{$\pm 0.0588$} & ~~(5)
 & 0.6617 & \footnotesize{$\pm 0.0397$} & ~~(5)
 & 0.9940 & \footnotesize{$\pm 0.0051$} & ~~(3)
 & 0.9381 & \footnotesize{$\pm 0.0629$} & ~~(4)
 & 0.9787 & \footnotesize{$\pm 0.0141$} & ~~(3)
 & 4.29 \\
    TabDiff & 0.9899 & \footnotesize{$\pm 0.0054$} & ~~(3)
 & 0.9770 & \footnotesize{$\pm 0.0027$} & ~~(4)
 & 0.9699 & \footnotesize{$\pm 0.0095$} & ~~(4)
 & 0.9290 & \footnotesize{$\pm 0.0651$} & ~~(4)
 & 0.9964 & \footnotesize{$\pm 0.0035$} & ~~(1)
 & 0.9057 & \footnotesize{$\pm 0.0588$} & ~~(5)
 & 0.9756 & \footnotesize{$\pm 0.0170$} & ~~(4)
 & 3.57 \\
    TabPFN & 0.8864 & \footnotesize{$\pm 0.0027$} & ~~(6)
 & 0.9961 & \footnotesize{$\pm 0.0029$} & ~~(3)
 & 0.8820 & \footnotesize{$\pm 0.0068$} & ~~(6)
 & \multicolumn{2}{c}{*} & ~~(9) & 0.9722 & \footnotesize{$\pm 0.0040$} & ~~(5)
 & 0.8380 & \footnotesize{$\pm 0.0009$} & ~~(6)
 & 0.7170 & \footnotesize{$\pm 0.0067$} & ~~(6)
 & 5.86 \\
    \midrule
    FF & 1.0000 & \footnotesize{$\pm{0.0000}$} & ~~(1) & 0.9979 & \footnotesize{$\pm 0.0029$} & ~~(2)
 & 0.9949 & \footnotesize{$\pm 0.0045$} & ~~(2)
 & 1.0000 & \footnotesize{$\pm{0.0000}$} & ~~(1) & 0.9892 & \footnotesize{$\pm 0.0072$} & ~~(4)
 & 0.9795 & \footnotesize{$\pm 0.0132$} & ~~(2)
 & 1.0000 & \footnotesize{$\pm{0.0000}$} & ~~(1) & 1.86 \\
    SM & 0.9879 & \footnotesize{$\pm 0.0041$} & ~~(4)
 & 0.9133 & \footnotesize{$\pm 0.0119$} & ~~(6)
 & 0.9826 & \footnotesize{$\pm 0.0054$} & ~~(3)
 & 0.9993 & \footnotesize{$\pm 0.0016$} & ~~(3)
 & 0.9551 & \footnotesize{$\pm 0.0056$} & ~~(6)
 & 0.9765 & \footnotesize{$\pm 0.0128$} & ~~(3)
 & 0.9106 & \footnotesize{$\pm 0.0501$} & ~~(5)
 & 4.29 \\
    TabPC & 0.9959 & \footnotesize{$\pm 0.0083$} & ~~(2)
 & 0.9989 & \footnotesize{$\pm 0.0015$} & ~~(1)
 & 0.9999 & \footnotesize{$\pm 0.0003$} & ~~(1)
 & 1.0000 & \footnotesize{$\pm{0.0000}$} & ~~(1) & 0.9953 & \footnotesize{$\pm 0.0067$} & ~~(2)
 & 0.9960 & \footnotesize{$\pm 0.0033$} & ~~(1)
 & 0.9969 & \footnotesize{$\pm 0.0056$} & ~~(2)
 & \textbf{1.43} \\
    \bottomrule[1.0pt]
\end{tabular}}
\end{table}

%% file: tables_new_protocol/xgb_detection.tex
\begin{table}[htbp]
\centering
\caption{\textbf{\ours offers competitive performance on \ctwost (XGB),} {slightly exceeding} the average rank across all datasets of \tabdiff. In Diabetes and Shoppers, \ours also greatly exceeds the performance of the current \sota, but suffers compared to \tabdiff on Beijing and Default. We note that all datasets struggle on the complex dataset News, which contains many numerical features and skewed distributions. {We use the default package parameters for our XGBoost classifier.}}
\label{tab:xgboost(C2ST)_updated}
\resizebox{\textwidth}{!}{\begin{tabular}{
l
r@{}l@{}r
r@{}l@{}r
r@{}l@{}r
r@{}l@{}r
r@{}l@{}r
r@{}l@{}r
r@{}l@{}r
r
}
    \toprule[0.8pt]
     \textbf{Method} & \multicolumn{3}{c}{\textbf{Adult}}
 & \multicolumn{3}{c}{\textbf{Beijing}}
 & \multicolumn{3}{c}{\textbf{Default}}
 & \multicolumn{3}{c}{\textbf{Diabetes}}
 & \multicolumn{3}{c}{\textbf{Magic}}
 & \multicolumn{3}{c}{\textbf{News}}
 & \multicolumn{3}{c}{\textbf{Shoppers}}
 & \textbf{Avg. Rank} \\
    \midrule 
    CTGAN & \multicolumn{2}{c}{<0.0001} & ~~(8) & \multicolumn{2}{c}{<0.0001} & ~~(10) & 0.0002 & \footnotesize{$\pm 0.0001$} & ~~(10)
 & \multicolumn{2}{c}{<0.0001} & ~~(6) & 0.1034 & \footnotesize{$\pm 0.0057$} & ~~(10)
 & \multicolumn{2}{c}{<0.0001} & ~~(6) & \multicolumn{2}{c}{<0.0001} & ~~(10) & 8.57 \\
    TVAE & \multicolumn{2}{c}{<0.0001} & ~~(8) & \multicolumn{2}{c}{<0.0001} & ~~(10) & 0.0020 & \footnotesize{$\pm 0.0002$} & ~~(8)
 & \multicolumn{2}{c}{<0.0001} & ~~(6) & 0.2981 & \footnotesize{$\pm 0.0276$} & ~~(8)
 & \multicolumn{2}{c}{<0.0001} & ~~(6) & \multicolumn{2}{c}{<0.0001} & ~~(10) & 8.00 \\
    GReaT & \multicolumn{2}{c}{*} & ~~(11) & 0.6016 & \footnotesize{$\pm 0.0028$} & ~~(3)
 & 0.2425 & \footnotesize{$\pm 0.0012$} & ~~(5)
 & \multicolumn{2}{c}{*} & ~~(9) & 0.2702 & \footnotesize{$\pm 0.0027$} & ~~(9)
 & \multicolumn{2}{c}{*} & ~~(11) & 0.2924 & \footnotesize{$\pm 0.0034$} & ~~(4)
 & 7.43 \\
    STaSy & 0.3425 & \footnotesize{$\pm 0.0272$} & ~~(6)
 & 0.3668 & \footnotesize{$\pm 0.0243$} & ~~(6)
 & 0.3571 & \footnotesize{$\pm 0.0370$} & ~~(3)
 & \multicolumn{2}{c}{*} & ~~(9) & 0.5221 & \footnotesize{$\pm 0.0790$} & ~~(6)
 & 0.1917 & \footnotesize{$\pm 0.0525$} & ~~(1)
 & 0.2277 & \footnotesize{$\pm 0.0806$} & ~~(6)
 & 5.29 \\
    CoDi & <0.0001 & \footnotesize{$\pm 0.0001$} & ~~(8)
 & 0.0002 & \footnotesize{$\pm 0.0001$} & ~~(9)
 & 0.0014 & \footnotesize{$\pm 0.0002$} & ~~(9)
 & \multicolumn{2}{c}{<0.0001} & ~~(6) & 0.4201 & \footnotesize{$\pm 0.0430$} & ~~(7)
 & \multicolumn{2}{c}{<0.0001} & ~~(6) & 0.0002 & \footnotesize{$\pm 0.0001$} & ~~(9)
 & 7.71 \\
    TabSyn & 0.7942 & \footnotesize{$\pm 0.0341$} & ~~(4)
 & 0.5266 & \footnotesize{$\pm 0.1397$} & ~~(5)
 & 0.4024 & \footnotesize{$\pm 0.0515$} & ~~(2)
 & 0.5222 & \footnotesize{$\pm 0.0362$} & ~~(3)
 & 0.7602 & \footnotesize{$\pm 0.0310$} & ~~(4)
 & 0.1208 & \footnotesize{$\pm 0.0590$} & ~~(3)
 & 0.6554 & \footnotesize{$\pm 0.0196$} & ~~(3)
 & 3.43 \\
    TabDiff & 0.8409 & \footnotesize{$\pm 0.0067$} & ~~(2)
 & 0.7723 & \footnotesize{$\pm 0.0138$} & ~~(1)
 & 0.5914 & \footnotesize{$\pm 0.0226$} & ~~(1)
 & 0.5618 & \footnotesize{$\pm 0.2842$} & ~~(2)
 & 0.7732 & \footnotesize{$\pm 0.0278$} & ~~(3)
 & 0.0296 & \footnotesize{$\pm 0.0252$} & ~~(4)
 & 0.7361 & \footnotesize{$\pm 0.0162$} & ~~(2)
 & 2.14 \\
    TabPFN & 0.8257 & \footnotesize{$\pm 0.0042$} & ~~(3)
 & 0.2289 & \footnotesize{$\pm 0.0048$} & ~~(7)
 & 0.0798 & \footnotesize{$\pm 0.0003$} & ~~(7)
 & \multicolumn{2}{c}{*} & ~~(9) & 0.8572 & \footnotesize{$\pm 0.0048$} & ~~(2)
 & 0.0002 & \footnotesize{$\pm 0.0001$} & ~~(5)
 & 0.0021 & \footnotesize{$\pm 0.0002$} & ~~(8)
 & 5.86 \\
    \midrule
    FF & 0.0160 & \footnotesize{$\pm 0.0007$} & ~~(7)
 & 0.0172 & \footnotesize{$\pm 0.0006$} & ~~(8)
 & 0.0002 & \footnotesize{$\pm 0.0001$} & ~~(10)
 & 0.0503 & \footnotesize{$\pm 0.0009$} & ~~(5)
 & 0.0052 & \footnotesize{$\pm 0.0002$} & ~~(11)
 & \multicolumn{2}{c}{<0.0001} & ~~(6) & 0.0211 & \footnotesize{$\pm 0.0013$} & ~~(7)
 & 7.71 \\
    SM & 0.7727 & \footnotesize{$\pm 0.0243$} & ~~(5)
 & 0.7332 & \footnotesize{$\pm 0.0148$} & ~~(2)
 & 0.1053 & \footnotesize{$\pm 0.0127$} & ~~(6)
 & 0.1329 & \footnotesize{$\pm 0.0328$} & ~~(4)
 & 0.6258 & \footnotesize{$\pm 0.0228$} & ~~(5)
 & <0.0001 & \footnotesize{$\pm 0.0001$} & ~~(6)
 & 0.2901 & \footnotesize{$\pm 0.0251$} & ~~(5)
 & 4.71 \\
    TabPC & 0.8554 & \footnotesize{$\pm 0.0328$} & ~~(1)
 & 0.5704 & \footnotesize{$\pm 0.0262$} & ~~(4)
 & 0.2617 & \footnotesize{$\pm 0.0052$} & ~~(4)
 & 0.7915 & \footnotesize{$\pm 0.0134$} & ~~(1)
 & 0.8602 & \footnotesize{$\pm 0.0142$} & ~~(1)
 & 0.1286 & \footnotesize{$\pm 0.0182$} & ~~(2)
 & 0.8936 & \footnotesize{$\pm 0.0100$} & ~~(1)
 & \textbf{2.00} \\
    \bottomrule[1.0pt]
\end{tabular}}
\end{table}

%% file: tables_new_protocol/alpha_precision_oc.tex
\begin{table}[htbp]
\centering
\caption{\textbf{\ours offers competitive performance on \alphaprecision compared to diffusion-based \sota}, a measure of sample fidelity, though its average rank is slightly lower than top-performing \tabdiff. However, the error bars for certain datasets like Adult and Beijing overlap, and hence ranking by mean does not tell the full story. Note that the following numbers are \textbf{not} backwards compatible with recent works like \tabsyn and \tabdiff due to the implementation being different; \cref{app:ap_br_discrepancy} has further details.}
\label{tab:alpha precision oc_updated}
\resizebox{\textwidth}{!}{\begin{tabular}{
l
r@{}l@{}r
r@{}l@{}r
r@{}l@{}r
r@{}l@{}r
r@{}l@{}r
r@{}l@{}r
r@{}l@{}r
r
}
    \toprule[0.8pt]
     \textbf{Method} & \multicolumn{3}{c}{\textbf{Adult}}
 & \multicolumn{3}{c}{\textbf{Beijing}}
 & \multicolumn{3}{c}{\textbf{Default}}
 & \multicolumn{3}{c}{\textbf{Diabetes}}
 & \multicolumn{3}{c}{\textbf{Magic}}
 & \multicolumn{3}{c}{\textbf{News}}
 & \multicolumn{3}{c}{\textbf{Shoppers}}
 & \textbf{Avg. Rank} \\
    \midrule 
    CTGAN & 0.8951 & \footnotesize{$\pm 0.0313$} & ~~(8)
 & 0.5669 & \footnotesize{$\pm 0.1178$} & ~~(9)
 & 0.9313 & \footnotesize{$\pm 0.0372$} & ~~(10)
 & 0.8381 & \footnotesize{$\pm 0.0976$} & ~~(6)
 & 0.8060 & \footnotesize{$\pm 0.0238$} & ~~(10)
 & 0.2107 & \footnotesize{$\pm 0.0023$} & ~~(3)
 & 0.7908 & \footnotesize{$\pm 0.0899$} & ~~(8)
 & 7.71 \\
    TVAE & 0.5219 & \footnotesize{$\pm 0.0182$} & ~~(11)
 & 0.9795 & \footnotesize{$\pm 0.0137$} & ~~(2)
 & 0.8898 & \footnotesize{$\pm 0.0145$} & ~~(11)
 & 0.7829 & \footnotesize{$\pm 0.0952$} & ~~(7)
 & 0.9580 & \footnotesize{$\pm 0.0350$} & ~~(5)
 & 0.1927 & \footnotesize{$\pm 0.0018$} & ~~(10)
 & 0.9220 & \footnotesize{$\pm 0.0139$} & ~~(4)
 & 7.14 \\
    GReaT & 0.9494 & \footnotesize{$\pm 0.0002$} & ~~(7)
 & \multicolumn{2}{c}{*} & ~~(11) & 0.9694 & \footnotesize{$\pm 0.0015$} & ~~(4)
 & \multicolumn{2}{c}{*} & ~~(10) & 0.9576 & \footnotesize{$\pm 0.0033$} & ~~(6)
 & \multicolumn{2}{c}{*} & ~~(11) & 0.7115 & \footnotesize{$\pm 0.0044$} & ~~(9)
 & 8.29 \\
    STaSy & 0.8573 & \footnotesize{$\pm 0.0824$} & ~~(9)
 & 0.9074 & \footnotesize{$\pm 0.0325$} & ~~(8)
 & 0.9489 & \footnotesize{$\pm 0.0358$} & ~~(7)
 & \multicolumn{2}{c}{<0.0001} & ~~(9) & 0.9353 & \footnotesize{$\pm 0.0547$} & ~~(8)
 & 0.1943 & \footnotesize{$\pm 0.0034$} & ~~(9)
 & 0.8010 & \footnotesize{$\pm 0.1559$} & ~~(7)
 & 8.14 \\
    CoDi & 0.8538 & \footnotesize{$\pm 0.0172$} & ~~(10)
 & 0.9664 & \footnotesize{$\pm 0.0227$} & ~~(3)
 & 0.9402 & \footnotesize{$\pm 0.0083$} & ~~(8)
 & 0.2804 & \footnotesize{$\pm{0.0000}$} & ~~(8) & 0.9890 & \footnotesize{$\pm 0.0055$} & ~~(2)
 & 0.2018 & \footnotesize{$\pm 0.0053$} & ~~(5)
 & 0.6690 & \footnotesize{$\pm 0.0133$} & ~~(10)
 & 6.57 \\
    TabSyn & 0.9948 & \footnotesize{$\pm 0.0027$} & ~~(2)
 & 0.9375 & \footnotesize{$\pm 0.0609$} & ~~(6)
 & 0.9747 & \footnotesize{$\pm 0.0033$} & ~~(3)
 & 0.9799 & \footnotesize{$\pm 0.0047$} & ~~(2)
 & 0.9863 & \footnotesize{$\pm 0.0047$} & ~~(3)
 & 0.1958 & \footnotesize{$\pm 0.0022$} & ~~(7)
 & 0.9886 & \footnotesize{$\pm 0.0049$} & ~~(1)
 & 3.43 \\
    TabDiff & 0.9955 & \footnotesize{$\pm 0.0011$} & ~~(1)
 & 0.9477 & \footnotesize{$\pm 0.0029$} & ~~(4)
 & 0.9891 & \footnotesize{$\pm 0.0027$} & ~~(1)
 & 0.9680 & \footnotesize{$\pm 0.0445$} & ~~(3)
 & 0.9820 & \footnotesize{$\pm 0.0052$} & ~~(4)
 & 0.2068 & \footnotesize{$\pm 0.0077$} & ~~(4)
 & 0.9649 & \footnotesize{$\pm 0.0050$} & ~~(3)
 & \textbf{2.86} \\
    TabPFN & 0.9781 & \footnotesize{$\pm 0.0031$} & ~~(6)
 & 0.9883 & \footnotesize{$\pm 0.0032$} & ~~(1)
 & 0.9682 & \footnotesize{$\pm 0.0028$} & ~~(6)
 & \multicolumn{2}{c}{*} & ~~(10) & 0.9941 & \footnotesize{$\pm 0.0026$} & ~~(1)
 & 0.1950 & \footnotesize{$\pm 0.0002$} & ~~(8)
 & 0.9168 & \footnotesize{$\pm 0.0040$} & ~~(5)
 & 5.29 \\
    \midrule
    FF & 0.9927 & \footnotesize{$\pm 0.0026$} & ~~(4)
 & 0.4089 & \footnotesize{$\pm 0.0021$} & ~~(10)
 & 0.9348 & \footnotesize{$\pm 0.0008$} & ~~(9)
 & 0.9333 & \footnotesize{$\pm 0.0029$} & ~~(5)
 & 0.4475 & \footnotesize{$\pm 0.0038$} & ~~(11)
 & 0.3023 & \footnotesize{$\pm 0.0018$} & ~~(1)
 & 0.5133 & \footnotesize{$\pm 0.0076$} & ~~(11)
 & 7.29 \\
    SM & 0.9931 & \footnotesize{$\pm 0.0016$} & ~~(3)
 & 0.9322 & \footnotesize{$\pm 0.0147$} & ~~(7)
 & 0.9687 & \footnotesize{$\pm 0.0064$} & ~~(5)
 & 0.9378 & \footnotesize{$\pm 0.0022$} & ~~(4)
 & 0.9181 & \footnotesize{$\pm 0.0062$} & ~~(9)
 & 0.2339 & \footnotesize{$\pm 0.0059$} & ~~(2)
 & 0.8142 & \footnotesize{$\pm 0.0305$} & ~~(6)
 & 5.14 \\
    TabPC & 0.9912 & \footnotesize{$\pm 0.0038$} & ~~(5)
 & 0.9419 & \footnotesize{$\pm 0.0060$} & ~~(5)
 & 0.9757 & \footnotesize{$\pm 0.0060$} & ~~(2)
 & 0.9833 & \footnotesize{$\pm 0.0075$} & ~~(1)
 & 0.9567 & \footnotesize{$\pm 0.0054$} & ~~(7)
 & 0.1959 & \footnotesize{$\pm 0.0003$} & ~~(6)
 & 0.9875 & \footnotesize{$\pm 0.0049$} & ~~(2)
 & 4.00 \\
    \bottomrule[1.0pt]
\end{tabular}}
\end{table}

%% file: tables_new_protocol/beta_recall_oc.tex
\begin{table}[htbp]
\centering
\caption{\textbf{\ours also offers competitive performance on \betarecall compared to \sota methods}, a measure of sample diversity. One interesting observation is that \tabpfn performs relatively higher than it does in other metrics here. As with \alphaprecision, error bars can often overlap. Note that the following numbers are \textbf{not} backwards compatible with recent works like \tabsyn and \tabdiff due to the implementation being different; \cref{app:ap_br_discrepancy} has further details.}
\label{tab:beta recall oc_updated}
\resizebox{\textwidth}{!}{\begin{tabular}{
l
r@{}l@{}r
r@{}l@{}r
r@{}l@{}r
r@{}l@{}r
r@{}l@{}r
r@{}l@{}r
r@{}l@{}r
r
}
    \toprule[0.8pt]
     \textbf{Method} & \multicolumn{3}{c}{\textbf{Adult}}
 & \multicolumn{3}{c}{\textbf{Beijing}}
 & \multicolumn{3}{c}{\textbf{Default}}
 & \multicolumn{3}{c}{\textbf{Diabetes}}
 & \multicolumn{3}{c}{\textbf{Magic}}
 & \multicolumn{3}{c}{\textbf{News}}
 & \multicolumn{3}{c}{\textbf{Shoppers}}
 & \textbf{Avg. Rank} \\
    \midrule 
    CTGAN & 0.4497 & \footnotesize{$\pm 0.0486$} & ~~(11)
 & 0.1670 & \footnotesize{$\pm 0.0551$} & ~~(9)
 & 0.4598 & \footnotesize{$\pm 0.0277$} & ~~(11)
 & 0.3055 & \footnotesize{$\pm 0.0107$} & ~~(6)
 & 0.3249 & \footnotesize{$\pm 0.0039$} & ~~(10)
 & 0.2077 & \footnotesize{$\pm 0.0001$} & ~~(5)
 & 0.2895 & \footnotesize{$\pm 0.0190$} & ~~(9)
 & 8.71 \\
    TVAE & 0.4547 & \footnotesize{$\pm 0.0136$} & ~~(10)
 & 0.2579 & \footnotesize{$\pm 0.0225$} & ~~(8)
 & 0.4696 & \footnotesize{$\pm 0.0114$} & ~~(8)
 & 0.1463 & \footnotesize{$\pm 0.0909$} & ~~(7)
 & 0.4556 & \footnotesize{$\pm 0.0096$} & ~~(7)
 & 0.2076 & \footnotesize{$\pm 0.0010$} & ~~(6)
 & 0.4103 & \footnotesize{$\pm 0.0096$} & ~~(7)
 & 7.57 \\
    GReaT & 0.5191 & \footnotesize{$\pm 0.0008$} & ~~(1)
 & \multicolumn{2}{c}{*} & ~~(11) & 0.4675 & \footnotesize{$\pm 0.0025$} & ~~(10)
 & \multicolumn{2}{c}{*} & ~~(10) & 0.4439 & \footnotesize{$\pm 0.0068$} & ~~(9)
 & \multicolumn{2}{c}{*} & ~~(11) & 0.4560 & \footnotesize{$\pm 0.0026$} & ~~(4)
 & 8.00 \\
    STaSy & 0.5020 & \footnotesize{$\pm 0.0102$} & ~~(5)
 & 0.3859 & \footnotesize{$\pm 0.0137$} & ~~(6)
 & 0.4924 & \footnotesize{$\pm 0.0048$} & ~~(5)
 & \multicolumn{2}{c}{<0.0001} & ~~(9) & 0.4540 & \footnotesize{$\pm 0.0269$} & ~~(8)
 & 0.2083 & \footnotesize{$\pm 0.0009$} & ~~(1)
 & 0.3752 & \footnotesize{$\pm 0.0583$} & ~~(8)
 & 6.00 \\
    CoDi & 0.5164 & \footnotesize{$\pm 0.0037$} & ~~(2)
 & 0.3696 & \footnotesize{$\pm 0.0178$} & ~~(7)
 & 0.4841 & \footnotesize{$\pm 0.0068$} & ~~(7)
 & 0.0388 & \footnotesize{$\pm{0.0000}$} & ~~(8) & 0.4821 & \footnotesize{$\pm 0.0091$} & ~~(5)
 & 0.2078 & \footnotesize{$\pm 0.0004$} & ~~(4)
 & 0.2158 & \footnotesize{$\pm 0.0332$} & ~~(11)
 & 6.29 \\
    TabSyn & 0.5003 & \footnotesize{$\pm 0.0041$} & ~~(7)
 & 0.4125 & \footnotesize{$\pm 0.0336$} & ~~(4)
 & 0.4940 & \footnotesize{$\pm 0.0050$} & ~~(3)
 & 0.4765 & \footnotesize{$\pm 0.0033$} & ~~(2)
 & 0.4873 & \footnotesize{$\pm 0.0048$} & ~~(2)
 & 0.2075 & \footnotesize{$\pm 0.0005$} & ~~(7)
 & 0.4952 & \footnotesize{$\pm 0.0059$} & ~~(3)
 & 4.00 \\
    TabDiff & 0.5031 & \footnotesize{$\pm 0.0026$} & ~~(4)
 & 0.4373 & \footnotesize{$\pm 0.0046$} & ~~(2)
 & 0.5018 & \footnotesize{$\pm 0.0015$} & ~~(1)
 & 0.4705 & \footnotesize{$\pm 0.0632$} & ~~(3)
 & 0.4873 & \footnotesize{$\pm 0.0025$} & ~~(2)
 & 0.2075 & \footnotesize{$\pm 0.0001$} & ~~(7)
 & 0.4998 & \footnotesize{$\pm 0.0090$} & ~~(2)
 & \textbf{3.00} \\
    TabPFN & 0.5079 & \footnotesize{$\pm 0.0027$} & ~~(3)
 & 0.5204 & \footnotesize{$\pm 0.0033$} & ~~(1)
 & 0.4953 & \footnotesize{$\pm 0.0034$} & ~~(2)
 & \multicolumn{2}{c}{*} & ~~(10) & 0.4965 & \footnotesize{$\pm 0.0016$} & ~~(1)
 & 0.2080 & \footnotesize{$\pm 0.0002$} & ~~(3)
 & 0.4337 & \footnotesize{$\pm 0.0098$} & ~~(5)
 & 3.57 \\
    \midrule
    FF & 0.4975 & \footnotesize{$\pm 0.0030$} & ~~(9)
 & 0.0961 & \footnotesize{$\pm 0.0030$} & ~~(10)
 & 0.4689 & \footnotesize{$\pm 0.0018$} & ~~(9)
 & 0.3670 & \footnotesize{$\pm 0.0015$} & ~~(5)
 & 0.1318 & \footnotesize{$\pm 0.0026$} & ~~(11)
 & 0.2046 & \footnotesize{$\pm 0.0001$} & ~~(10)
 & 0.2536 & \footnotesize{$\pm 0.0064$} & ~~(10)
 & 9.14 \\
    SM & 0.5018 & \footnotesize{$\pm 0.0042$} & ~~(6)
 & 0.4214 & \footnotesize{$\pm 0.0059$} & ~~(3)
 & 0.4925 & \footnotesize{$\pm 0.0029$} & ~~(4)
 & 0.4029 & \footnotesize{$\pm 0.0123$} & ~~(4)
 & 0.4616 & \footnotesize{$\pm 0.0065$} & ~~(6)
 & 0.2073 & \footnotesize{$\pm 0.0001$} & ~~(9)
 & 0.4235 & \footnotesize{$\pm 0.0138$} & ~~(6)
 & 5.43 \\
    TabPC & 0.4983 & \footnotesize{$\pm 0.0027$} & ~~(8)
 & 0.3886 & \footnotesize{$\pm 0.0030$} & ~~(5)
 & 0.4914 & \footnotesize{$\pm 0.0045$} & ~~(6)
 & 0.4804 & \footnotesize{$\pm 0.0027$} & ~~(1)
 & 0.4839 & \footnotesize{$\pm 0.0031$} & ~~(4)
 & 0.2081 & \footnotesize{$\pm 0.0001$} & ~~(2)
 & 0.5049 & \footnotesize{$\pm 0.0035$} & ~~(1)
 & 3.86 \\
    \bottomrule[1.0pt]
\end{tabular}}
\end{table}

%% file: tables_new_protocol/mle.tex
\begin{table}[htbp]
\centering
\caption{\textbf{\ours generates synthetic data which is generally competitive with \sota for training new machine learning models} (in terms of \textbf{\textit{machine learning efficacy; MLE}}). Values for Beijing and News are root mean squared errors (RMSEs), whereas values for other datasets are classifier area under receiver operating characteristic curves (AUROCs). Metric direction of improvement is indicated by the arrow next to the dataset name. Values should be compared to models which used the real training data, which can be found in the first row.}
\label{tab:mle_updated}
\resizebox{\textwidth}{!}{\begin{tabular}{
l
r@{}l
r@{}l
r@{}l
r@{}l
r@{}l
r@{}l
r@{}l
}
    \toprule[0.8pt]
     \textbf{Method}
     & \multicolumn{2}{c}{\textbf{Adult $(\uparrow)$}}
 & \multicolumn{2}{c}{\textbf{Beijing $(\downarrow)$}}
 & \multicolumn{2}{c}{\textbf{Default $(\uparrow)$}}
 & \multicolumn{2}{c}{\textbf{Diabetes $(\uparrow)$}}
 & \multicolumn{2}{c}{\textbf{Magic $(\uparrow)$}}
 & \multicolumn{2}{c}{\textbf{News $(\downarrow)$}}
 & \multicolumn{2}{c}{\textbf{Shoppers $(\uparrow)$}}
 \\
    \midrule 
      \textit{Real Data} & 0.9271 & \footnotesize{$\pm 0.0006$}
 & 0.4367 & \footnotesize{$\pm 0.0116$}
 & 0.7697 & \footnotesize{$\pm 0.0008$}
 & 0.7040 & \footnotesize{$\pm 0.0016$}
 & 0.9478 & \footnotesize{$\pm 0.0024$}
 & 0.8403 & \footnotesize{$\pm 0.0041$}
 & 0.9254 & \footnotesize{$\pm 0.0010$} \\
    \midrule 
    CTGAN & 0.8623 & \footnotesize{$\pm 0.0078$}
 & 1.0337 & \footnotesize{$\pm 0.1615$}
 & 0.7268 & \footnotesize{$\pm 0.0158$}
 & 0.5908 & \footnotesize{$\pm 0.0094$}
 & 0.8823 & \footnotesize{$\pm 0.0097$}
 & 0.8603 & \footnotesize{$\pm 0.0179$}
 & 0.8501 & \footnotesize{$\pm 0.0168$}
 \\
    TVAE & 0.8710 & \footnotesize{$\pm 0.0152$}
 & 1.0378 & \footnotesize{$\pm 0.0376$}
 & 0.7467 & \footnotesize{$\pm 0.0043$}
 & 0.5798 & \footnotesize{$\pm 0.0221$}
 & 0.9181 & \footnotesize{$\pm 0.0067$}
 & 0.9831 & \footnotesize{$\pm 0.0309$}
 & 0.9032 & \footnotesize{$\pm 0.0095$}
 \\
    GReaT & 0.8341 & \footnotesize{$\pm 0.0051$}
 & 0.6216 & \footnotesize{$\pm 0.0142$}
 & 0.7584 & \footnotesize{$\pm 0.0055$}
 & \multicolumn{2}{c}{*} & 0.9114 & \footnotesize{$\pm 0.0053$}
 & \multicolumn{2}{c}{*} & 0.9067 & \footnotesize{$\pm 0.0039$}
 \\
    STaSy & 0.9058 & \footnotesize{$\pm 0.0010$}
 & 0.6761 & \footnotesize{$\pm 0.0417$}
 & 0.7534 & \footnotesize{$\pm 0.0075$}
 & \multicolumn{2}{c}{*} & 0.9320 & \footnotesize{$\pm 0.0038$}
 & 0.9222 & \footnotesize{$\pm 0.1028$}
 & 0.9103 & \footnotesize{$\pm 0.0057$}
 \\
    CoDi & 0.8048 & \footnotesize{$\pm 0.0353$}
 & 0.8082 & \footnotesize{$\pm 0.0580$}
 & 0.5030 & \footnotesize{$\pm 0.0107$}
 & 0.4778 & \footnotesize{$\pm{0.0000}$} & 0.9318 & \footnotesize{$\pm 0.0048$}
 & 2.1437 & \footnotesize{$\pm 0.7155$}
 & 0.8640 & \footnotesize{$\pm 0.0214$}
 \\
    TabSyn & 0.9104 & \footnotesize{$\pm 0.0014$}
 & 0.6260 & \footnotesize{$\pm 0.0496$}
 & 0.7590 & \footnotesize{$\pm 0.0058$}
 & 0.6862 & \footnotesize{$\pm 0.0032$}
 & 0.9370 & \footnotesize{$\pm 0.0022$}
 & \textbf{0.8577} & \footnotesize{$\pm 0.0312$}
 & \textbf{0.9144} & \footnotesize{$\pm 0.0071$}
 \\
    TabDiff & 0.9130 & \footnotesize{$\pm 0.0013$}
 & 0.5642 & \footnotesize{$\pm 0.0123$}
 & \textbf{0.7625} & \footnotesize{$\pm 0.0064$}
 & 0.6806 & \footnotesize{$\pm 0.0242$}
 & 0.9358 & \footnotesize{$\pm 0.0050$}
 & 0.8784 & \footnotesize{$\pm 0.0068$}
 & \textbf{0.9170} & \footnotesize{$\pm 0.0043$}
 \\
    TabPFN & 0.9121 & \footnotesize{$\pm 0.0010$}
 & \textbf{0.4780} & \footnotesize{$\pm 0.0150$}
 & \textbf{0.7626} & \footnotesize{$\pm 0.0030$}
 & \multicolumn{2}{c}{*} & \textbf{0.9413} & \footnotesize{$\pm 0.0028$}
 & \textbf{0.8524} & \footnotesize{$\pm 0.0340$}
 & \textbf{0.9155} & \footnotesize{$\pm 0.0115$}
 \\
    \midrule
    FF & 0.5159 & \footnotesize{$\pm 0.0237$}
 & 1.1062 & \footnotesize{$\pm 0.0385$}
 & 0.4895 & \footnotesize{$\pm 0.0159$}
 & 0.4960 & \footnotesize{$\pm 0.0049$}
 & 0.5124 & \footnotesize{$\pm 0.0450$}
 & 0.9225 & \footnotesize{$\pm 0.0146$}
 & 0.4487 & \footnotesize{$\pm 0.0611$}
 \\
    SM & 0.9022 & \footnotesize{$\pm 0.0055$}
 & 0.5584 & \footnotesize{$\pm 0.0133$}
 & 0.7321 & \footnotesize{$\pm 0.0027$}
 & 0.5741 & \footnotesize{$\pm 0.0103$}
 & 0.9284 & \footnotesize{$\pm 0.0043$}
 & 0.9097 & \footnotesize{$\pm 0.0129$}
 & 0.8626 & \footnotesize{$\pm 0.0136$}
 \\
    TabPC & \textbf{0.9162} & \footnotesize{$\pm 0.0026$}
 & 0.6558 & \footnotesize{$\pm 0.0278$}
 & 0.7442 & \footnotesize{$\pm 0.0045$}
 & \textbf{0.6873} & \footnotesize{$\pm 0.0037$}
 & 0.9324 & \footnotesize{$\pm 0.0027$}
 & 0.8864 & \footnotesize{$\pm 0.0099$}
 & 0.8970 & \footnotesize{$\pm 0.0038$}
 \\
    \bottomrule[1.0pt]
\end{tabular}}
\end{table}

%% file: tables_new_protocol/training_time.tex
\begin{table}[htbp]
\centering
\caption{\textbf{\ours offers competitive performance with \textit{training times} one or two orders of magnitude faster than \sota.} Even early, less performant models treated as a simple baseline like \ctgan and \tvae take longer to train than all circuit-based methods while achieving significantly worse results. {No training times are recorded for \tabpfn as it is a pre-trained foundation model.}}
\label{tab:training time_updated}
\resizebox{\textwidth}{!}{\begin{tabular}{
l
r@{}l@{}r
r@{}l@{}r
r@{}l@{}r
r@{}l@{}r
r@{}l@{}r
r@{}l@{}r
r@{}l@{}r
r
}
    \toprule[0.8pt]
     \textbf{Method} & \multicolumn{3}{c}{\textbf{Adult}}
 & \multicolumn{3}{c}{\textbf{Beijing}}
 & \multicolumn{3}{c}{\textbf{Default}}
 & \multicolumn{3}{c}{\textbf{Diabetes}}
 & \multicolumn{3}{c}{\textbf{Magic}}
 & \multicolumn{3}{c}{\textbf{News}}
 & \multicolumn{3}{c}{\textbf{Shoppers}}
 & \textbf{Avg. Rank} \\
    \midrule 
    CTGAN & 3798.4 & \footnotesize{$\pm 173.8$} & ~~(6)
 & 4126.4 & \footnotesize{$\pm 315.6$} & ~~(6)
 & 3836.4 & \footnotesize{$\pm 235.4$} & ~~(6)
 & 13872.4 & \footnotesize{$\pm 621.6$} & ~~(6)
 & 1797.2 & \footnotesize{$\pm 160.7$} & ~~(5)
 & 8221.2 & \footnotesize{$\pm 761.1$} & ~~(7)
 & 1453.0 & \footnotesize{$\pm 68.6$} & ~~(5)
 & 5.86 \\
    TVAE & 1309.6 & \footnotesize{$\pm 102.3$} & ~~(4)
 & 1482.4 & \footnotesize{$\pm 68.8$} & ~~(4)
 & 1705.0 & \footnotesize{$\pm 102.2$} & ~~(4)
 & 3863.8 & \footnotesize{$\pm 194.3$} & ~~(4)
 & 834.6 & \footnotesize{$\pm 85.1$} & ~~(4)
 & 4588.0 & \footnotesize{$\pm 125.7$} & ~~(5)
 & 579.2 & \footnotesize{$\pm 72.7$} & ~~(4)
 & 4.14 \\
    GReaT & 16714.2 & \footnotesize{$\pm 573.1$} & ~~(9)
 & 13739.2 & \footnotesize{$\pm 371.1$} & ~~(8)
 & 26138.4 & \footnotesize{$\pm 691.2$} & ~~(10)
 & 20472.0 & \footnotesize{$\pm{0.0000}$} & ~~(7) & 6808.0 & \footnotesize{$\pm 244.4$} & ~~(8)
 & 63508.5 & \footnotesize{$\pm 39867.5$} & ~~(10)
 & 5238.4 & \footnotesize{$\pm 2630.3$} & ~~(8)
 & 8.57 \\
    STaSy & 6526.5 & \footnotesize{$\pm 263.0$} & ~~(7)
 & 7034.0 & \footnotesize{$\pm 299.6$} & ~~(7)
 & 6300.4 & \footnotesize{$\pm 758.6$} & ~~(7)
 & 21701.6 & \footnotesize{$\pm 605.7$} & ~~(8)
 & 4949.2 & \footnotesize{$\pm 355.9$} & ~~(7)
 & 7079.8 & \footnotesize{$\pm 693.4$} & ~~(6)
 & 4385.0 & \footnotesize{$\pm 402.8$} & ~~(7)
 & 7.00 \\
    CoDi & 25186.2 & \footnotesize{$\pm 593.2$} & ~~(10)
 & 27023.8 & \footnotesize{$\pm 396.6$} & ~~(10)
 & 20233.8 & \footnotesize{$\pm 314.5$} & ~~(9)
 & 1027557.0 & \footnotesize{$\pm{0.0000}$} & ~~(10) & 9617.8 & \footnotesize{$\pm 112.0$} & ~~(9)
 & 20495.5 & \footnotesize{$\pm 394.9$} & ~~(9)
 & 8238.2 & \footnotesize{$\pm 242.1$} & ~~(9)
 & 9.43 \\
    TabSyn & 3082.5 & \footnotesize{$\pm 516.0$} & ~~(5)
 & 3321.2 & \footnotesize{$\pm 327.2$} & ~~(5)
 & 3438.5 & \footnotesize{$\pm 578.1$} & ~~(5)
 & 6506.2 & \footnotesize{$\pm 389.9$} & ~~(5)
 & 2270.8 & \footnotesize{$\pm 369.8$} & ~~(6)
 & 3746.5 & \footnotesize{$\pm 395.9$} & ~~(4)
 & 2441.2 & \footnotesize{$\pm 192.9$} & ~~(6)
 & 5.14 \\
    TabDiff & 11604.6 & \footnotesize{$\pm 2638.3$} & ~~(8)
 & 14418.4 & \footnotesize{$\pm 3527.0$} & ~~(9)
 & 13899.0 & \footnotesize{$\pm 2724.0$} & ~~(8)
 & 32069.2 & \footnotesize{$\pm 2849.5$} & ~~(9)
 & 11643.4 & \footnotesize{$\pm 3487.7$} & ~~(10)
 & 16040.0 & \footnotesize{$\pm 2865.3$} & ~~(8)
 & 10320.8 & \footnotesize{$\pm 2256.3$} & ~~(10)
 & 8.86 \\
    \midrule
    FF & 1.1 & \footnotesize{$\pm 0.3$} & ~~(1)
 & 1.7 & \footnotesize{$\pm 0.4$} & ~~(1)
 & 2.5 & \footnotesize{$\pm 0.3$} & ~~(1)
 & 2.4 & \footnotesize{$\pm 0.4$} & ~~(1)
 & 0.5 & \footnotesize{$\pm 0.0$} & ~~(1)
 & 7.1 & \footnotesize{$\pm 0.8$} & ~~(1)
 & 0.5 & \footnotesize{$\pm 0.0$} & ~~(1)
 & \textbf{1.00} \\
    SM & 59.7 & \footnotesize{$\pm 6.9$} & ~~(2)
 & 106.3 & \footnotesize{$\pm 5.8$} & ~~(2)
 & 163.9 & \footnotesize{$\pm 4.3$} & ~~(3)
 & 293.4 & \footnotesize{$\pm 7.8$} & ~~(3)
 & 15.4 & \footnotesize{$\pm 0.7$} & ~~(2)
 & 268.4 & \footnotesize{$\pm 41.3$} & ~~(2)
 & 47.6 & \footnotesize{$\pm 0.8$} & ~~(2)
 & 2.29 \\
    TabPC & 272.5 & \footnotesize{$\pm 18.8$} & ~~(3)
 & 581.6 & \footnotesize{$\pm 40.3$} & ~~(3)
 & 78.3 & \footnotesize{$\pm 16.2$} & ~~(2)
 & 247.9 & \footnotesize{$\pm 22.3$} & ~~(2)
 & 38.9 & \footnotesize{$\pm 1.9$} & ~~(3)
 & 404.0 & \footnotesize{$\pm 232.2$} & ~~(3)
 & 50.6 & \footnotesize{$\pm 0.8$} & ~~(3)
 & 2.71 \\
    \bottomrule[1.0pt]
\end{tabular}}
\end{table}

%% file: tables_new_protocol/sampling_time.tex
\begin{table}[htbp]
\centering
\caption{{\textbf{\ours is the best ranked \sota method in terms of dataset \textit{sampling time} (i.e. time taken to generate one dataset),} beating \tabsyn and \tabdiff in six out of the seven datasets. In each case, the generated dataset is of the same size as the training dataset. Note the extremely high sampling time for \tabpfn, which must autoregressively sample features and average over multiple feature ordering permutations for robustness.}
}
\label{tab:sampling time_updated}
\resizebox{\textwidth}{!}{\begin{tabular}{
l
r@{}l@{}r
r@{}l@{}r
r@{}l@{}r
r@{}l@{}r
r@{}l@{}r
r@{}l@{}r
r@{}l@{}r
r
}
    \toprule[0.8pt]
     \textbf{Method} & \multicolumn{3}{c}{\textbf{Adult}}
 & \multicolumn{3}{c}{\textbf{Beijing}}
 & \multicolumn{3}{c}{\textbf{Default}}
 & \multicolumn{3}{c}{\textbf{Diabetes}}
 & \multicolumn{3}{c}{\textbf{Magic}}
 & \multicolumn{3}{c}{\textbf{News}}
 & \multicolumn{3}{c}{\textbf{Shoppers}}
 & \textbf{Avg. Rank} \\
    \midrule 
    CTGAN & 1.9 & \footnotesize{$\pm 0.8$} & ~~(4)
 & 2.4 & \footnotesize{$\pm 0.7$} & ~~(4)
 & 2.4 & \footnotesize{$\pm 0.3$} & ~~(5)
 & 6.3 & \footnotesize{$\pm 0.4$} & ~~(4)
 & 1.1 & \footnotesize{$\pm 0.2$} & ~~(5)
 & 4.1 & \footnotesize{$\pm 0.2$} & ~~(3)
 & 1.1 & \footnotesize{$\pm 0.3$} & ~~(5)
 & 4.29 \\
    TVAE & 1.2 & \footnotesize{$\pm 0.4$} & ~~(2)
 & 1.9 & \footnotesize{$\pm 0.8$} & ~~(3)
 & 1.9 & \footnotesize{$\pm 0.4$} & ~~(4)
 & 6.1 & \footnotesize{$\pm 2.8$} & ~~(3)
 & 0.7 & \footnotesize{$\pm 0.2$} & ~~(2)
 & 4.0 & \footnotesize{$\pm 2.1$} & ~~(2)
 & 0.6 & \footnotesize{$\pm 0.1$} & ~~(2)
 & 2.57 \\
    GReaT & 354.0 & \footnotesize{$\pm 21.5$} & ~~(10)
 & 232.4 & \footnotesize{$\pm 1.5$} & ~~(10)
 & 534.2 & \footnotesize{$\pm 5.8$} & ~~(10)
 & \multicolumn{2}{c}{*} & ~~(10) & 104.0 & \footnotesize{$\pm 2.3$} & ~~(10)
 & \multicolumn{2}{c}{*} & ~~(11) & 156.2 & \footnotesize{$\pm 104.1$} & ~~(10)
 & 10.14 \\
    STaSy & 48.9 & \footnotesize{$\pm 75.9$} & ~~(9)
 & 13.4 & \footnotesize{$\pm 0.5$} & ~~(9)
 & 16.9 & \footnotesize{$\pm 11.1$} & ~~(9)
 & 5081.2 & \footnotesize{$\pm{0.0000}$} & ~~(9) & 5.8 & \footnotesize{$\pm 3.6$} & ~~(9)
 & 18.9 & \footnotesize{$\pm 8.2$} & ~~(9)
 & 7.0 & \footnotesize{$\pm 3.4$} & ~~(9)
 & 9.00 \\
    CoDi & 8.8 & \footnotesize{$\pm 0.2$} & ~~(7)
 & 9.2 & \footnotesize{$\pm 0.3$} & ~~(8)
 & 6.9 & \footnotesize{$\pm 0.2$} & ~~(7)
 & 635.4 & \footnotesize{$\pm{0.0000}$} & ~~(8) & 3.5 & \footnotesize{$\pm 0.2$} & ~~(7)
 & 8.3 & \footnotesize{$\pm 0.3$} & ~~(6)
 & 3.0 & \footnotesize{$\pm 0.2$} & ~~(7)
 & 7.14 \\
    TabSyn & 5.0 & \footnotesize{$\pm 0.2$} & ~~(5)
 & 5.4 & \footnotesize{$\pm 0.4$} & ~~(6)
 & 5.2 & \footnotesize{$\pm 0.6$} & ~~(6)
 & 15.6 & \footnotesize{$\pm 1.7$} & ~~(6)
 & 2.6 & \footnotesize{$\pm 0.2$} & ~~(6)
 & 11.1 & \footnotesize{$\pm 1.7$} & ~~(7)
 & 2.0 & \footnotesize{$\pm 0.2$} & ~~(6)
 & 6.00 \\
    TabDiff & 10.8 & \footnotesize{$\pm 0.5$} & ~~(8)
 & 9.0 & \footnotesize{$\pm 0.4$} & ~~(7)
 & 9.6 & \footnotesize{$\pm 0.3$} & ~~(8)
 & 176.5 & \footnotesize{$\pm 1.4$} & ~~(7)
 & 4.4 & \footnotesize{$\pm 0.4$} & ~~(8)
 & 16.2 & \footnotesize{$\pm 0.4$} & ~~(8)
 & 6.0 & \footnotesize{$\pm 0.5$} & ~~(8)
 & 7.71 \\
    TabPFN & 571.3 & \footnotesize{$\pm 1.6$} & ~~(11)
 & 623.6 & \footnotesize{$\pm 9.3$} & ~~(11)
 & 1213.1 & \footnotesize{$\pm 37.1$} & ~~(11)
 & \multicolumn{2}{c}{*} & ~~(10) & 343.4 & \footnotesize{$\pm 16.8$} & ~~(11)
 & 3978.5 & \footnotesize{$\pm 60.9$} & ~~(10)
 & 254.1 & \footnotesize{$\pm 3.1$} & ~~(11)
 & 10.71 \\
    \midrule
    FF & 0.5 & \footnotesize{$\pm 0.2$} & ~~(1)
 & 0.6 & \footnotesize{$\pm 0.1$} & ~~(1)
 & 0.9 & \footnotesize{$\pm 0.1$} & ~~(1)
 & 0.8 & \footnotesize{$\pm 0.1$} & ~~(1)
 & 0.1 & \footnotesize{$\pm 0.0$} & ~~(1)
 & 2.5 & \footnotesize{$\pm 0.2$} & ~~(1)
 & 0.2 & \footnotesize{$\pm 0.0$} & ~~(1)
 & \textbf{1.00} \\
    SM & 1.3 & \footnotesize{$\pm 0.0$} & ~~(3)
 & 1.4 & \footnotesize{$\pm 0.1$} & ~~(2)
 & 1.8 & \footnotesize{$\pm 0.0$} & ~~(3)
 & 3.1 & \footnotesize{$\pm 0.0$} & ~~(2)
 & 1.0 & \footnotesize{$\pm 0.0$} & ~~(4)
 & 4.4 & \footnotesize{$\pm 0.1$} & ~~(4)
 & 1.0 & \footnotesize{$\pm 0.0$} & ~~(3)
 & 3.00 \\
    TabPC & 5.9 & \footnotesize{$\pm 0.2$} & ~~(6)
 & 4.3 & \footnotesize{$\pm 0.0$} & ~~(5)
 & 1.5 & \footnotesize{$\pm 0.0$} & ~~(2)
 & 7.7 & \footnotesize{$\pm 0.0$} & ~~(5)
 & 0.7 & \footnotesize{$\pm 0.0$} & ~~(3)
 & 6.6 & \footnotesize{$\pm 0.1$} & ~~(5)
 & 1.1 & \footnotesize{$\pm 0.0$} & ~~(4)
 & 4.29 \\
    \bottomrule[1.0pt]
\end{tabular}}
\end{table}

%% file: tables_new_protocol/num_parameters.tex
\begin{table}[htbp]
\centering
\caption{Parameter counts for all methods and datasets. For \ours, our choice of overparameterisation yields PCs that can have even an order of magnitude higher number of parameters compared to other DGMs. 
This observation is explained by PCs' tractability: since the only non-linearities are in the inputs, to attain similar expressivity to non-tractable DGMs, PCs must compensate by having more parameters. This does not come at the cost of an increased training time. {\tabpfn is omitted since it is a pre-trained foundation model.}
}
\label{tab:num parameters_updated}
\begin{tabular}{
l
r@{}l@{}l
r@{}l@{}l
r@{}l@{}l
r@{}l@{}l
r@{}l@{}l
r@{}l@{}l
r@{}l@{}l
}
    \toprule[0.8pt]
     \textbf{Method} & \multicolumn{3}{c}{\textbf{Adult}}
 & \multicolumn{3}{c}{\textbf{Beijing}}
 & \multicolumn{3}{c}{\textbf{Default}}
 & \multicolumn{3}{c}{\textbf{Diabetes}}
 & \multicolumn{3}{c}{\textbf{Magic}}
 & \multicolumn{3}{c}{\textbf{News}}
 & \multicolumn{3}{c}{\textbf{Shoppers}}
 \\
    \midrule 
    CTGAN & $9.60$ & $\times$ & $10^{6}$
 & $9.58$ & $\times$ & $10^{6}$
 & $9.66$ & $\times$ & $10^{6}$
 & $1.18$ & $\times$ & $10^{7}$
 & $9.56$ & $\times$ & $10^{6}$
 & $9.89$ & $\times$ & $10^{6}$
 & $9.60$ & $\times$ & $10^{6}$
 \\
    TVAE & $1.07$ & $\times$ & $10^{7}$
 & $1.06$ & $\times$ & $10^{7}$
 & $1.07$ & $\times$ & $10^{7}$
 & $1.28$ & $\times$ & $10^{7}$
 & $1.06$ & $\times$ & $10^{7}$
 & $1.09$ & $\times$ & $10^{7}$
 & $1.07$ & $\times$ & $10^{7}$
 \\
    GReaT & $1.21$ & $\times$ & $10^{8}$
 & $1.21$ & $\times$ & $10^{8}$
 & $1.21$ & $\times$ & $10^{8}$
 & \multicolumn{3}{c}{*} & $1.21$ & $\times$ & $10^{8}$
 & \multicolumn{3}{c}{*} & $1.21$ & $\times$ & $10^{8}$
 \\
    STaSy & $4.26$ & $\times$ & $10^{7}$
 & $4.15$ & $\times$ & $10^{7}$
 & $4.19$ & $\times$ & $10^{7}$
 & $1.49$ & $\times$ & $10^{8}$
 & $3.85$ & $\times$ & $10^{7}$
 & $4.05$ & $\times$ & $10^{7}$
 & $4.12$ & $\times$ & $10^{7}$
 \\
    CoDi & $1.19$ & $\times$ & $10^{7}$
 & $1.19$ & $\times$ & $10^{7}$
 & $1.19$ & $\times$ & $10^{7}$
 & $1.19$ & $\times$ & $10^{7}$
 & $1.19$ & $\times$ & $10^{7}$
 & $1.19$ & $\times$ & $10^{7}$
 & $1.19$ & $\times$ & $10^{7}$
 \\
    TabSyn & $1.06$ & $\times$ & $10^{7}$
 & $1.06$ & $\times$ & $10^{7}$
 & $1.07$ & $\times$ & $10^{7}$
 & $1.08$ & $\times$ & $10^{7}$
 & $1.06$ & $\times$ & $10^{7}$
 & $1.09$ & $\times$ & $10^{7}$
 & $1.07$ & $\times$ & $10^{7}$
 \\
    TabDiff & $2.12$ & $\times$ & $10^{7}$
 & $2.12$ & $\times$ & $10^{7}$
 & $2.14$ & $\times$ & $10^{7}$
 & $2.16$ & $\times$ & $10^{7}$
 & $2.12$ & $\times$ & $10^{7}$
 & $2.18$ & $\times$ & $10^{7}$
 & $2.13$ & $\times$ & $10^{7}$
 \\
    \midrule
    FF & $1.34$ & $\times$ & $10^{2}$
 & $1.53$ & $\times$ & $10^{2}$
 & $1.19$ & $\times$ & $10^{2}$
 & $2.35$ & $\times$ & $10^{3}$
 & $2.40$ & $\times$ & $10^{1}$
 & $3.60$ & $\times$ & $10^{2}$
 & $1.45$ & $\times$ & $10^{2}$
 \\
    SM & $2.70$ & $\times$ & $10^{6}$
 & $7.70$ & $\times$ & $10^{6}$
 & $2.40$ & $\times$ & $10^{6}$
 & $4.69$ & $\times$ & $10^{7}$
 & $2.50$ & $\times$ & $10^{5}$
 & $3.61$ & $\times$ & $10^{6}$
 & $2.92$ & $\times$ & $10^{6}$
 \\
    TabPC & $4.54$ & $\times$ & $10^{8}$
 & $3.19$ & $\times$ & $10^{8}$
 & $1.29$ & $\times$ & $10^{7}$
 & $2.31$ & $\times$ & $10^{8}$
 & $6.30$ & $\times$ & $10^{7}$
 & $1.23$ & $\times$ & $10^{8}$
 & $1.51$ & $\times$ & $10^{8}$
 \\
    \bottomrule[1.0pt]
\end{tabular}
\end{table}

%% file: tables_new_protocol/dcr.tex
\begin{table}[htbp]
\centering
\caption{According to DCR, \ours offers a similar privacy risk as \sota diffusion-based models. However, we recall the criticisms of \citealt{yao_dcr_2025} that DCR is a ``flawed by design'' metric. These numbers are reported for reference to compare with older works. For DCR, values closer to $0.5$ are preferable.}
\label{tab:dcr_updated}
\resizebox{\textwidth}{!}{\begin{tabular}{
l
r@{}l
r@{}l
r@{}l
r@{}l
r@{}l
r@{}l
r@{}l
}
    \toprule[0.8pt]
     \textbf{Method} & \multicolumn{2}{c}{\textbf{Adult}}
 & \multicolumn{2}{c}{\textbf{Beijing}}
 & \multicolumn{2}{c}{\textbf{Default}}
 & \multicolumn{2}{c}{\textbf{Diabetes}}
 & \multicolumn{2}{c}{\textbf{Magic}}
 & \multicolumn{2}{c}{\textbf{News}}
 & \multicolumn{2}{c}{\textbf{Shoppers}}
 \\
    \midrule 
    CTGAN & 0.4838 & \footnotesize{$\pm 0.0184$}
 & 0.4969 & \footnotesize{$\pm 0.0149$}
 & 0.4970 & \footnotesize{$\pm 0.0047$}
 & 0.5041 & \footnotesize{$\pm 0.0047$}
 & 0.4972 & \footnotesize{$\pm 0.0060$}
 & 0.4944 & \footnotesize{$\pm 0.0014$}
 & 0.4934 & \footnotesize{$\pm 0.0562$}
 \\
    TVAE & 0.5027 & \footnotesize{$\pm 0.0119$}
 & 0.4892 & \footnotesize{$\pm 0.0145$}
 & 0.5024 & \footnotesize{$\pm 0.0046$}
 & 0.3618 & \footnotesize{$\pm 0.2028$}
 & 0.4958 & \footnotesize{$\pm 0.0021$}
 & 0.5104 & \footnotesize{$\pm 0.0069$}
 & 0.4780 & \footnotesize{$\pm 0.0304$}
 \\
    GReaT & 0.5194 & \footnotesize{$\pm 0.0042$}
 & 0.5124 & \footnotesize{$\pm 0.0040$}
 & 0.4877 & \footnotesize{$\pm 0.0028$}
 & \multicolumn{2}{c}{*} & 0.5131 & \footnotesize{$\pm 0.0045$}
 & \multicolumn{2}{c}{*} & 0.4934 & \footnotesize{$\pm 0.0046$}
 \\
    STaSy & 0.5020 & \footnotesize{$\pm 0.0052$}
 & 0.5047 & \footnotesize{$\pm 0.0030$}
 & 0.5061 & \footnotesize{$\pm 0.0048$}
 & 0.9974 & \footnotesize{$\pm{0.0000}$} & 0.5103 & \footnotesize{$\pm 0.0053$}
 & 0.5029 & \footnotesize{$\pm 0.0104$}
 & 0.4866 & \footnotesize{$\pm 0.0329$}
 \\
    CoDi & 0.5021 & \footnotesize{$\pm 0.0055$}
 & 0.5057 & \footnotesize{$\pm 0.0061$}
 & 0.5053 & \footnotesize{$\pm 0.0066$}
 & 0.5048 & \footnotesize{$\pm{0.0000}$} & 0.5180 & \footnotesize{$\pm 0.0035$}
 & 0.4976 & \footnotesize{$\pm 0.0021$}
 & 0.5057 & \footnotesize{$\pm 0.0060$}
 \\
    TabSyn & 0.5069 & \footnotesize{$\pm 0.0012$}
 & 0.5022 & \footnotesize{$\pm 0.0016$}
 & 0.5084 & \footnotesize{$\pm 0.0020$}
 & 0.5071 & \footnotesize{$\pm 0.0018$}
 & 0.5045 & \footnotesize{$\pm 0.0029$}
 & 0.4953 & \footnotesize{$\pm 0.0021$}
 & 0.5069 & \footnotesize{$\pm 0.0044$}
 \\
    TabDiff & 0.5488 & \footnotesize{$\pm 0.0071$}
 & 0.5139 & \footnotesize{$\pm 0.0025$}
 & 0.5221 & \footnotesize{$\pm 0.0021$}
 & 0.5181 & \footnotesize{$\pm 0.0079$}
 & 0.5011 & \footnotesize{$\pm 0.0032$}
 & 0.5091 & \footnotesize{$\pm 0.0026$}
 & 0.5120 & \footnotesize{$\pm 0.0028$}
 \\
    \midrule
    FF & 0.5034 & \footnotesize{$\pm 0.0011$}
 & 0.4978 & \footnotesize{$\pm 0.0011$}
 & 0.5127 & \footnotesize{$\pm 0.0018$}
 & 0.5057 & \footnotesize{$\pm 0.0028$}
 & 0.4876 & \footnotesize{$\pm 0.0024$}
 & 0.4852 & \footnotesize{$\pm 0.0024$}
 & 0.5021 & \footnotesize{$\pm 0.0043$}
 \\
    SM & 0.5498 & \footnotesize{$\pm 0.0014$}
 & 0.5545 & \footnotesize{$\pm 0.0041$}
 & 0.5133 & \footnotesize{$\pm 0.0025$}
 & 0.5067 & \footnotesize{$\pm 0.0022$}
 & 0.5348 & \footnotesize{$\pm 0.0023$}
 & 0.4942 & \footnotesize{$\pm 0.0025$}
 & 0.5091 & \footnotesize{$\pm 0.0048$}
 \\
    TabPC & 0.5431 & \footnotesize{$\pm 0.0047$}
 & 0.5162 & \footnotesize{$\pm 0.0020$}
 & 0.5072 & \footnotesize{$\pm 0.0031$}
 & 0.5208 & \footnotesize{$\pm 0.0025$}
 & 0.5295 & \footnotesize{$\pm 0.0028$}
 & 0.4943 & \footnotesize{$\pm 0.0025$}
 & 0.5128 & \footnotesize{$\pm 0.0024$}
 \\
    \bottomrule[1.0pt]
\end{tabular}}
\end{table}

%% file: tables_new_protocol/dcr002.tex
\begin{table}[htbp]
\centering
\caption{According to DCR-002, \ours may offer a similar risk of privacy leakage as \sota diffusion-based models. For DCR-002, numbers equal to or lower than $2$ may suggest a lower risk of privacy leakage.}
\label{tab:dcr fraction<0.02 test quantile_updated}
\resizebox{\textwidth}{!}{\begin{tabular}{
l
r@{}l
r@{}l
r@{}l
r@{}l
r@{}l
r@{}l
r@{}l
}
    \toprule[0.8pt]
     \textbf{Method} & \multicolumn{2}{c}{\textbf{Adult}}
 & \multicolumn{2}{c}{\textbf{Beijing}}
 & \multicolumn{2}{c}{\textbf{Default}}
 & \multicolumn{2}{c}{\textbf{Diabetes}}
 & \multicolumn{2}{c}{\textbf{Magic}}
 & \multicolumn{2}{c}{\textbf{News}}
 & \multicolumn{2}{c}{\textbf{Shoppers}}
 \\
    \midrule 
    CTGAN & 0.0934 & \footnotesize{$\pm 0.1242$}
 & \multicolumn{2}{c}{<0.0001} & \multicolumn{2}{c}{<0.0001} & 0.0226 & \footnotesize{$\pm 0.0096$}
 & 0.0023 & \footnotesize{$\pm 0.0052$}
 & 0.0824 & \footnotesize{$\pm 0.0175$}
 & \multicolumn{2}{c}{<0.0001} \\
    TVAE & 5.7775 & \footnotesize{$\pm 1.6151$}
 & 0.0059 & \footnotesize{$\pm 0.0029$}
 & 0.0874 & \footnotesize{$\pm 0.0350$}
 & 14.1846 & \footnotesize{$\pm 9.6827$}
 & 0.0140 & \footnotesize{$\pm 0.0147$}
 & 0.2175 & \footnotesize{$\pm 0.0697$}
 & 0.0162 & \footnotesize{$\pm 0.0075$}
 \\
    GReaT & \multicolumn{2}{c}{<0.0001} & 1.4199 & \footnotesize{$\pm 0.0606$}
 & 12.5178 & \footnotesize{$\pm 0.1007$}
 & \multicolumn{2}{c}{*} & 0.3529 & \footnotesize{$\pm 0.0517$}
 & \multicolumn{2}{c}{*} & 10.6047 & \footnotesize{$\pm 0.2320$}
 \\
    STaSy & 0.9809 & \footnotesize{$\pm 0.2132$}
 & 0.4454 & \footnotesize{$\pm 0.1188$}
 & 0.4504 & \footnotesize{$\pm 0.1140$}
 & \multicolumn{2}{c}{<0.0001} & 0.2746 & \footnotesize{$\pm 0.1323$}
 & 0.2433 & \footnotesize{$\pm 0.1508$}
 & 0.7101 & \footnotesize{$\pm 0.9099$}
 \\
    CoDi & 0.0276 & \footnotesize{$\pm 0.0104$}
 & \multicolumn{2}{c}{<0.0001} & \multicolumn{2}{c}{<0.0001} & \multicolumn{2}{c}{<0.0001} & 0.0374 & \footnotesize{$\pm 0.0158$}
 & 0.0006 & \footnotesize{$\pm 0.0013$}
 & \multicolumn{2}{c}{<0.0001} \\
    TabSyn & 1.6842 & \footnotesize{$\pm 0.1084$}
 & 0.5732 & \footnotesize{$\pm 0.1172$}
 & 1.1444 & \footnotesize{$\pm 0.0626$}
 & 0.6296 & \footnotesize{$\pm 0.0252$}
 & 0.0584 & \footnotesize{$\pm 0.0160$}
 & 1.0325 & \footnotesize{$\pm 0.1639$}
 & 1.8564 & \footnotesize{$\pm 0.2962$}
 \\
    TabDiff & 1.8673 & \footnotesize{$\pm 0.0624$}
 & 0.6386 & \footnotesize{$\pm 0.0518$}
 & 1.4067 & \footnotesize{$\pm 0.0647$}
 & 1.0662 & \footnotesize{$\pm 0.5802$}
 & 0.0608 & \footnotesize{$\pm 0.0270$}
 & 0.8178 & \footnotesize{$\pm 0.1928$}
 & 2.2601 & \footnotesize{$\pm 0.2484$}
 \\
    TabPFN & 2.4932 & \footnotesize{$\pm 0.0595$}
 & 2.6370 & \footnotesize{$\pm 0.0725$}
 & 0.8437 & \footnotesize{$\pm 0.0848$}
 & \multicolumn{2}{c}{*} & 0.2059 & \footnotesize{$\pm 0.0250$}
 & 0.2522 & \footnotesize{$\pm 0.0040$}
 & 0.0180 & \footnotesize{$\pm 0.0169$}
 \\
    \midrule
    FF & 0.0424 & \footnotesize{$\pm 0.0120$}
 & 0.0841 & \footnotesize{$\pm 0.0101$}
 & \multicolumn{2}{c}{<0.0001} & 0.0111 & \footnotesize{$\pm 0.0036$}
 & \multicolumn{2}{c}{<0.0001} & \multicolumn{2}{c}{<0.0001} & 0.0018 & \footnotesize{$\pm 0.0040$}
 \\
    SM & 2.0159 & \footnotesize{$\pm 0.1441$}
 & 0.9340 & \footnotesize{$\pm 0.0701$}
 & 0.8207 & \footnotesize{$\pm 0.1759$}
 & 0.0694 & \footnotesize{$\pm 0.0398$}
 & 0.5398 & \footnotesize{$\pm 0.1367$}
 & 0.0695 & \footnotesize{$\pm 0.0597$}
 & 1.6419 & \footnotesize{$\pm 0.3882$}
 \\
    TabPC & 1.9029 & \footnotesize{$\pm 0.1342$}
 & 0.5311 & \footnotesize{$\pm 0.0419$}
 & 0.8719 & \footnotesize{$\pm 0.0937$}
 & 1.4344 & \footnotesize{$\pm 0.0495$}
 & 0.1718 & \footnotesize{$\pm 0.0479$}
 & 0.4316 & \footnotesize{$\pm 0.3200$}
 & 2.5773 & \footnotesize{$\pm 0.1746$}
 \\
    \bottomrule[1.0pt]
\end{tabular}}
\end{table}

%% file: tables_new_protocol/dcr005.tex
\begin{table}[htbp]
\centering
\caption{According to DCR-005, \ours may offer a similar risk of privacy leakage as \sota diffusion-based models. For DCR-005, numbers equal to or lower than $5$ may suggest a lower risk of privacy leakage.}
\label{tab:dcr fraction<0.05 test quantile_updated}
\resizebox{\textwidth}{!}{\begin{tabular}{
l
r@{}l
r@{}l
r@{}l
r@{}l
r@{}l
r@{}l
r@{}l
}
    \toprule[0.8pt]
     \textbf{Method} & \multicolumn{2}{c}{\textbf{Adult}}
 & \multicolumn{2}{c}{\textbf{Beijing}}
 & \multicolumn{2}{c}{\textbf{Default}}
 & \multicolumn{2}{c}{\textbf{Diabetes}}
 & \multicolumn{2}{c}{\textbf{Magic}}
 & \multicolumn{2}{c}{\textbf{News}}
 & \multicolumn{2}{c}{\textbf{Shoppers}}
 \\
    \midrule 
    CTGAN & 0.3225 & \footnotesize{$\pm 0.4189$}
 & 0.0117 & \footnotesize{$\pm 0.0058$}
 & 0.0222 & \footnotesize{$\pm 0.0191$}
 & 0.1209 & \footnotesize{$\pm 0.0437$}
 & 0.1122 & \footnotesize{$\pm 0.0263$}
 & 0.4686 & \footnotesize{$\pm 0.0401$}
 & 1.2183 & \footnotesize{$\pm 0.4743$}
 \\
    TVAE & 15.0886 & \footnotesize{$\pm 4.1617$}
 & 0.0804 & \footnotesize{$\pm 0.0051$}
 & 1.2956 & \footnotesize{$\pm 0.3023$}
 & 31.5980 & \footnotesize{$\pm 18.4872$}
 & 0.7420 & \footnotesize{$\pm 0.1228$}
 & 1.5768 & \footnotesize{$\pm 0.4036$}
 & 10.2550 & \footnotesize{$\pm 3.2313$}
 \\
    GReaT & \multicolumn{2}{c}{<0.0001} & 3.5747 & \footnotesize{$\pm 0.1059$}
 & 23.9711 & \footnotesize{$\pm 0.2405$}
 & \multicolumn{2}{c}{*} & 4.8373 & \footnotesize{$\pm 0.0413$}
 & \multicolumn{2}{c}{*} & 21.8780 & \footnotesize{$\pm 0.3721$}
 \\
    STaSy & 2.5134 & \footnotesize{$\pm 0.4673$}
 & 1.3539 & \footnotesize{$\pm 0.3471$}
 & 1.6000 & \footnotesize{$\pm 0.2372$}
 & \multicolumn{2}{c}{<0.0001} & 2.8533 & \footnotesize{$\pm 0.4710$}
 & 0.9827 & \footnotesize{$\pm 0.4453$}
 & 2.4403 & \footnotesize{$\pm 1.6136$}
 \\
    CoDi & 0.1480 & \footnotesize{$\pm 0.0270$}
 & 0.0037 & \footnotesize{$\pm 0.0040$}
 & \multicolumn{2}{c}{<0.0001} & \multicolumn{2}{c}{<0.0001} & 1.5096 & \footnotesize{$\pm 0.2736$}
 & 0.0241 & \footnotesize{$\pm 0.0117$}
 & 0.6759 & \footnotesize{$\pm 0.3872$}
 \\
    TabSyn & 4.3039 & \footnotesize{$\pm 0.2098$}
 & 1.6109 & \footnotesize{$\pm 0.3337$}
 & 3.1133 & \footnotesize{$\pm 0.1441$}
 & 2.0344 & \footnotesize{$\pm 0.0849$}
 & 1.5283 & \footnotesize{$\pm 0.1844$}
 & 2.6268 & \footnotesize{$\pm 0.3070$}
 & 4.9941 & \footnotesize{$\pm 0.4275$}
 \\
    TabDiff & 4.6829 & \footnotesize{$\pm 0.1926$}
 & 1.7179 & \footnotesize{$\pm 0.1245$}
 & 3.9593 & \footnotesize{$\pm 0.0926$}
 & 2.8962 & \footnotesize{$\pm 1.5304$}
 & 1.5377 & \footnotesize{$\pm 0.2004$}
 & 2.0561 & \footnotesize{$\pm 0.5072$}
 & 5.9187 & \footnotesize{$\pm 0.9103$}
 \\
    TabPFN & 6.3039 & \footnotesize{$\pm 0.1552$}
 & 6.6539 & \footnotesize{$\pm 0.1011$}
 & 4.0415 & \footnotesize{$\pm 0.2526$}
 & \multicolumn{2}{c}{*} & 2.8013 & \footnotesize{$\pm 0.0721$}
 & 1.5920 & \footnotesize{$\pm 0.0277$}
 & 2.5088 & \footnotesize{$\pm 0.1153$}
 \\
    \midrule
    FF & 0.1112 & \footnotesize{$\pm 0.0250$}
 & 0.2411 & \footnotesize{$\pm 0.0115$}
 & \multicolumn{2}{c}{<0.0001} & 0.0671 & \footnotesize{$\pm 0.0140$}
 & 0.0023 & \footnotesize{$\pm 0.0032$}
 & \multicolumn{2}{c}{<0.0001} & 0.2145 & \footnotesize{$\pm 0.0443$}
 \\
    SM & 5.0293 & \footnotesize{$\pm 0.1940$}
 & 2.3703 & \footnotesize{$\pm 0.1427$}
 & 2.3252 & \footnotesize{$\pm 0.4999$}
 & 0.3331 & \footnotesize{$\pm 0.1147$}
 & 5.6938 & \footnotesize{$\pm 0.6394$}
 & 0.3610 & \footnotesize{$\pm 0.1499$}
 & 4.4174 & \footnotesize{$\pm 0.5258$}
 \\
    TabPC & 4.8438 & \footnotesize{$\pm 0.2155$}
 & 1.4757 & \footnotesize{$\pm 0.1191$}
 & 2.7200 & \footnotesize{$\pm 0.1932$}
 & 3.9087 & \footnotesize{$\pm 0.1079$}
 & 3.0309 & \footnotesize{$\pm 0.3601$}
 & 1.2136 & \footnotesize{$\pm 0.6237$}
 & 6.2720 & \footnotesize{$\pm 0.1197$}
 \\
    \bottomrule[1.0pt]
\end{tabular}}
\end{table}

%% file: tables_new_protocol/bpd_values.tex
\begin{table}[!t]
    \centering
    \caption{\ours BPD values (see \cref{eq:bpd}) on all datasets and splits (using the hyperparameters in \cref{app:tabpc_implementation})}
    \label{tab:bpd_values}

    \begin{tabular}{
    l %
    S[table-format=1.4, table-alignment = center] %
    S[table-format=1.4, table-alignment = center] %
    S[table-format=1.4, table-alignment = center] %
    }
        \toprule
        \multirow{2}{*}{\textbf{Dataset}} & \multicolumn{3}{c}{\textbf{\ours BPD}} \\
         & \textbf{Training} & \textbf{Validation} & \textbf{Test}  \\ 
        \midrule
        Adult  &  1.1227   &  1.0476   & 1.0500 \\
        Beijing &  1.3615  &   1.3582  & 1.4655 \\
        Default &  0.3431  &  0.3458   & 0.5583 \\
        Diabetes &  1.2901  &  1.1650  & 1.2906 \\
        Magic   &    0.5294 &   0.5283 & 0.9391 \\
        News    &  0.3894   &  0.3884  & 0.5832 \\
        Shoppers   & 0.9056   & 0.8978 & 1.0208 \\ 
        \bottomrule
    \end{tabular}
\end{table}

%% file: tables_new_protocol/nmis.tex
\begin{table}[htbp]
\centering
\caption{\nmis results across all methods \textbf{show  more saturated scores than \wnmis}, since all values are above $0.9$, whereas for \wnmis, the weakest methods posted lower results. For a clear example of this, compare the \FF results here (where the lowest score is $0.8756$ on Magic) to their \wnmis scores (where the lowest is $0.6942$ on Default).}
\label{tab:nmi l1 complement_updated}
\resizebox{\textwidth}{!}{\begin{tabular}{
l
r@{}l@{}r
r@{}l@{}r
r@{}l@{}r
r@{}l@{}r
r@{}l@{}r
r@{}l@{}r
r@{}l@{}r
r
}
    \toprule[0.8pt]
     \textbf{Method} & \multicolumn{3}{c}{\textbf{Adult}}
 & \multicolumn{3}{c}{\textbf{Beijing}}
 & \multicolumn{3}{c}{\textbf{Default}}
 & \multicolumn{3}{c}{\textbf{Diabetes}}
 & \multicolumn{3}{c}{\textbf{Magic}}
 & \multicolumn{3}{c}{\textbf{News}}
 & \multicolumn{3}{c}{\textbf{Shoppers}}
 & \textbf{Avg. Rank} \\
    \midrule 
    CTGAN & 0.9737 & \footnotesize{$\pm 0.0035$} & ~~(8)
 & 0.9667 & \footnotesize{$\pm 0.0097$} & ~~(9)
 & 0.9537 & \footnotesize{$\pm 0.0050$} & ~~(10)
 & 0.9930 & \footnotesize{$\pm 0.0004$} & ~~(5)
 & 0.9594 & \footnotesize{$\pm 0.0037$} & ~~(10)
 & 0.9848 & \footnotesize{$\pm 0.0007$} & ~~(9)
 & 0.9615 & \footnotesize{$\pm 0.0079$} & ~~(11)
 & 8.86 \\
    TVAE & 0.9720 & \footnotesize{$\pm 0.0023$} & ~~(9)
 & 0.9305 & \footnotesize{$\pm 0.0137$} & ~~(11)
 & 0.9794 & \footnotesize{$\pm 0.0021$} & ~~(7)
 & 0.9813 & \footnotesize{$\pm 0.0081$} & ~~(9)
 & 0.9805 & \footnotesize{$\pm 0.0022$} & ~~(6)
 & 0.9863 & \footnotesize{$\pm 0.0006$} & ~~(7)
 & 0.9832 & \footnotesize{$\pm 0.0017$} & ~~(8)
 & 8.14 \\
    GReaT & 0.9889 & \footnotesize{$\pm 0.0001$} & ~~(7)
 & 0.9946 & \footnotesize{$\pm 0.0002$} & ~~(5)
 & 0.9617 & \footnotesize{$\pm 0.0005$} & ~~(9)
 & \multicolumn{2}{c}{*} & ~~(10) & 0.9776 & \footnotesize{$\pm 0.0005$} & ~~(8)
 & \multicolumn{2}{c}{*} & ~~(11) & 0.9896 & \footnotesize{$\pm 0.0006$} & ~~(5)
 & 7.86 \\
    STaSy & 0.9903 & \footnotesize{$\pm 0.0010$} & ~~(6)
 & 0.9929 & \footnotesize{$\pm 0.0017$} & ~~(6)
 & 0.9876 & \footnotesize{$\pm 0.0037$} & ~~(5)
 & 0.9907 & \footnotesize{$\pm{0.0000}$} & ~~(8) & 0.9840 & \footnotesize{$\pm 0.0087$} & ~~(5)
 & 0.9930 & \footnotesize{$\pm 0.0010$} & ~~(3)
 & 0.9909 & \footnotesize{$\pm 0.0005$} & ~~(4)
 & 5.29 \\
    CoDi & 0.9648 & \footnotesize{$\pm 0.0008$} & ~~(10)
 & 0.9915 & \footnotesize{$\pm 0.0031$} & ~~(8)
 & 0.9635 & \footnotesize{$\pm 0.0031$} & ~~(8)
 & 0.9927 & \footnotesize{$\pm{0.0000}$} & ~~(6) & 0.9795 & \footnotesize{$\pm 0.0021$} & ~~(7)
 & 0.9900 & \footnotesize{$\pm 0.0005$} & ~~(6)
 & 0.9773 & \footnotesize{$\pm 0.0023$} & ~~(9)
 & 7.71 \\
    TabSyn & 0.9949 & \footnotesize{$\pm 0.0006$} & ~~(4)
 & 0.9957 & \footnotesize{$\pm 0.0007$} & ~~(4)
 & 0.9955 & \footnotesize{$\pm 0.0005$} & ~~(2)
 & 0.9974 & \footnotesize{$\pm 0.0001$} & ~~(2)
 & 0.9975 & \footnotesize{$\pm 0.0002$} & ~~(1)
 & 0.9952 & \footnotesize{$\pm 0.0010$} & ~~(1)
 & 0.9958 & \footnotesize{$\pm 0.0006$} & ~~(3)
 & 2.43 \\
    TabDiff & 0.9965 & \footnotesize{$\pm 0.0002$} & ~~(2)
 & 0.9978 & \footnotesize{$\pm 0.0001$} & ~~(2)
 & 0.9967 & \footnotesize{$\pm 0.0003$} & ~~(1)
 & 0.9970 & \footnotesize{$\pm 0.0024$} & ~~(3)
 & 0.9965 & \footnotesize{$\pm 0.0009$} & ~~(2)
 & 0.9920 & \footnotesize{$\pm 0.0025$} & ~~(4)
 & 0.9965 & \footnotesize{$\pm 0.0001$} & ~~(2)
 & 2.29 \\
    TabPFN & 0.9959 & \footnotesize{$\pm 0.0002$} & ~~(3)
 & 0.9987 & \footnotesize{$\pm 0.0001$} & ~~(1)
 & 0.9863 & \footnotesize{$\pm 0.0004$} & ~~(6)
 & \multicolumn{2}{c}{*} & ~~(10) & 0.9946 & \footnotesize{$\pm 0.0006$} & ~~(4)
 & 0.9904 & \footnotesize{$\pm 0.0002$} & ~~(5)
 & 0.9894 & \footnotesize{$\pm 0.0002$} & ~~(6)
 & 5.00 \\
    \midrule
    FF & 0.9594 & \footnotesize{$\pm{0.0000}$} & ~~(11) & 0.9637 & \footnotesize{$\pm{0.0000}$} & ~~(10) & 0.9339 & \footnotesize{$\pm{0.0000}$} & ~~(11) & 0.9924 & \footnotesize{$\pm{0.0000}$} & ~~(7) & 0.8756 & \footnotesize{$\pm{0.0000}$} & ~~(11) & 0.9791 & \footnotesize{$\pm{0.0000}$} & ~~(10) & 0.9770 & \footnotesize{$\pm{0.0000}$} & ~~(10) & 10.00 \\
    SM & 0.9930 & \footnotesize{$\pm 0.0005$} & ~~(5)
 & 0.9927 & \footnotesize{$\pm 0.0005$} & ~~(7)
 & 0.9901 & \footnotesize{$\pm 0.0006$} & ~~(4)
 & 0.9934 & \footnotesize{$\pm 0.0008$} & ~~(4)
 & 0.9699 & \footnotesize{$\pm 0.0021$} & ~~(9)
 & 0.9853 & \footnotesize{$\pm 0.0006$} & ~~(8)
 & 0.9865 & \footnotesize{$\pm 0.0007$} & ~~(7)
 & 6.29 \\
    TabPC & 0.9979 & \footnotesize{$\pm 0.0004$} & ~~(1)
 & 0.9973 & \footnotesize{$\pm 0.0002$} & ~~(3)
 & 0.9949 & \footnotesize{$\pm 0.0003$} & ~~(3)
 & 0.9984 & \footnotesize{$\pm 0.0003$} & ~~(1)
 & 0.9955 & \footnotesize{$\pm 0.0007$} & ~~(3)
 & 0.9940 & \footnotesize{$\pm 0.0006$} & ~~(2)
 & 0.9975 & \footnotesize{$\pm 0.0003$} & ~~(1)
 & 2.00 \\
    \bottomrule[1.0pt]
\end{tabular}}
\end{table}

%% file: tables_new_protocol/wnmis_discretisation.tex
\begin{table}[htbp]
\centering
\caption{\textbf{\wnmis is robust to the number of discretisation bins.} Using \ours, for every considered dataset we find that the computation of \wnmis varies little with the number of bins chosen to discretise continuous features. Throughout the rest of the paper, we use 10 bins.}
\label{tab:wnmis_ablation_discretisation}
\resizebox{\textwidth}{!}{\begin{tabular}{
l
r@{}l
r@{}l
r@{}l
r@{}l
r@{}l
r@{}l
r@{}l
r
}
    \toprule[0.8pt]
     \textbf{Num. Bins} & \multicolumn{2}{c}{\textbf{Adult}}
 & \multicolumn{2}{c}{\textbf{Beijing}}
 & \multicolumn{2}{c}{\textbf{Default}}
 & \multicolumn{2}{c}{\textbf{Diabetes}}
 & \multicolumn{2}{c}{\textbf{Magic}}
 & \multicolumn{2}{c}{\textbf{News}}
 & \multicolumn{2}{c}{\textbf{Shoppers}} \\
    \midrule 
    5 & 0.9949 & \footnotesize{$\pm 0.0038$}
 & 0.9924 & \footnotesize{$\pm 0.0010$}
 & 0.9909 & \footnotesize{$\pm 0.0016$}
 & 0.9919 & \footnotesize{$\pm 0.0025$}
 & 0.9915 & \footnotesize{$\pm 0.0018$}
 & 0.9231 & \footnotesize{$\pm 0.0086$}
 & 0.9896 & \footnotesize{$\pm 0.0015$} \\
10 & 0.9933 & \footnotesize{$\pm 0.0049$}
 & 0.9925 & \footnotesize{$\pm 0.0010$}
 & 0.9894 & \footnotesize{$\pm 0.0012$}
 & 0.9919 & \footnotesize{$\pm 0.0024$}
 & 0.9925 & \footnotesize{$\pm 0.0012$}
 & 0.9388 & \footnotesize{$\pm 0.0048$}
 & 0.9927 & \footnotesize{$\pm 0.0009$} \\
    20 & 0.9924 & \footnotesize{$\pm 0.0051$}
 & 0.9924 & \footnotesize{$\pm 0.0008$}
 & 0.9872 & \footnotesize{$\pm 0.0011$}
 & 0.9921 & \footnotesize{$\pm 0.0024$}
 & 0.9929 & \footnotesize{$\pm 0.0006$}
 & 0.9459 & \footnotesize{$\pm 0.0045$}
 & 0.9949 & \footnotesize{$\pm 0.0007$} \\
 50 & 0.9923 & \footnotesize{$\pm 0.0050$}
 & 0.9932 & \footnotesize{$\pm 0.0008$}
 & 0.9854 & \footnotesize{$\pm 0.0009$}
 & 0.9923 & \footnotesize{$\pm 0.0023$}
 & 0.9939 & \footnotesize{$\pm 0.0006$}
 & 0.9512 & \footnotesize{$\pm 0.0027$}
 & 0.9958 & \footnotesize{$\pm 0.0004$} \\
    \bottomrule[1.0pt]
\end{tabular}}
\end{table}

%% file: tables_new_protocol/alpha_precision.tex
\begin{table}[htbp]
\centering
\caption{We present the ``naive'' implementation results for \alphaprecision for backwards compatibility with prior works. \textbf{\ours is the best ranked and best performing model in terms of \alphaprecision (naive)}, a measure of sample realism, attaining scores greater than $0.99$ in six out of seven datasets, and being ranked first for all datasets.}
\label{tab:alpha precision_updated}
\resizebox{\textwidth}{!}{\begin{tabular}{
l
r@{}l@{}r
r@{}l@{}r
r@{}l@{}r
r@{}l@{}r
r@{}l@{}r
r@{}l@{}r
r@{}l@{}r
r
}
    \toprule[0.8pt]
     \textbf{Method} & \multicolumn{3}{c}{\textbf{Adult}}
 & \multicolumn{3}{c}{\textbf{Beijing}}
 & \multicolumn{3}{c}{\textbf{Default}}
 & \multicolumn{3}{c}{\textbf{Diabetes}}
 & \multicolumn{3}{c}{\textbf{Magic}}
 & \multicolumn{3}{c}{\textbf{News}}
 & \multicolumn{3}{c}{\textbf{Shoppers}}
 & \textbf{Avg. Rank} \\
    \midrule 
    CTGAN & 0.7539 & \footnotesize{$\pm 0.0362$} & ~~(9)
 & 0.9494 & \footnotesize{$\pm 0.0365$} & ~~(7)
 & 0.6734 & \footnotesize{$\pm 0.0486$} & ~~(11)
 & 0.7983 & \footnotesize{$\pm 0.0178$} & ~~(6)
 & 0.9040 & \footnotesize{$\pm 0.0200$} & ~~(8)
 & 0.9750 & \footnotesize{$\pm 0.0071$} & ~~(2)
 & 0.7790 & \footnotesize{$\pm 0.0374$} & ~~(10)
 & 7.57 \\
    TVAE & 0.6233 & \footnotesize{$\pm 0.0595$} & ~~(10)
 & 0.8622 & \footnotesize{$\pm 0.0651$} & ~~(10)
 & 0.9087 & \footnotesize{$\pm 0.0392$} & ~~(7)
 & 0.1084 & \footnotesize{$\pm 0.0749$} & ~~(8)
 & 0.9684 & \footnotesize{$\pm 0.0104$} & ~~(5)
 & 0.9216 & \footnotesize{$\pm 0.0380$} & ~~(7)
 & 0.4797 & \footnotesize{$\pm 0.1287$} & ~~(11)
 & 8.29 \\
    GReaT & 0.5898 & \footnotesize{$\pm 0.0010$} & ~~(11)
 & \multicolumn{2}{c}{*} & ~~(11) & 0.8657 & \footnotesize{$\pm 0.0036$} & ~~(8)
 & \multicolumn{2}{c}{*} & ~~(10) & 0.8678 & \footnotesize{$\pm 0.0049$} & ~~(9)
 & \multicolumn{2}{c}{*} & ~~(11) & 0.7843 & \footnotesize{$\pm 0.0030$} & ~~(9)
 & 9.86 \\
    STaSy & 0.8872 & \footnotesize{$\pm 0.0809$} & ~~(7)
 & 0.9199 & \footnotesize{$\pm 0.0612$} & ~~(9)
 & 0.9385 & \footnotesize{$\pm 0.0282$} & ~~(5)
 & \multicolumn{2}{c}{<0.0001} & ~~(9) & 0.9362 & \footnotesize{$\pm 0.0501$} & ~~(7)
 & 0.9218 & \footnotesize{$\pm 0.0548$} & ~~(6)
 & 0.9009 & \footnotesize{$\pm 0.0694$} & ~~(8)
 & 7.29 \\
    CoDi & 0.8021 & \footnotesize{$\pm 0.0287$} & ~~(8)
 & 0.9736 & \footnotesize{$\pm 0.0089$} & ~~(5)
 & 0.8189 & \footnotesize{$\pm 0.0063$} & ~~(9)
 & 0.4335 & \footnotesize{$\pm{0.0000}$} & ~~(7) & 0.8627 & \footnotesize{$\pm 0.0055$} & ~~(10)
 & 0.9159 & \footnotesize{$\pm 0.0077$} & ~~(9)
 & 0.9186 & \footnotesize{$\pm 0.0271$} & ~~(7)
 & 7.86 \\
    TabSyn & 0.9902 & \footnotesize{$\pm 0.0042$} & ~~(3)
 & 0.9822 & \footnotesize{$\pm 0.0052$} & ~~(3)
 & 0.9889 & \footnotesize{$\pm 0.0033$} & ~~(2)
 & 0.9795 & \footnotesize{$\pm 0.0095$} & ~~(4)
 & 0.9938 & \footnotesize{$\pm 0.0021$} & ~~(3)
 & 0.9558 & \footnotesize{$\pm 0.0099$} & ~~(4)
 & 0.9898 & \footnotesize{$\pm 0.0019$} & ~~(2)
 & 3.00 \\
    TabDiff & 0.9861 & \footnotesize{$\pm 0.0040$} & ~~(4)
 & 0.9779 & \footnotesize{$\pm 0.0028$} & ~~(4)
 & 0.9873 & \footnotesize{$\pm 0.0032$} & ~~(3)
 & 0.9422 & \footnotesize{$\pm 0.0106$} & ~~(5)
 & 0.9939 & \footnotesize{$\pm 0.0022$} & ~~(2)
 & 0.9172 & \footnotesize{$\pm 0.0576$} & ~~(8)
 & 0.9851 & \footnotesize{$\pm 0.0142$} & ~~(3)
 & 4.14 \\
    TabPFN & 0.9037 & \footnotesize{$\pm 0.0031$} & ~~(6)
 & 0.9946 & \footnotesize{$\pm 0.0007$} & ~~(2)
 & 0.9287 & \footnotesize{$\pm 0.0033$} & ~~(6)
 & \multicolumn{2}{c}{*} & ~~(10) & 0.9747 & \footnotesize{$\pm 0.0043$} & ~~(4)
 & 0.9539 & \footnotesize{$\pm 0.0009$} & ~~(5)
 & 0.9200 & \footnotesize{$\pm 0.0059$} & ~~(6)
 & 5.57 \\
    \midrule
    FF & 0.9568 & \footnotesize{$\pm 0.0036$} & ~~(5)
 & 0.9531 & \footnotesize{$\pm 0.0009$} & ~~(6)
 & 0.6969 & \footnotesize{$\pm 0.0006$} & ~~(10)
 & 0.9816 & \footnotesize{$\pm 0.0015$} & ~~(3)
 & 0.8612 & \footnotesize{$\pm 0.0008$} & ~~(11)
 & 0.8378 & \footnotesize{$\pm 0.0028$} & ~~(10)
 & 0.9713 & \footnotesize{$\pm 0.0043$} & ~~(4)
 & 7.00 \\
    SM & 0.9915 & \footnotesize{$\pm 0.0023$} & ~~(2)
 & 0.9394 & \footnotesize{$\pm 0.0102$} & ~~(8)
 & 0.9699 & \footnotesize{$\pm 0.0048$} & ~~(4)
 & 0.9820 & \footnotesize{$\pm 0.0085$} & ~~(2)
 & 0.9654 & \footnotesize{$\pm 0.0055$} & ~~(6)
 & 0.9732 & \footnotesize{$\pm 0.0061$} & ~~(3)
 & 0.9616 & \footnotesize{$\pm 0.0231$} & ~~(5)
 & 4.29 \\
    TabPC & 0.9961 & \footnotesize{$\pm 0.0014$} & ~~(1)
 & 0.9951 & \footnotesize{$\pm 0.0011$} & ~~(1)
 & 0.9950 & \footnotesize{$\pm 0.0008$} & ~~(1)
 & 0.9953 & \footnotesize{$\pm 0.0023$} & ~~(1)
 & 0.9948 & \footnotesize{$\pm 0.0017$} & ~~(1)
 & 0.9824 & \footnotesize{$\pm 0.0064$} & ~~(1)
 & 0.9917 & \footnotesize{$\pm 0.0046$} & ~~(1)
 & \textbf{1.00} \\
    \bottomrule[1.0pt]
\end{tabular}}
\end{table}

%% file: tables_new_protocol/beta_recall.tex
\begin{table}[htbp]
\centering
\caption{We present the ``naive'' implementation results for \betarecall for backwards compatibility with prior works. \textbf{\ours has the second highest average rank in \betarecall (naive)}, a measure of sample diversity. \ours offers performance which generally competitive with \sota, but it sometimes falls slightly behind \tabdiff which is the highest ranked method for this metric. {Note the relatively high scores of \tabpfn, which suggests that it is able to generate a comparably diverse range of samples.}}
\label{tab:beta recall_updated}
\resizebox{\textwidth}{!}{\begin{tabular}{
l
r@{}l@{}r
r@{}l@{}r
r@{}l@{}r
r@{}l@{}r
r@{}l@{}r
r@{}l@{}r
r@{}l@{}r
r
}
    \toprule[0.8pt]
     \textbf{Method} & \multicolumn{3}{c}{\textbf{Adult}}
 & \multicolumn{3}{c}{\textbf{Beijing}}
 & \multicolumn{3}{c}{\textbf{Default}}
 & \multicolumn{3}{c}{\textbf{Diabetes}}
 & \multicolumn{3}{c}{\textbf{Magic}}
 & \multicolumn{3}{c}{\textbf{News}}
 & \multicolumn{3}{c}{\textbf{Shoppers}}
 & \textbf{Avg. Rank} \\
    \midrule 
    CTGAN & 0.1313 & \footnotesize{$\pm 0.0613$} & ~~(8)
 & 0.3850 & \footnotesize{$\pm 0.0203$} & ~~(9)
 & 0.1050 & \footnotesize{$\pm 0.0124$} & ~~(10)
 & 0.0948 & \footnotesize{$\pm 0.0050$} & ~~(5)
 & 0.1588 & \footnotesize{$\pm 0.0076$} & ~~(10)
 & 0.2339 & \footnotesize{$\pm 0.0130$} & ~~(8)
 & 0.2378 & \footnotesize{$\pm 0.0197$} & ~~(8)
 & 8.29 \\
    TVAE & 0.1221 & \footnotesize{$\pm 0.0133$} & ~~(9)
 & 0.0540 & \footnotesize{$\pm 0.0233$} & ~~(10)
 & 0.3039 & \footnotesize{$\pm 0.0121$} & ~~(8)
 & 0.0150 & \footnotesize{$\pm 0.0134$} & ~~(7)
 & 0.3631 & \footnotesize{$\pm 0.0124$} & ~~(9)
 & 0.2811 & \footnotesize{$\pm 0.0256$} & ~~(7)
 & 0.1909 & \footnotesize{$\pm 0.0488$} & ~~(11)
 & 8.71 \\
    GReaT & 0.4825 & \footnotesize{$\pm 0.0001$} & ~~(4)
 & \multicolumn{2}{c}{*} & ~~(11) & 0.4211 & \footnotesize{$\pm 0.0012$} & ~~(5)
 & \multicolumn{2}{c}{*} & ~~(10) & 0.3995 & \footnotesize{$\pm 0.0041$} & ~~(8)
 & \multicolumn{2}{c}{*} & ~~(11) & 0.4551 & \footnotesize{$\pm 0.0061$} & ~~(5)
 & 7.71 \\
    STaSy & 0.3454 & \footnotesize{$\pm 0.0170$} & ~~(7)
 & 0.4796 & \footnotesize{$\pm 0.0348$} & ~~(7)
 & 0.3883 & \footnotesize{$\pm 0.0185$} & ~~(6)
 & \multicolumn{2}{c}{<0.0001} & ~~(9) & 0.4387 & \footnotesize{$\pm 0.0272$} & ~~(7)
 & 0.3870 & \footnotesize{$\pm 0.0147$} & ~~(3)
 & 0.3518 & \footnotesize{$\pm 0.0537$} & ~~(7)
 & 6.57 \\
    CoDi & 0.0875 & \footnotesize{$\pm 0.0092$} & ~~(10)
 & 0.5377 & \footnotesize{$\pm 0.0081$} & ~~(5)
 & 0.1863 & \footnotesize{$\pm 0.0026$} & ~~(9)
 & 0.0138 & \footnotesize{$\pm{0.0000}$} & ~~(8) & 0.5100 & \footnotesize{$\pm 0.0081$} & ~~(1)
 & 0.3662 & \footnotesize{$\pm 0.0082$} & ~~(4)
 & 0.1915 & \footnotesize{$\pm 0.0063$} & ~~(10)
 & 6.71 \\
    TabSyn & 0.4796 & \footnotesize{$\pm 0.0059$} & ~~(5)
 & 0.5221 & \footnotesize{$\pm 0.0440$} & ~~(6)
 & 0.4598 & \footnotesize{$\pm 0.0163$} & ~~(3)
 & 0.3521 & \footnotesize{$\pm 0.0070$} & ~~(3)
 & 0.4775 & \footnotesize{$\pm 0.0041$} & ~~(6)
 & 0.4310 & \footnotesize{$\pm 0.0282$} & ~~(2)
 & 0.4785 & \footnotesize{$\pm 0.0047$} & ~~(3)
 & 4.00 \\
    TabDiff & 0.5255 & \footnotesize{$\pm 0.0092$} & ~~(1)
 & 0.5981 & \footnotesize{$\pm 0.0029$} & ~~(3)
 & 0.5154 & \footnotesize{$\pm 0.0046$} & ~~(1)
 & 0.3690 & \footnotesize{$\pm 0.1486$} & ~~(2)
 & 0.4799 & \footnotesize{$\pm 0.0091$} & ~~(5)
 & 0.3561 & \footnotesize{$\pm 0.0739$} & ~~(5)
 & 0.5107 & \footnotesize{$\pm 0.0249$} & ~~(1)
 & \textbf{2.57} \\
    TabPFN & 0.5133 & \footnotesize{$\pm 0.0018$} & ~~(2)
 & 0.7340 & \footnotesize{$\pm 0.0020$} & ~~(1)
 & 0.5032 & \footnotesize{$\pm 0.0019$} & ~~(2)
 & \multicolumn{2}{c}{*} & ~~(10) & 0.4997 & \footnotesize{$\pm 0.0070$} & ~~(3)
 & 0.4629 & \footnotesize{$\pm 0.0021$} & ~~(1)
 & 0.4678 & \footnotesize{$\pm 0.0058$} & ~~(4)
 & 3.29 \\
    \midrule
    FF & 0.0769 & \footnotesize{$\pm 0.0007$} & ~~(11)
 & 0.4385 & \footnotesize{$\pm 0.0029$} & ~~(8)
 & 0.0665 & \footnotesize{$\pm 0.0009$} & ~~(11)
 & 0.0943 & \footnotesize{$\pm 0.0009$} & ~~(6)
 & 0.0196 & \footnotesize{$\pm 0.0007$} & ~~(11)
 & 0.0172 & \footnotesize{$\pm 0.0007$} & ~~(10)
 & 0.2307 & \footnotesize{$\pm 0.0060$} & ~~(9)
 & 9.43 \\
    SM & 0.4692 & \footnotesize{$\pm 0.0094$} & ~~(6)
 & 0.6421 & \footnotesize{$\pm 0.0118$} & ~~(2)
 & 0.3479 & \footnotesize{$\pm 0.0081$} & ~~(7)
 & 0.1637 & \footnotesize{$\pm 0.0249$} & ~~(4)
 & 0.4926 & \footnotesize{$\pm 0.0079$} & ~~(4)
 & 0.0664 & \footnotesize{$\pm 0.0033$} & ~~(9)
 & 0.3684 & \footnotesize{$\pm 0.0134$} & ~~(6)
 & 5.43 \\
    TabPC & 0.4995 & \footnotesize{$\pm 0.0073$} & ~~(3)
 & 0.5731 & \footnotesize{$\pm 0.0062$} & ~~(4)
 & 0.4490 & \footnotesize{$\pm 0.0040$} & ~~(4)
 & 0.4514 & \footnotesize{$\pm 0.0019$} & ~~(1)
 & 0.5047 & \footnotesize{$\pm 0.0079$} & ~~(2)
 & 0.3475 & \footnotesize{$\pm 0.0023$} & ~~(6)
 & 0.5003 & \footnotesize{$\pm 0.0061$} & ~~(2)
 & 3.14 \\
    \bottomrule[1.0pt]
\end{tabular}}
\end{table}